\documentclass[11pt]{amsart}

\usepackage{graphicx}
\usepackage[margin=1.1in]{geometry}
\usepackage{amsmath}
\usepackage{amssymb}
\usepackage{amsthm}
\usepackage{amsfonts}
\usepackage[export]{adjustbox}
\usepackage{graphicx}
\usepackage{float}
\usepackage{algorithm}% http://ctan.org/pkg/algorithms
\usepackage{algpseudocode}% http://ctan.org/pkg/algorithmicx
\usepackage[algo2e]{algorithm2e} 
\usepackage{enumitem}
\usepackage{subfig}
\usepackage[cmyk,usenames,dvipsnames]{xcolor}
\usepackage{mathtools}
\usepackage{microtype}
\usepackage{times}

\usepackage{hyperref}
\usepackage{pdfsync}
\usepackage{dsfont}
\usepackage{color}

\usepackage{cite}
\usepackage{bbm}
\usepackage{bm}
\usepackage{multirow}
\usepackage[font=footnotesize,labelfont=bf]{caption}
\usepackage{bibentry}
\nobibliography*

\setlist[1]{labelindent=\parindent}
\setlist[enumerate,1]{label = \Roman*., ref = \Roman*}
\setlist[enumerate,2]{label = \Alph*., ref = \Alph*}
\setlist[enumerate,3]{label = \roman*., ref = \roman*}
\setlist[enumerate,4]{label = \alph*., ref = \alph*}

\numberwithin{equation}{section}

\newtheorem{theorem}{Theorem}[section]
\newtheorem{proposition}[theorem]{Proposition}
\newtheorem{lemma}[theorem]{Lemma}

\newtheorem{definition}[theorem]{Definition}

\theoremstyle{remark}
\newtheorem{remark}[theorem]{Remark}

\newcommand{\divv}{\mathrm{div}}
\newcommand{\G}{\mathcal{G}}

\newcommand{\R}{\mathbb{R}}
\newcommand{\M}{\mathcal{M}}

\newcommand{\ba}{\begin{array}}
\newcommand{\ea}{\end{array}}

\newcommand{\bthm}{\begin{theorem}}
\newcommand{\ethm}{\end{theorem}}
\newcommand{\bprop}{\begin{proposition}}
\newcommand{\eprop}{\end{proposition}}
\newcommand{\blemma}{\begin{lemma}}
\newcommand{\elemma}{\end{lemma}}

\newcommand{\beqn}{\begin{equation}}
\newcommand{\eeqn}{\end{equation}}

\newcommand{\EE}{\mathcal{E}}
\newcommand{\EEE}{\hat{\mathcal{E}}}

\newcommand{\beqns}{\begin{equation*}}
\newcommand{\eeqns}{\end{equation*}}

\newcommand{\X}{\mathcal{X}}
\newcommand{\Y}{\mathcal{Y}}
\renewcommand{\L}{\mathcal{L}}
\renewcommand{\leq}{\leqslant}
\renewcommand{\geq}{\geqslant}

\definecolor{mygreen}{rgb}{0.1,0.75,0.2}

\newcommand{\nc}{\normalcolor}

\newcommand{\N}{\mathbb{N}}

\newcommand{\veps}{\varepsilon}
\newcommand{\eps}{\veps}

\DeclareMathOperator{\diag}{diag}

\newcommand{\Rd}{{\mathord{\mathbb R}^d}}

\def\P{{\mathcal P}}
\definecolor{darkgreen}{rgb}{0.1,0.6,0.1}
\definecolor{darkred}{rgb}{0.6,0,0}
\definecolor{lightgray}{rgb}{0.5,0.5,0.5}

\newcommand{\1}{\mathbf{1}}
\title[From spectral clustering to mean shift]{Clustering dynamics on graphs: from spectral clustering to mean shift through Fokker-Planck interpolation}

\author{Katy Craig, Nicol\'as Garc\'ia Trillos, and Dejan Slep\v cev}
\date{\today}                                        
\begin{document}

\begin{abstract}
In this work we build a unifying framework to interpolate between density-driven and geometry-based algorithms for data clustering, and specifically, to connect the mean shift algorithm with spectral clustering at discrete and continuum levels. We seek this connection through the introduction of %a variety of notions of
 Fokker-Planck equations on data graphs. Besides introducing new forms of mean shift algorithms on graphs, we provide new theoretical insights on the behavior of the family of diffusion maps in the large sample limit as well as provide new connections between diffusion maps and mean shift dynamics on a fixed graph. Several numerical examples illustrate our theoretical findings and highlight the benefits of interpolating density-driven and geometry-based clustering algorithms. 
\end{abstract}

\maketitle

\setcounter{tocdepth}{2}
\tableofcontents

%%%%%%%%%%%%%%%%%%%%%%%%%%%%%%%%%%%%%%%%%%%%%%%%%%%%%%%%%%%%%%%%%%%%%%%%%%%%%%%%%%%%%%%%%%%%%%%%%%%%%%%% INTRODUCTION
%%%%%%%%%%%%%%%%%%%%%%%%%%%%%%%%%%%%%%%%%%%%%%%%%%%%%%%%%%%%%%%%%%%%%%%%%%%%%%%%%%%%%%%%%%%%%%%%%%%

%Colors for comments: \Ka{Katy}, \Ni{Nicolas}, \De{Dejan}

\section{Introduction}

% This work aims at contributing to the mathematical foundations of data analysis by 
In this work we establish new connections between two popular but seemingly unrelated families of methodologies used in unsupervised learning. The first family that we consider is density based and includes mode seeking clustering approaches such as the mean-shift algorithm introduced in \cite{Fukunaga75MS} and reviewed in \cite{Carreira16overview}, while
% , as well as hierarchical mthods based on identifying connected components of level sets of density estimators \cite{}; 
the second family is based on spectral geometric ideas applied to graph settings and includes methodologies such as Laplacian eigenmaps \cite{belkin2003laplacian}, and spectral clustering \cite{JordanNg01}. After discussing the mean shift algorithm in Euclidean space and reviewing a family of spectral methods for clustering on graphs, we seek these connections in two different ways. First, motivated by some heuristics at the continuum level (i.e. infinite data setting level), we take a suitable dynamic perspective and introduce appropriate interpolating \emph{Fokker-Planck} equations on data graphs. This construction is inspired by the variational formulation in Wasserstein space of Fokker-Planck equations at the continuum level and utilizes recent geometric formulations of PDEs on graphs. Second, we revisit the diffusion maps from \cite{Coifman1} (with an extended range for the parameter indexing the family) and in particular show that, when parameterized conveniently and in the large data limit, the family of diffusion maps is closely related to the same family of continuum dynamics motivating our Fokker-Planck interpolations on graphs. At the finite data level, we show that by taking an extreme value of the parameter indexing the diffusion maps we can retrieve a specific graph version of the mean shift algorithm introduced in \cite{koontz1976graph}. Our new theoretical insights are accompanied by extensive numerical examples aimed at illustrating the benefits of interpolating density and geometry driven clustering algorithms.

% Apart from introducing new methodologies for unsupervised learning, we demonstrate, through mathematical analysis and numerical experiments, how the Fokker-Planck interpolation can exploit the strengths of each individual family of algorithms and in this way overcome some of the specific drawbacks that each of them possesses. A key piece in this analysis is the derivation of formal continuum limits (i.e. infinite data setting) for our clustering methodologies. It turns out that these continuum limits coincide with those of  

% Throughout our analysis we  able to establish a direct connection between the family of diffusion maps from \cite{} and the 

% Our Fokker-Planck interpolation at the graph level can be thought of as a discrete counterpart of a Fokker-Planck equation at the continuum level (i.e. in the infinite data limit setting) of the form
% \[  \frac{\partial u}{\partial t} = \beta \nabla u  + (1-\beta)\divv(u ). \]
% Interestingly enough, after a 

% serves also as the continuum limit for the (full) family of diffusion maps from \cite{} when appropriate time scalings are taken into consideration. These observations enlarge the 

% In addition, we connect our methodology with other frameworks in the literature that have been introduced to interpolate between density and spectral based methods, the most prominent example being the diffusion maps framework from \cite{Maggioni19}. 

To begin our discussion, let us recall that unsupervised learning is one of the fundamental settings in machine learning where the goal is to find structure in a data set $\X$ without the aid of any labels associated to the data. For example, if the data set $\X$ consisted of images of animals, a standard task in unsupervised learning would be to recognize the structure of groups in $\X$ without using information of the actual classes that may be represented in the data set (e.g. \{dog, olinguito, caterpillar, \dots\}); in the literature this task is known as \emph{data clustering} and will be our main focus in this paper. Other unsupervised learning tasks include dimensionality reduction \cite{van2009dimensionality}, anomaly detection \cite{Hodge2004}, among others. 

When clusters are geometrically simple, for example when they are dense sets of points separated in space, elementary clustering methods, like $k$-means or $k$-median clustering, are sufficient to identify the clusters. However, in practice, clusters are often geometrically complex due to the natural variations that objects that belong to similar classes may have and also due to invariances that some object classes posses. To handle such data sets there is a large class of clustering algorithms, including the ones that will be explored in this paper, that are described as two-step procedures consisting of an embedding step and an actual clustering step where a more standard, typically simple, clustering method is used on the embedded data. In mathematical terms, in the first step the goal is to construct a map
\[ \Psi: \X \rightarrow \Y \]
between the data points in $\X$ and a space $\Y$ (e.g. $\Y=\R^k$ for some small $k$, or in general a metric space) to ``disentangle" the original data as much as possible; in the second step the actual clustering is obtained by running a simple clustering algorithm such as $K$-means (if the set $\Y$ is the Euclidean space, for example). While both steps are important and need careful consideration, it is in the first step, and specifically in the choice of $\Psi$, that most clustering algorithms differ from each other. For example, as we will see in section \ref{sec:MeanShift}, in some version of mean shift, $\Psi$ is induced by gradient ascent dynamics of a density estimator starting at the different data points (when the data points are assumed to lie in Euclidean space), whereas in spectral methods for clustering in graph-based settings the map $\Psi$ is typically built using the low lying spectrum of a suitable graph Laplacian. At its heart, different choices of $\Psi$ capture different heuristic interpretations of the loosely defined notion of ``data cluster". Some constructions have a density-driven flavor (e.g. mean shift), while others are inspired by geometric considerations (e.g. spectral clustering). The notions of density-based and geometry-based algorithms, however, are not monolithic, and each individual algorithm has nuances and drawbacks that are important to take into account when deciding whether to use it or not in a given situation; in our numerical experiments we will provide a series of examples that highlight some of the qualitative weaknesses of different clustering algorithms, density or geometry driven. Our main aim is to introduce a new mathematical framework to interpolate between these two seemingly unrelated families of clustering algorithms, and to provide new insights for existing interpolations like \textit{diffusion maps}.

In the rest of this introduction we present some background material that is used in the remainder.

\subsection{Mean-shift based methods}
\label{sec:MeanShift}

Let $\X= \{ x_1, \dots, x_n \}$ be a data set in $\R^d$.  One heuristic way to define data ``clusters" is to describe them as regions in space of high concentration of points separated by areas of low density. One algorithm that uses this heuristic definition is the \emph{mean shift algorithm}; see  \cite{Carreira16overview}. In essence, mean shift is a hill climbing scheme that seeks the local modes of a density estimator constructed from the observed data in an attempt to identify data clusters. 

To make the discussion more precise, let us first describe the setting where the points $x_1, \dots, x_n$ are obtained by sampling a probability distribution supported on the whole $\R^d$ with (unknown) smooth enough density $\rho : \R^d \rightarrow \R $. The first step in mean shift is to build an estimator for $\rho$ of the form:
\[ \hat{\rho}(x):= \frac{1}{n \delta^d} \sum_{i=1}^n \kappa\left( \frac{|x- x_i|}{\delta} \right), \]
where $  \kappa : \R_+ \rightarrow \R_+$ is an appropriately normalized kernel, and $\delta >0$ is a suitable bandwidth; for simplicity, we can take the standard Gaussian kernel $\kappa(s)= \frac{1}{\sqrt{2\pi}}\exp^{-s^2/2}$. Then, for every point $x_i$ one considers the iterations:
\begin{equation}  \label{eqn:MS1}
 x_i(t+1) = x_{i}(t) + \nabla  \log  \hat{\rho} (x_{i}(t))
\end{equation}
starting at time $t=0$ with $x_i(0) = x_i$. Each data point $x_i$ is then mapped to its associated $x_i(T)$ for some user-specified $T$, implicitly defining in this way a map $\Psi$ as described in the introduction. It is important to notice that mean shift is, in the way introduced above, a \emph{monotonic} scheme, i.e. $\hat{\rho}(x_i(t))$ is non-decreasing as a function of $t\in \N$. In \cite{CarreiraPerpMeanShift2} this is shown to be a consequence of a deeper property that in particular relates mean shift with the expectation-maximization algorithm applied to a closely related problem. The name mean-shift originates from the fact that iterate $x_i(t+1)$ in \eqref{eqn:MS1} coincides with the mean of some distribution that is centered around the iterate $x_i(t)$, and thus \eqref{eqn:MS1} can be described as a ``mean shifting" scheme.

%While it is possible to modify the mean shift iterations defined above by adding a time step in front of the gradient of the density estimator in \eqref{eqn:MS1}, it is worth highlighting that in general the resulting scheme is only monotonic when the time step is equal to one (i.e. as in \eqref{eqn:MS1}). Still, from these time discretization schemes it is natural to pass to continuous time and work with the gradient flow dynamics:

 We now introduce the continuum analogue of the mean shift algorithm \eqref{eqn:MS1}. Namely \eqref{eqn:MS1}
can be seen as the iterates of the Euler scheme, for time step $h=1$ of the following ODE system
\begin{equation}
\begin{cases}
\dot{x}_i(t) = \nabla \log \hat{\rho}( x_i(t)), & t >0
\\ x_i(0) = x_i.
\end{cases}
\label{eqn:flowMS}
\end{equation}
In the remainder we will abuse the terminology slightly and refer to the above continuous time dynamics as \textit{mean shift}. We note that the $\hat \rho$ is monotonically increasing along the trajectories of \eqref{eqn:flowMS}. 

To utilize the mean shift dynamics for clustering, for some prespecified time $T>0$ (that at least theoretically can be taken to be infinity under mild conditions on $\rho$), we consider the embedding map:
\[ \Psi_{MS} (x_i) :=  x_i(T), \quad x_i \in \X. \]
 When the number of data points $n$ in $\X$ is large, and the bandwidth $h$ is small enough (but not too small!), one can heuristically expect that the gradient lines of the density estimator $\hat{\rho}$ resemble those of the true density $\rho$; see for example \cite{arias-casttro16meanshift}. In particular, with an appropriate tuning of bandwidth  $\delta$ as a function of $n$, and with a large value of $T$ defining the time horizon for the dynamics, one can expect $\Psi_{MS}$ to send the original data points to a set of points that are close to the local modes of the density $\rho$. In short, mean shift is expected to cluster the original data set by assigning points to the same cluster if they belong to the same basin of attraction of the gradient ascent dynamics for the density $\rho$. 
 \medskip

If the density $\rho$ is supported on a manifold, $\M$, embedded in $\R^d$, and that information is available, one can consider mean shift dynamics restricted to the manifold. 
Indeed to define manifold mean shift we just need to consider the flow ODE \eqref{eqn:flowMS}, where $\nabla$ is replaced by the gradient on $\M$,  which for manifolds in $\R^d$ is just the projection of $\nabla$ to the tangent space $T_x \M$.  We notice that this extends to manifolds with boundary where at boundary point $x \in \partial \M$ one is projecting to the interior half-space $T^{in}_x \M$. 
 We denote the projection to the tangent space (interior tangent space at the boundary) by $P_{T\M}$ and write the resulting ODE as
 \begin{equation}
\begin{cases}
\dot{x}_i(t) = P_{T\M} \nabla \log(\hat{\rho})( x_i(t)), & t >0
\\ x_i(0) = x_i \in \M.
\end{cases}
\label{eqn:flowMSM}
\end{equation}

% \subsubsection{Generative models}
% A typical modelling assumption used in the statistics literature to motivate the use of mean shift for data clustering is that $\rho$ is the density of a mixture of Gaussians:
% \[\rho(x) \propto \sum_{l=1}^K w_l \exp\left(-\frac{1}{2}\langle \Sigma_l^{-1}  (x-m_l) ,(x-m_l) \rangle \right).   \]
% In the above, $(m_1,\Sigma_1), \dots, (m_k,\Sigma_k)$ are the means and covariance matrices of the mixture components, and $w_1, \dots, w_k$ are the weights assigned to each of the components. 

% The data points $\{x_1, \dots, x_n\}$ can be imagined to be obtained using the following sampling scheme: 1) sample a variable $y_i$ according to $\Prob(y_i= j)= w_j$. 2) After generating $y_i$ sample $x_i$ according to $x_i \sim \Prob(x_i|y_i) = N(\mu_{y_i}, \Sigma_{y_i})$. The $y_i$ variables are a useful conceptual artifact that can be interpreted as \emph{latent} variables or true classes for a given data point. 

% % In this generative setting, we can thus write the clustering problem as:
% % \[ \min_{l: \R^d \rightarrow \{ 1, \dots, K\}} \E_{(x,y)} \left[ \ell(l(x),y) \right] \]
% % where the function $\ell$ is the 0-1 loss $\ell(y,y'):= \mathds{1}_{y \not = y'}$. A straightforward computation reveals that $\ell$ is 

\subsubsection{Lifting the dynamics to the Wasserstein space}
\label{sec:MSWassSpace}

Looking forward to our discussion in subsequent sections where we introduce mean shift algorithms on graphs, it is convenient to rewrite \eqref{eqn:flowMS} in an alternative way using dynamics in the \emph{Wasserstein} space $\mathcal{P}_2(\R^d)$. As it turns out, the ODE \eqref{eqn:flowMS} is closely related to an ODE in $\mathcal{P}_2(\R^d)$ (i.e. the space of Borel probability measures over $\R^d$ with finite second moments). Precisely, we consider:
\begin{equation}
\partial _t \mu_t+ \divv( \nabla \log(\hat{\rho})\mu_t )=0 , \quad t >0,
	\label{eqn:FokkerPlanckRd}
	\end{equation}
with initial datum $\mu_0=\delta_{x_i}$; equation \eqref{eqn:FokkerPlanckRd} must be interpreted in the weak sense (see Chapter 8.1. in \cite{AGS}). Indeed, when $\mu_0=\delta_{x_i}$, it is straightforward to see that the solution to \eqref{eqn:FokkerPlanckRd} is given by $\mu_t=\delta_{x_i(t)}$ where $x_i(\cdot)$ solves the ODE \eqref{eqn:flowMS} in the base space $\R^d$. What is more, in the same way that \eqref{eqn:flowMS} can be understood as the gradient descent dynamics for $- \log(\hat{\rho})$ in the base space $\R^d$, it is possible to interpret \eqref{eqn:FokkerPlanckRd} directly as the gradient flow of the potential energy:
\begin{equation}
 \EE(\mu):= -\int_{\R^d} \log(\hat{\rho}(x)) d\mu(x), \quad \mu \in \mathcal{P}_2(\R^d)   
 \label{eqn:EnergyGradFlow}
\end{equation}
with respect to the Wasserstein metric $d_W$, which in dynamic form reads:
\begin{equation} \label{eq:BB}
   d_W^2(\nu, \tilde \nu) = \inf _{t \in [0,1] \mapsto (\nu_t, \vec{V}_t) }\int_{0}^{1}\int_{\R^{d}}|\vec{V}_t|^{2}\hspace{1mm}d\nu_{t}dt, 
\end{equation}
where the infimum is taken over all solutions $(\nu_t, \vec{V}_t)$ to the continuity equation
	\begin{equation*}\label{Conteqn}
	\partial_t \nu_t + \text{div}( \nu_t \vec{V}_t) =0,
	\end{equation*}
	with $\nu_0=\nu$ and $\nu_1=\tilde\nu$.

The previous discussion suggests the following alternative definition for the embedding map associated to mean shift:
\[ \Psi_{MS}(x_i):= \mu_{i,T} \in \mathcal{P}_2(\R^d),  \]
where in order to obtain $\mu_{i,T}$ we consider the gradient flow dynamics $\EE$ in the Wasserstein space initialized at the point $\mu_0=\delta_{x_i}$ (i.e. equation \eqref{eqn:FokkerPlanckRd}). While this new interpretation may seem superfluous at first sight given that $\mu_{i,T}=\delta_{x_i(T)}$, we will later see that working in the space of probability measures is convenient, as this alternative representation motivates new versions of mean shift algorithms for data clustering on structures such as weighted graphs; see section \ref{sec:MSAonM}.

\subsection{Spectral methods}
\label{sec:SpecMethods}

Let us now discuss another family of algorithms used in unsupervised learning that are based on ideas from spectral geometry. The input in these algorithms is a collection of edge weights $w$ describing the similarities between data points in $\X$; we let $n = |\X|$. For simplicity, we assume that the weight function $w: \X \times \X \rightarrow \R$ is symmetric and that all its entries are non-negative. We further assume that the weighted graph $\G =(\X , w)$ is connected in the sense that for every $x,x' \in \X$ there exists a path $x_0, \dots, x_m \in \X$ with $x_0=x$, $x_m = x'$ and $w(x_l, x_{l+1})>0$ for every $l=0,\dots, m-1$. At this stage we do not assume any specific geometric structure in the data set $\X$ nor on the weight function $w$ (in section \ref{sec:CotninuumLimits}, however, we focus our discussion on proximity graphs).

Let us now give the definition of well known graph analogues of gradient, divergence, and Laplacian operators. To a function $\phi: \X \rightarrow \R$ we associate a \emph{discrete gradient}, a function of the form $\nabla_\G \phi: \X \times \X \rightarrow \R$ defined by:
\[ \nabla_\G \phi(x,x') := \phi(x') - \phi(x). \]
% For a function $\phi$ on two variables $x,g$ we will often write
% \[ \nabla \phi(x,g,g'):= (\nabla_x \phi(x,g), \nabla_g \phi(x,g,g')).\]
Given a function $U: \X \times \X \rightarrow \R$ (i.e. a discrete vector field) we define its \emph{discrete divergence} as the function $\divv_\G U: \X \rightarrow \R$ given by
\[ \divv _\G \hspace{1mm} U (x) := \frac{1}{2}\sum_{x'} ( U(x',x)- U(x,x'))w(x,x').  \]
With these definitions we can now introduce the \emph{unnormalized Laplacian} associated to the graph $\G$ as the operator $\Delta_\G: L^2(\X) \rightarrow L^2(\X) $ defined according to:
\[ \Delta_\G:=\divv_\G\circ \nabla_\G,  \]
or more explicitly as
 \begin{equation}\Delta_\G u (x_i) = \sum_{j} ( u(x_i) - u(x_j) ) w(x_i, x_j), \quad x_i \in \X , \quad u \in L^2(\X). 
 \label{eqn:UnnormalizedGraphLapl}
\end{equation} 
From the representation $\Delta_\G = \divv_\G \circ \nabla_\G$ it is straightforward to verify that $\Delta_\G$ is a self-adjoint and positive semi-definite operator with respect to the $L^2(\X)$ inner product (i.e. the Euclidean inner product in $\R^n$ after identifying  real valued functions on $\X$ with $\R^n$). It can also be shown that $\Delta_\G$ has zero as an eigenvalue with multiplicity equal to the number of connected components of $\G$ (in this case $1$ by assumption); see \cite{vonLuxburg2007Tutorial}. Moreover, even when the multiplicity of the zero eigenvalue is uninformative about the group structure of the data set, the low-lying spectrum of $\Delta_\G$ still carries important geometric information for clustering. In particular, $\Delta_\G$'s small eigenvalues and their corresponding eigenvectors contain information on \emph{bottlenecks} in $\G$ and on the corresponding regions that are separated by them; the connection between the spectrum of $\Delta_\G$ and the bottlenecks in $\G$ is expressed more precisely with the relationship between Cheeger constants and Fiedler eigenvalues; see \cite{vonLuxburg2007Tutorial}. With this motivation in mind, \cite{belkin2003laplacian} introduced a nonlinear transformation of the data points known as \emph{Laplacian eigenmap}:
\[ x_i \in \X \longmapsto \left( \begin{matrix} \phi_1(x_i) \\ \vdots \\ \phi_k(x_i) \end{matrix} \right) \in \R^k, \]
where $\phi_1, \dots, \phi_k$ are the eigenvectors corresponding to the first $k$ eigenvalues of $\Delta_\G$. The above Laplacian eigenmap as well as other similar transformations serve as the embedding map in the first step in most spectral methods for partitioning and data clustering. Said clustering algorithms have a rich history, and related ideas have been present in the literature for decades, see 
\cite{JordanNg01, ShiMalik, vonLuxburg2007Tutorial} and references within. For example, some versions of spectral clustering consider a conformal transformation of the Laplacian eigenmap in which coordinates in the embedding space are rescaled differently according to corresponding eigenvalues and the choice of a time-scale parameter. More precisely, one may consider:
\[ \hat{\Psi}_{SC}(x_i) := \left( \begin{matrix} e^{-T \lambda_1} \phi_1(x_i) \\ \vdots \\ e^{-T\lambda_k}\phi_k(x_i) \end{matrix} \right) \in \R^k, \quad x_i \in \X, \]
for some $T>0$, where in the above $\lambda_l$ represents the eigenvalue corresponding to the eigenvector $\phi_l$. In section \ref{sec:MSGraphs} we provide a dynamic interpretation of the map $\hat{\Psi}_{SC}$.

\begin{remark}
There are several other ways in the literature to construct the embedding maps $\Psi$ from graph Laplacian eigenvectors. In \cite{JordanNg01}, for example, an extra normalization step across eigenvectors is considered for each data point. By introducing this extra normalization step one effectively maps the data points into the unit sphere in Euclidean space. The work \cite{Schiebinger} argues in favor of this type of normalization and proposes the use of an angular version of $k$-means clustering on the embedded data set. The work \cite{HosseiniHoffmann} also analyzes the geometric structure of spectral embeddings, both at the data level, as well as at the continuum population level. The normalization step in \cite{JordanNg01} can also be motivated from a robustness to outliers perspective if one insists on running $k$-means with the $\ell^2$ metric and not with for example the $\ell^1$ metric. As discussed in the introduction, constructing a data embedding is only part of the full clustering problem. What metric and what clustering method should be used on the embedded data are important practical and theoretical questions. In subsequent sections our embedded data points will have the form of probability vectors. Several metrics could then be used to cluster the embedded data points ($TV, L^2, W_2$, etc), each with its own advantages and disadvantages (theoretical or computational). The emphasis of our discussion in the rest of the paper, however, will be on the embedding maps themselves. We leave the analysis of the effect of different metrics used to cluster the embedded data for future work. 
\end{remark}

\subsubsection{Normalized versions of the graph Laplacian}
Different versions of graph Laplacians can be constructed to 
include additional information about vertex degrees as well as to normalize the size of eigenvalues. We distinguish between two ways to normalize the graph Laplacian $\Delta_\G$. One is based on reweighing operators and the other on renormalizing edge weights.  
\medskip

\emph{Operator-based renormalizations.} To start, we first write the graph Laplacian $\Delta_\G$ in matrix form. For that purpose, let $W = [w(x_i,x_j)]_{i,j}$ be the matrix of weights, and let 
\begin{equation}
d(x_i) = \sum_{x_j \not = x_i} w(x_i,x_j) 
\label{eqn:degree}
\end{equation}
be the weighted degrees; in the remainder we may also use the notation $d_i = d(x_i)$ whenever no confusion arises from doing so. Let $D = \diag(d_1, \dots, d_n)$ be the diagonal matrix of degrees. The Laplacian $\Delta_\G$ can then be written in matrix form as:
\begin{equation}
   \Delta_\G = D-W.
\end{equation}
In terms of this matrix representation the \emph{normalized} graph Laplacian, as introduced in  \cite{vonLuxburg2007Tutorial}, can be written as 
\begin{equation}
   L = D^{-\frac12}  \Delta_\G  D^{-\frac12} = I -D^{-\frac12} W D^{-\frac12} .
\end{equation}
Notice that the matrix $L$ is symmetric, and positive semidefinite as it follows directly from the same properties for $\Delta_\G$. The \emph{random-walk} graph Laplacian, on the other hand, is given by 
\begin{equation}
   L^{rw} =  D^{-1} \Delta_\G  = I -D^{-1} W . 
\end{equation}
% \begin{equation}
%   \Delta_{rw\G} =  \Delta_\G D^{-1} = I -W D^{-1}. 
% \end{equation}

\begin{remark}
\label{rem:EigenvecRWLapl}
It is straightforward to show that the matrix $L^{rw}$ is similar to the matrix $L$ and thus the random walk Laplacian has the same eigenvalues as the normalized graph Laplacian. Moreover, if we explicitly use the representation $L^{rw^T} = D^{1/2} L D^{-1/2}$, where $L^{rw^T}$ is the transpose of $L^{rw}$, we can see that if $\tilde u$ is an eigenvector of $L$ with eigenvalue $\lambda$, then $\phi = D^{1/2}\tilde \phi$ is an eigenvector of $L^{rw^T}$ with eigenvalue $\lambda$. In particular, since $L$ is symmetric, we can find a collection of vectors $\phi_1, \dots, \phi_n \in \R^n$ that form an orthonormal basis (with respect to the inner product $\langle D^{-1} \cdot ,\cdot \rangle $) for $\R^n$ and where each of the $\phi_l$ is an eigenvector for $L^{rw^T}$.

We use the above observation in section \ref{sec:MoreGeneralSpectralEmbeddings} and later at the beginning of section \ref{sec:MSGraphs}. 
\end{remark}

% We note that the random walk graph Laplacian  this article is given by the adjoint of the matrix used in  \cite{vonLuxburg2007Tutorial}. The reason is that the graph configurations are considered as column vectors we apply the operator from the left (which we think is more convenient for considering continuum limits). % On the other hand when studying Markov chains configurations are often given as row vectors and the operator matrix is applied from the right. 

\medskip

\emph{Edge weights renormalizations.} Here the idea is to adjust the weights of the graph $\G=(\X, w)$ and use one of the Laplacian normalizations introduced before on the new graph.  One of the most popular families of graph Laplacian normalizations based on edge reweighing was introduced in \cite{Coifman1} and includes the generators for the so called \emph{diffusion maps} which we now discuss. 

For a given choice of parameter $\alpha \in (-\infty,1]$, we construct new edge weights, $w_\alpha(x,y)$, as follows:
\[ w_\alpha(x,y) := \frac{w(x,y)}{d(x)^\alpha  d(y)^\alpha } , \]
where recall $d$ is the weighted vertex degree. On the new graph $(\X, w_\alpha)$ one can consider all forms of graph Laplacian discussed earlier. In the sequel, however, we follow \cite{Coifman1} and restrict our attention to the reweighed random-walk Laplacian which in matrix form can be written as
%\[ \Delta_\alpha^{rw} = I - W_\alpha D_\alpha^{-1}, \]
\[ L_\alpha^{rw} = I - D_\alpha^{-1} W_\alpha , \]
where $W_\alpha$ is the matrix of edge weights of the new graph and $D_\alpha$ its associated degree matrix, i.e. $d_\alpha(x) = \sum_{y \not = x} w_\alpha(x,y)$.
We note that the weight matrix $W_\alpha$ is still symmetric.  

Let $Q^{rw}_\alpha$ be the \textit{weighted diffusion rate} matrix
\[ Q^{rw}_\alpha := - C_\alpha L^{rw}_\alpha \]
where $C_\alpha$ is a positive constant that we introduce for modeling purposes. In particular, in section \ref{sec:CotninuumLimits} we will see that, in the context of proximity graphs on data sampled from a distribution on a manifold $\M$, by choosing the constant $C_\alpha$ appropriately we can ensure a desirable behavior of $Q^{rw}_\alpha$ as the number of data points in $\X$ grows. Notice that we may alternatively write $Q^{rw}_\alpha$ as a function:
\begin{align}
\label{Qrw} 
Q^{rw}_\alpha(x,y) &:= {C_\alpha} \; \begin{cases} \frac{w_\alpha(x,y)}{\sum_{z \not =x } w_\alpha(x,z)} &\text{ if } x \neq y , \\
 - 1 &\text{ if } x =y. \end{cases} 
%  Q^{un}(x,y) &:= \begin{cases}  w_\alpha(x,y) &\text{ if } x \neq y , \\
%  -\sum_{z \in \mathcal{X}} w_\alpha(x,z)  &\text{ if } x =y , \end{cases} \label{Qun}
\end{align}

\begin{remark}
In this paper we consider the range $\alpha \in (-\infty,1]$ and not $[0,1]$ as usually done in the literature. One important point that we stress throughout the paper is that by considering the interval $(-\infty, 1]$ we obtain an actual interpolation between density-based (in the form of some version of  mean shift) and geometry-driven clustering algorithms. 
\end{remark}

\subsubsection{More general spectral embeddings.} 
\label{sec:MoreGeneralSpectralEmbeddings}

Following Remark \ref{rem:EigenvecRWLapl}, for a given $\alpha \in (-\infty,1]$ we consider an orthonormal basis (relative to the inner product $\langle D_\alpha^{-1} \cdot , \cdot \rangle$)   $\phi_1, \dots, \phi_n$ of eigenvectors of $L_\alpha^{rw^T}$ with corresponding eigenvalues $\lambda_1 \leq \lambda_2 \leq \dots  \leq \lambda_n$, and define
\begin{equation}
\label{eqn:SpectralEmbeddingDiffusionMaps}
\hat{\Psi}'_{\alpha}(x_i):=  \left( \begin{matrix} e^{-T \lambda_1} \tilde \phi_1(x_i)  \\ \vdots \\ e^{-T\lambda_k} \tilde \phi_k(x_i)  \end{matrix} \right) \in \R^k, \quad x_i \in \X,
\end{equation}
where in the above, $\tilde{\phi}_l = D_\alpha^{-1/2}\phi_l$. We can see that the map $\hat{\Psi}_\alpha'$ has the same form as the map $\hat{\Psi}_{SC}$ at the beginning of section \ref{sec:SpecMethods}. In section \ref{sec:MSGraphs} we provide a more dynamic interpretation of the map $\hat{\Psi}'_\alpha$.

% while $Q^{un}$ is the (negative) unnormalized graph Laplacian.

% \begin{equation}
%  \sum_{\G} \divv_\G(h)(x) \phi(x)  = - \sum_{x,x'}h(x,x')\nabla_\G \phi(x,x')w(x,x').
%  \end{equation}
% \nc

% Discrete gradients and discrete divergences are related to each other via a discrete integration by parts formula.  Namely, a straightforward computation shows that for every $h: \X \times \X \rightarrow \R$ and $\phi: \X \rightarrow \R$ it holds

% \red 
% In the above $\divv_x$ is the divergence operator in $\R^d$, $\nabla_x$ the gradient operator, and $\Delta_x$ the \emph{Laplacian} operator $\Delta_x:=\divv_x\circ \nabla_x$. 
% 	In general, equation \eqref{eqn:FokkerPlanckRd} must be interpreted in weak form.
% \nc

\subsection{Outline}

Having discussed the mean shift algorithm in Euclidean setting (or in general on a submanifold $\M$ of $\R^d$), as well as some spectral methods for clustering in the graph setting, in what follows we attempt to build bridges between geometry based and density driven clustering algorithms. Our first step is to introduce general data embedding maps $\Psi$ associated to the dynamics induced by arbitrary rate matrices on $\X$. We will then define data graph analogues of mean shift dynamics. We do this in section \ref{sec:MSGraphs} where we define a new version of mean shift on graphs inspired by the discussion in section \ref{sec:MSWassSpace}, and review other versions of mean shift on graphs such as Quickshift \cite{vedaldi2008quick} and KNF \cite{koontz1976graph}. In section \ref{sec:FokkerPlanckGraphsInt} we discuss two versions of Fokker-Planck equations on graphs which serve as interpolating dynamics between geometry and density driven dynamics for clustering on data graphs. One version is based on a direct interpolation between diffusion and mean shift (the latter one as defined in section \ref{sec:MSGraphsWass}) and is inspired by Fokker-Planck equations at the continuum level. The second version is an extended version of the diffusion maps from \cite{Coifman1} obtained by appropriate reweighing and normalization of the data graph. In section \ref{sec:DiffMpasKNF} we show that the KNF mean shift dynamics can be seen as a particular case of the family of diffusion maps when the parameter indexing this family is sent to negative infinity. This result is our first concrete connection between mean shift algorithms and spectral methods for clustering.  In section \ref{sec:CotninuumLimits} we study the continuum limits of the Fokker-Planck equations introduced in section \ref{sec:FokkerPlanckGraphs} when the graph of interest is a proximity graph. This analysis will allow us to provide further insights into diffusion maps, mean shift, and spectral clustering. In section \ref{sec:Numerics} we present a series of numerical experiments aimed at illustrating some of our theoretical insights.

% As a first step, in section \ref{}, we discuss evolution equations in the space of probability measures on a data set $\X$. We will then introduce a version of the mean shit algorithm on graphs motivated by the discussion in subsection \ref{}. After carefully rewriting the... 

% In section \ref{} we investigate the connection between the Fokker-Planck equation on graphs that we discuss  here...

\section{Mean shift and Fokker-Planck dynamics  on graphs}
\label{sec:MSGraphs}

Consider a weighted graph $\G=(\X, w)$ as in section \eqref{sec:SpecMethods}. Let $\P(\mathcal{X})$ denote the set of probability measures on $\X$ which we identify with $n$-dimensional vectors. All of the dynamics we consider can be written as (continuous time) Markov chains on graphs. 
 
 \begin{definition}
 \label{def:MarkovChain}
 A \emph{continuous time Markov chain} $u: [0,T] \to \P(\mathcal{X})$ is a solution to the ordinary differential equation
 \begin{align} \label{Markovchain}
     \begin{cases}
     \partial_t {u}_t(y) = \sum_{x \in \mathcal{X}} u_t(x) Q(x,y) , \quad t >0 \\
     u_0 = u^0 
     \end{cases}
 \end{align}
 where $u^0 \in \P(\mathcal{X})$ and $Q: \mathcal{X} \times \mathcal{X} \to \mathbb{R}$ is a  \emph{transition rate matrix} $Q$; that is, $Q$ satisfies  $Q(x,y) \geq 0$ for $y \neq x$ and $Q(x,x) = -\sum_{y \neq x} Q(x,y)$.
 
Notice that in terms of matrix exponentials, the solution to \eqref{Markovchain} can be written as
 \begin{align*}
u_t(y) = \sum_{x \in \mathcal{X}} u^0(x) {(e^{t Q})(x,y)}.
 \end{align*}
 \end{definition}

 \begin{remark}
 \label{rem:RowVectors}
 The operation on the right hand side of the first equation in \eqref{Markovchain} can be interpreted as a matrix multiplication of the form $u_t Q$, where $u_t$ is interpreted as a row vector. Alternatively, we can use the transpose of $Q$ and write $Q^T u_t$ if we interpret $u_t$ as a column vector. 
 \end{remark}
 
 \begin{remark}[Conservation of mass and positivity]
 For any transition rate matrix, we have 
 \[\sum_{x \in \mathcal{X}} Q(x,y) =0, \quad \forall y \in \X,\]
 which ensures that $\sum_x u_t(x) = \sum_x u_0(x) = 1$ for all $t \geq 0$.
 Note that, in practice, this is often accomplished by specifying the \emph{off-diagonal} entries of the transition rate matrix, $Q(x,y)$ for $ x \neq y$, and then setting the diagonal entries to equal opposite the associated diagonal degree matrix, $d(x,x) = \sum_{y\neq x} Q(x,y) $. Likewise, for any transition rate matrix, the fact that $Q(x,y) \geq 0$ for $y \neq x$ ensures that if $u_0(x) \geq 0,$ then $u_t(x) \geq 0$ for all $t \geq 0$.
 \label{rem:ConservationMass}
 \end{remark}
 
 \begin{remark}
 Notice that $Q= - \Delta_\G$, where we recall $\Delta_\G$ is defined in \eqref{eqn:UnnormalizedGraphLapl}, is indeed a rate matrix and thus induces evolution equations in $\mathcal{P}(\X)$. Likewise, the weighted diffusion rate matrix $Q^{rw}_\alpha$ from \eqref{Qrw} is a rate matrix as introduced in Definition \ref{def:MarkovChain}. 
 \end{remark}

%  \begin{remark}[transition probabilities]
%  Let $P_t(x,y)$ denote the transition probabilities of the Markov chain, which represent the probability a particle starting at position $x$ at time zero reaches position $y$ at time $t$. Then, $P_t$ solves the \emph{forward equation}
%  \[ \begin{cases} \dot{P}_t(x,y) = \sum_{z \in \mathcal{X}} P_t(x,z) Q(z,y)  \\
%  P_0(x,y) = I(x,y) \end{cases}. \]
%  where $I$ denotes the identity matrix. Equivalently, $P_t$ may be expressed in terms of the matrix exponential $P_t = e^{t Q}$.
 
%  The solution of the continuous time Markov chain (\ref{Markovchain}) may be obtained in terms of the transition probabilities via the equation
%  \begin{align*}
% u_t(y) = \sum_{x \in \mathcal{X}} u_0(x) P_t(x,y)
%  \end{align*}
%  \end{remark}
 
%  \begin{remark}[transition kernel]
%  The corresponding Markov \emph{transition kernel} is given by
%  \[ K(x,y) = Q(x,y) + I(x,y) , \]
% {\color{blue} perhaps talk about irreducible (I think this is that all the entries of $\lim_{t \to +\infty} e^{tQ}$ are positive), reversible here}
%  \end{remark}

For a general rate matrix $Q$, we define the data embedding map:
\begin{equation}
   \hat{\Psi}_Q (x_i) := u_{i,T,Q}\in \mathcal{P}(X), \quad x_i \in \X,
   \label{eqn:QEmbedding}
\end{equation}
where $u_{i,T,Q}$ represents the solution of \eqref{Markovchain} when the initial condition $u^0 \in \mathcal{P}(\X)$ is defined by $u^0(x)=1$ for $x=x_i$ and $u^{0}(x)=0$ otherwise. In the sequel, and specifically in our numerics section, we will use the $L^2(\X)$ metric between the points $\hat{\Psi}_{Q}(x_i)$ in order to build clusters through $K$-means regardless of the rate matrix $Q$ that we use to construct the embedding $\hat{\Psi}_Q$.
Remember that by $L^2(\X)$ distance we mean the quantity:
\[ \sum_{x'} ( u_{i,T,Q}(x') - u_{j,T,Q}(x')  )^2.  \]

In the next subsections we discuss some specifics of the choice $Q= Q^{rw}_\alpha$ and then introduce two classes of rate matrices $Q$ that give meaning to the idea of mean shift on graphs.

\subsection{Dynamic interpretation of spectral embeddings}
\label{sec:EmbeddingRWLap}
When we take $Q=Q^{rw}_\alpha$ we abuse notation slightly and write $\hat{\Psi}_{\alpha}$ instead of $\hat{\Psi}_{Q_\alpha^{rw}}$ and $u_{i,T,\alpha}$ instead of $u_{i,T, Q^{rw}_\alpha}$. In the next proposition we make the connection between the embedding map $\hat{\Psi}_{\alpha}$ and the spectral embedding $\hat{\Psi}_\alpha'$ from \eqref{eqn:SpectralEmbeddingDiffusionMaps} explicit. 

\begin{proposition}
For every $\alpha \in (-\infty,1]$ we can write 
\[ u_{i,T,\alpha}(x)= \sum_{l=1}^n  e^{-T\lambda_l}\frac{\phi_l(x_i)}{(d_\alpha(x_i))^{1/2}} \phi_l(x), \quad \forall x \in \X, \]
where the $\phi_1, \dots, \phi_{n}$ form an orthonormal basis for $\R^n$ (with respect to the inner product $\langle D^{-1}_\alpha \cdot , \cdot \rangle$) and each $\phi_l$ is an eigenvector of $L^{rw^T}_\alpha $ with eigenvalue $\lambda_l$. In other words, the coordinates of the vector $\hat{\Psi}_\alpha'(x_i)$ correspond to the representation of $\hat{\Psi}_\alpha(x_i)$ in the basis $\phi_1, \dots, \phi_n$.
\end{proposition}

\begin{proof}
This follows from a simple application of the spectral theorem using Remark \ref{rem:EigenvecRWLapl}.
\end{proof}

\begin{remark}
An alternative way to compare the points $\hat{\Psi}_{\alpha}(x_i)$ and $\hat{\Psi}_{\alpha}(x_j)$ is to compute their weighted distance:
\[ \sum_{x'} \left( \frac{u_{i,T,Q}(x')}{\sqrt{d_\alpha(x')}} - \frac{u_{j,T,Q}(x')}{\sqrt{d_\alpha(x')}}  \right)^2.  \]
This construction is introduced in \cite{Coifman1} and is referred to as \textit{diffusion distance}. It is worth mentioning that, as pointed out in \cite{Coifman1}, the diffusion distance between $\hat{\Psi}_\alpha(x_i)$ and $\hat{\Psi}_{\alpha}(x_j)$ coincides with the Euclidean distance between $\hat{\Psi}_{\alpha}'(x_i)$ and $\hat{\Psi}'_\alpha(x_j)$ when $k=n$. Notice that the diffusion metric is conformal to the $L^2(\X)$ metric. 
\end{remark}

\begin{remark}
We remark that the embedding maps $\hat{\Psi}_{SC}$ and $\hat{\Psi}_{Q}$ for $Q=-\Delta_\G$ are connected in a similar way as the maps $\hat{\Psi}_{\alpha}'$ and $\hat{\Psi}_\alpha$ are. Indeed, if we let $\phi_1, \dots,\phi_n $ be an orthonormal basis for $L^2(\X)$ consisting of eigenvectors of $\Delta_\G$ (remember that $\Delta_\G$ is positive semi-definite with respect to $L^2(\X)$), then the coordinates of $\hat{\Psi}_{SC}(x_i)$ are precisely the coordinates of the representation of $u_{i,T,Q}$ in the basis $\phi_1, \dots, \phi_n$.

\end{remark}

\subsection{The mean shift algorithm on graphs}
\label{sec:MSAonM}

In the next two subsections we discuss two different ways to introduce mean shift on $\G= (\X, w)$. In both cases we define an associated rate matrix $Q$.

\subsubsection{Mean shift on graphs as inspired by Wasserstein gradient flows}
\label{sec:MSGraphsWass}

The discussion in section \ref{sec:MSWassSpace}
shows that the mean shift dynamics can be viewed as a gradient flow in the spaces of probability measures endowed with Wasserstein metric. Recent works \cite{esposito2019nonlocal,Maas11gradient} provide a way to consider Wasserstein type gradient flows which are restricted to graphs. This allows one to take advantage of the information about the geometry of data that their initial distribution provides. More importantly for our considerations it allows one to combine the mean shift and spectral methods.

 The notion of Wasserstein metric on graphs introduced by Maas \cite{Maas11gradient} provides the desired framework. 
Here we will consider the upwind variant of the Wasserstein geometry on graphs introduced in
\cite{esposito2019nonlocal} since it avoids the problems that the metric of \cite{Maas11gradient} has when dealing with the continuity equations on graphs, see Remark 1.2 in \cite{esposito2019nonlocal}.

In particular we actually  consider a \textit{quasi-metric} on $\mathcal{P}(\X)$ (and not a metric) defined by: 
\[ \hat{d}_W^2(v, \tilde v) := \inf _{t \in [0,1] \mapsto (v_t, V_t) } \int_{0}^{1}\sum_{x,y}|V_t(x,y)_+|^{2}\hspace{1mm}\overline{v}_t(x,y)w(x,y)  dt, \]
where the infimum is taken over all solutions $(v_t, V_t)$ to the discrete continuity equation:
	\begin{equation*}\label{Conteqn}
	\partial_t v_t + \text{div}_\G( \overline{v}_t \cdot V_t) =0,
	\end{equation*}
	with $v_0=v$ and $v_1=\tilde v$ and where $V_t$ is anti-symmetric for all $t$. In the above, the constant $C_{ms}>0$ is introduced for modeling purposes and will become relevant in section \ref{sec:CotninuumLimits}. We use the upwinding interpolation \cite{chow2018entropy,esposito2019nonlocal}:
\begin{align}
\overline{v}_t(x,y) := \begin{cases} v_t(x) & \text{ if } V_t(x,y) \geq 0 , \\
v_t(y) &\text{ if } V_t(x,x) < 0, \end{cases}
\end{align}
and interpret the discrete vector field $\overline{v}_t \cdot V_t$ as the elementwise product of $\overline{v}_t$ and $V_t$. Finally, we use $a_+$ to denote the positive part of the number $a$.

Next, we consider the general  potential energy:
\begin{equation*}
 \EEE(u):= -\sum_{x \in \X} B(x) u(x), \quad u \in \mathcal{P}(\X),
\end{equation*}
for some $B : \X \rightarrow \R$. This energy serves as an analogue of \eqref{eqn:EnergyGradFlow}.

Following the analysis and geometric interpretation in \cite{esposito2019nonlocal}, it is possible to show that the gradient flow of $\EEE$ with respect to the quasi-metric $\hat{d}_W$ takes the form:
\begin{align} \label{Bgradascent}
\partial _t {u}_t(y) = \sum_{x \in \mathcal{X}} u_t(x) Q^{B}(x,y),
\end{align}
where $Q^{B}$ is the rate matrix defined by
\begin{equation} \label{eq:QB-def}
    Q^{B}(x,y) :=  \begin{cases} (B(y)-B(x))_+ w(x,y) , & \text{ for } x \neq y, \\
-\sum_{z \neq x} (B(z)-B(x))_+ w(x,z) , & \text{ for } x =y. \end{cases}
\end{equation}

% discrete version of the continuity equation   

% \begin{align*}
% \dot{u}_t(y) + \sum_{x \in \mathcal{X}} v_t(y,x) w(x,y) \hat{u}_t(y,x)  = 0 .
% \end{align*}

% Suppose $v_t(x,y) = -v_t(y,x)$, so in particular $v_t(x,x) =0$. In this case, we may rewrite the continuity equation as a continuous time Markov chain in the following manner,
% \begin{align*}
% \dot{u}_t(y)& = - \sum_{x \in \mathcal{X}} (v_t(y,x))_+ w(x,y) u_t(y)  - (v_t(y,x))_- w(x,y) u_t(x)  \\
% & = - \sum_{x \neq y} (v_t(y,x))_+ w(x,y) u_t(y)  - (v_t(x,y))_+ w(x,y) u_t(x)  \\
% &= \sum_{x \in \mathcal{X}} u_t(x)Q(x,y)  ,
% \end{align*}
% where $Q$ denotes for the transition rate matrix,
% \begin{align*}
% Q(x,y) &= \quad
% \begin{cases}  
% (v_t(x,y))_+ w(x,y) &\text{ if } x \neq y ,\\
% - \sum_{z \neq x} (v_t(y,z))_+ w(z,y) &\text{ if } x = y , 
% \end{cases} \quad = \quad \begin{cases}  
% (v_t(y,x))_- w(x,y) &\text{ if } x \neq y ,\\
% - \sum_{z \neq x} (v_t(z,y))_- w(z,y) &\text{ if } x = y . 
% \end{cases}
% \end{align*}

% The gradient ascent of a potential energy $\mathcal{B}(u) = \sum_{x \in \mathcal{X}} B(x) u(x)$ corresponds to the choice of velocity $v_t(x,y) = \bar{\grad} B(x,y) = B(y) - B(x)$, in which case the transition rate is given by 

%\begin{remark} \label{rem:PotentialBForlog}
In section \ref{sec:ContLimitMS} we explore the connection between the graph mean shift dynamics \eqref{Bgradascent} and the mean shift dynamics on an $m$-dimensional submanifold of $\R^d$ as introduced in section \ref{eqn:FokkerPlanckRd}. This is done in the context of proximity graphs over a data set $\X= \{ x_1, \dots, x_n\}$  obtained by sampling a distribution with density $\rho$ on $\M$. In particular we formally  show that the graph mean shift dynamics converges to the continuum one if
\begin{align*} \label{KDEeqn}
  {B}(x) = - \frac{C_{ms}}{  \rho(x)}, \quad x \in \M.
  \end{align*}
In practice, however, since $\rho$ is in general unavailable, $\rho$ above can be replaced with a  density estimator $\hat \rho$.
Given such an estimator $\hat \rho$ (which in principle can be considered on general graphs, not just ones embedded in $\R^d$) we define the transition kernel for the \emph{graph mean shift}  dynamics as
\begin{equation} \label{eq:Qms-def}
    Q^{ms}(x,y) :=  C_{ms} \begin{cases} \left(-\frac{1}{\hat \rho(y)} + \frac{1}{\hat \rho(x)} \right)_+ w(x,y) , & \text{ for } x \neq y, \\
-\sum_{z \neq x} \left(-\frac{1}{\hat \rho(z)}+\frac{1}{\hat \rho(x)} \right)_+ w(x,z) , & \text{ for } x =y \end{cases}
\end{equation}
where the constant $C_{ms}>0$ just sets the time scale. It will be specified in Section \ref{sec:ContLimitMS} so that the equation has the desired limit as $n \to \infty$. 
For data in $\R^d$ it is natural to consider a kernel density estimator $\hat \rho$. In particular in all of our experiments  we consider the following kernel density estimator:
 \begin{equation} \label{KDEkernel} 
  \hat{\rho}_\delta(x):=  \frac{1}{n} \sum_{y \in \mathcal{X}} \psi_\delta(x-y)  , \quad \psi_\delta(x) = \frac{1}{(2 \pi)^{m/2} \delta^m} e^{-|x|^2/(2 \delta^2)}.
\end{equation} 
 For an abstract graph $\mathcal{G}$ (i.e. $\X$ is not necessarily a subset of Euclidean space) one can consider the degree of the graph at each $x \in \X$ as a substitute for $\hat{\rho}$ in the above expressions.

From the previous discussion and in direct analogy with the discussion in section \ref{sec:MSWassSpace} we introduce the data embedding map:
\begin{equation}
\hat{\Psi}_{ms}(x_i):= u_{i,T, Q^{ms}} \in \mathcal{P}(\X).
\end{equation}

\begin{remark}
The quasi-metric $\hat{d}_W$ and the geometry of the PDEs on graphs that it induces have been studied in \cite{esposito2019nonlocal}. One important point made in that paper (see Remark 1.2. in \cite{esposito2019nonlocal}) is that the support of the solution to the induced equations may change and move as time increases. For us, this property is essential as we initialize our graph mean shift dynamics at Diracs located at each of the data points.
\end{remark}

%%% MS experiments

We now provide a couple of illustrations comparing the mean shift dynamics \eqref{eqn:flowMS} and the graph mean shift dynamics defined by \eqref{eq:Qms-def}. 
We consider data on manifold $\M = [0,4] \times  \{0, 0.7\}$. The measure $\rho$ has uniform density on the two line segments. We consider 280  data points sampled from $\rho$, Figure 
\ref{fig:inic} and sampled from $\rho$ with Gaussian noise of variance $0.1$ in vertical direction, Figure  \ref{fig:noisy}(a).
 
 We compare the dynamics on $\M$ for different bandwidths $\delta$ of the kernel density estimator. In particular we consider a value of $\delta$ that is small enough for the strips to be seen as separate and a value of delta that is large enough for the strips to be considered together,  Figure \ref{fig:clear}. For large $\delta$ we see  rather different behavior of the two dynamics. The standard mean shift quickly mixes the data from the two lines and the information about the two clusters is lost. On the other hand, while the driving force is  the same, in the graph mean shift the dynamics is restricted to the data, thus preventing the mixing. In particular, separate modes are identified in each clump. 
 
 We note that this desirable behavior is somewhat fragile when noise is present, \ref{fig:noisy}(b). In particular, the roughness of the boundary prevents the mass to reach the mode. We will discuss later that this is mitigated by adding a bit of diffusion to the dynamics, see Section \ref{sec:BlueSky}.

\begin{figure}
\centering
\includegraphics[width=0.5\textwidth]
%{figures/clear_blue_sky_begin_sigma_1over_4}
{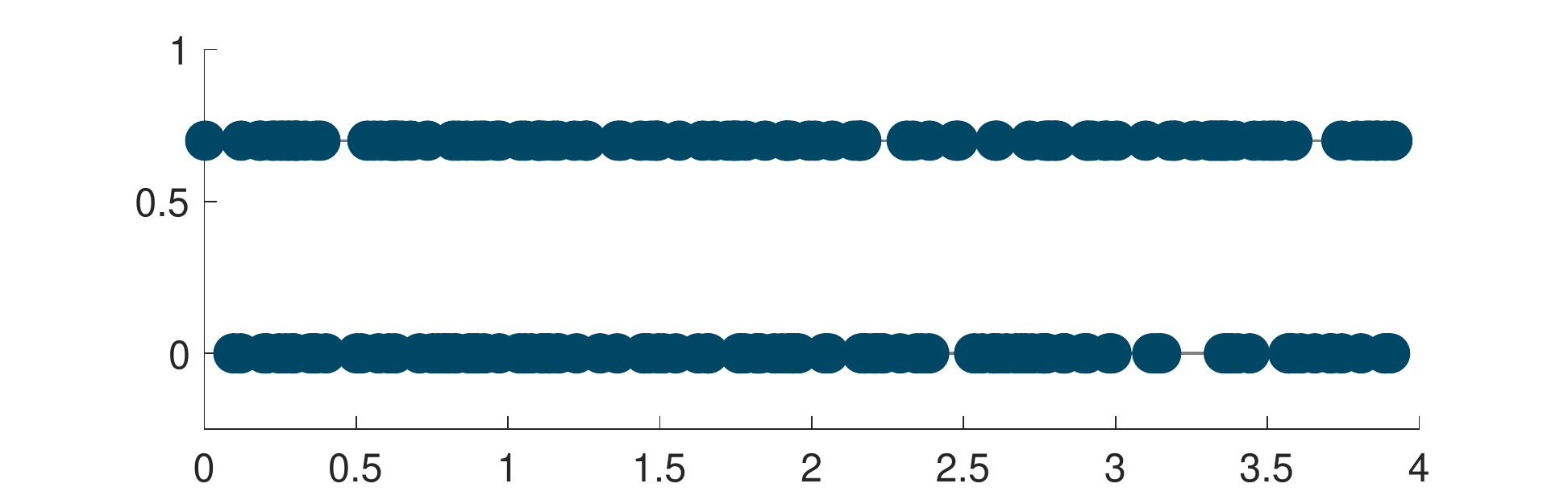}
\caption{Initial data for the  experiments below. There are 280 points sampled from a uniform distribution on two line segments. }
\label{fig:inic} 
\end{figure}

\begin{figure}[h]
\centering
\hspace*{-20pt}
\subfloat[Mean shift at intermediate time ]{\includegraphics[width=0.38\textwidth]{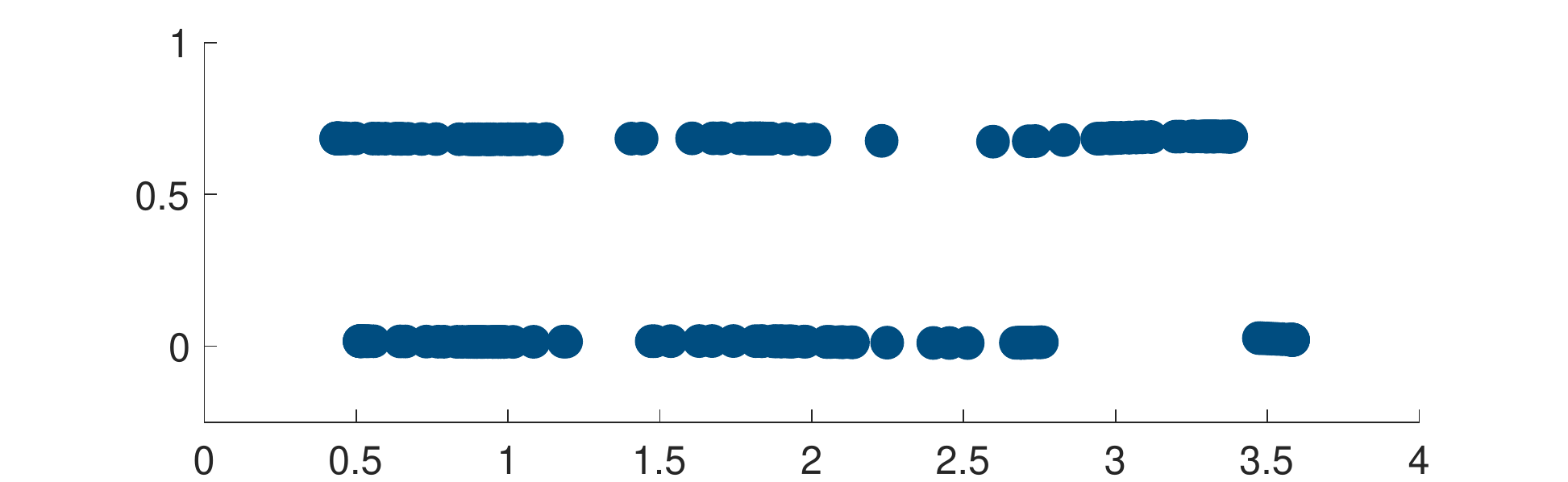}}
%{figures/clear_blue_sky_middle_sigma_1over_4_MSODE}}
\hspace*{-20pt}
\subfloat[Mean shift at long time]{\includegraphics[width=0.38\textwidth]{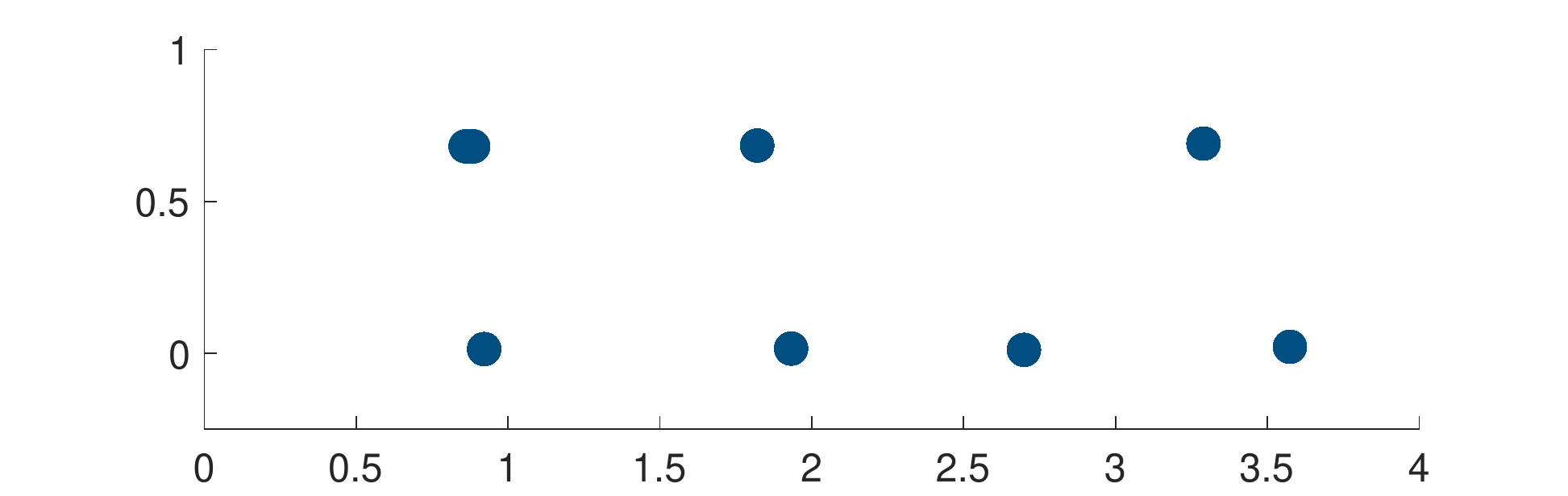}}
% {figures/clear_blue_sky_end_sigma_1over_4_MSODE}} 
\hspace*{-20pt}
\subfloat[Graph mean shift at long time. Brightness indicates mass.]{\includegraphics[width=0.38\textwidth]
{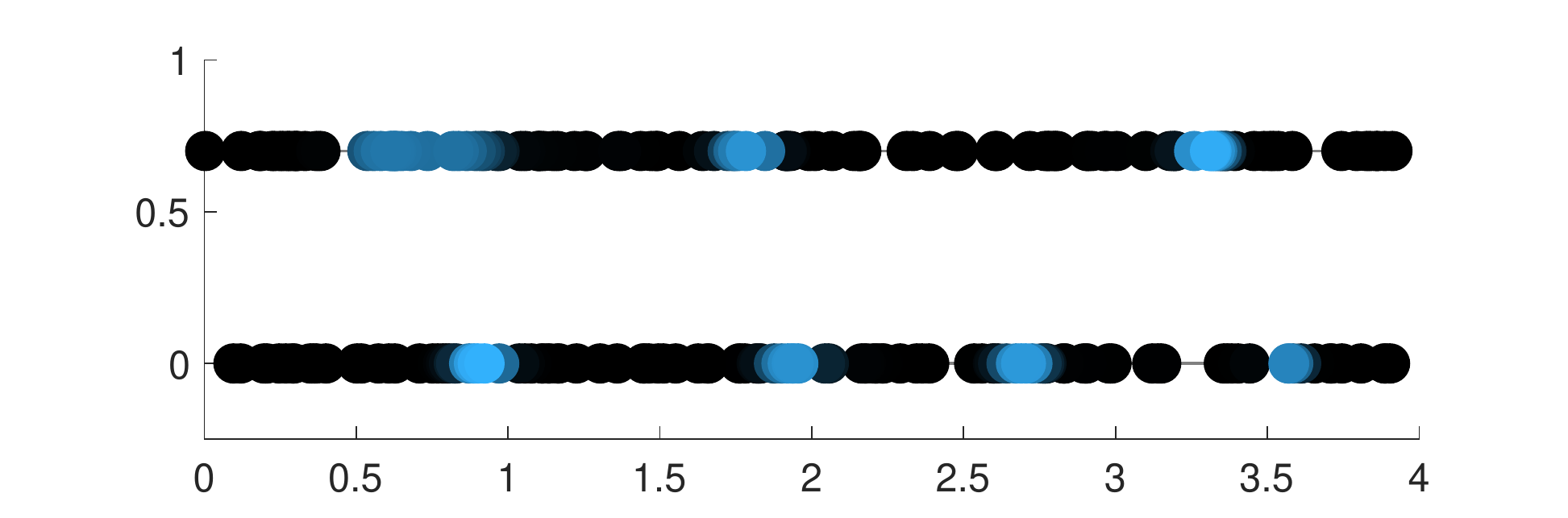}} \\
% {figures/clear_blue_sky_end_sigma_1over_4}} \\
\hspace*{-20pt}
\subfloat[Mean shift at $t=0$. Arrows represent velocity]{\includegraphics[width=0.38\textwidth]
{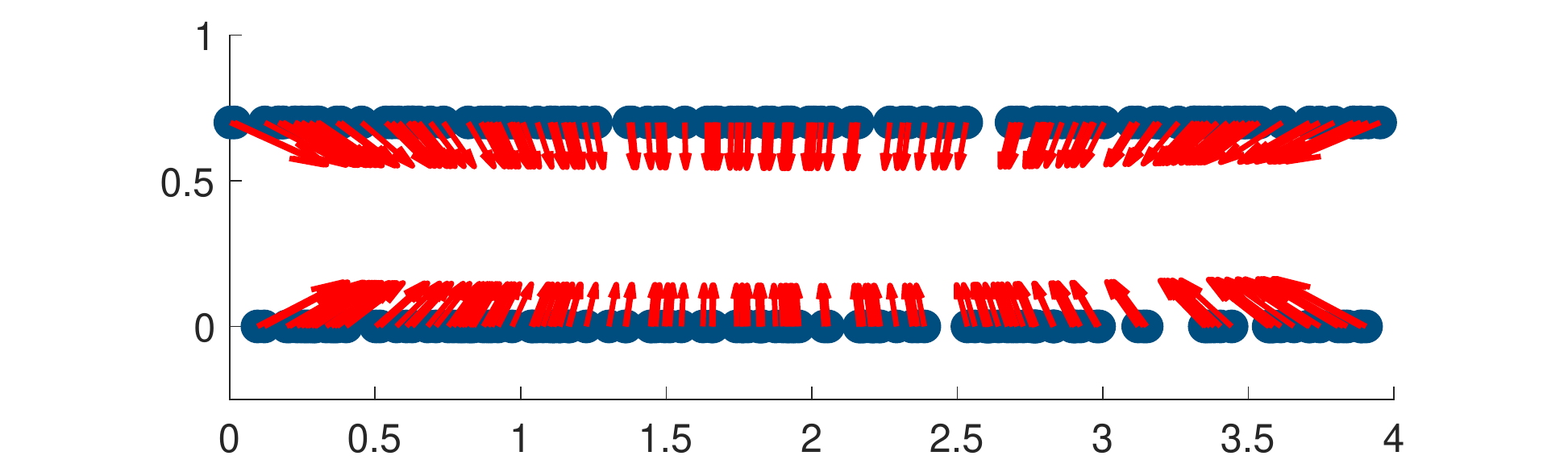}} \hspace*{-20pt}
%{figures/clear_blue_sky_begin_sigma_1over_root2_MSODE}}
\subfloat[Mean shift at intermediate time]{\includegraphics[width=0.38\textwidth]
{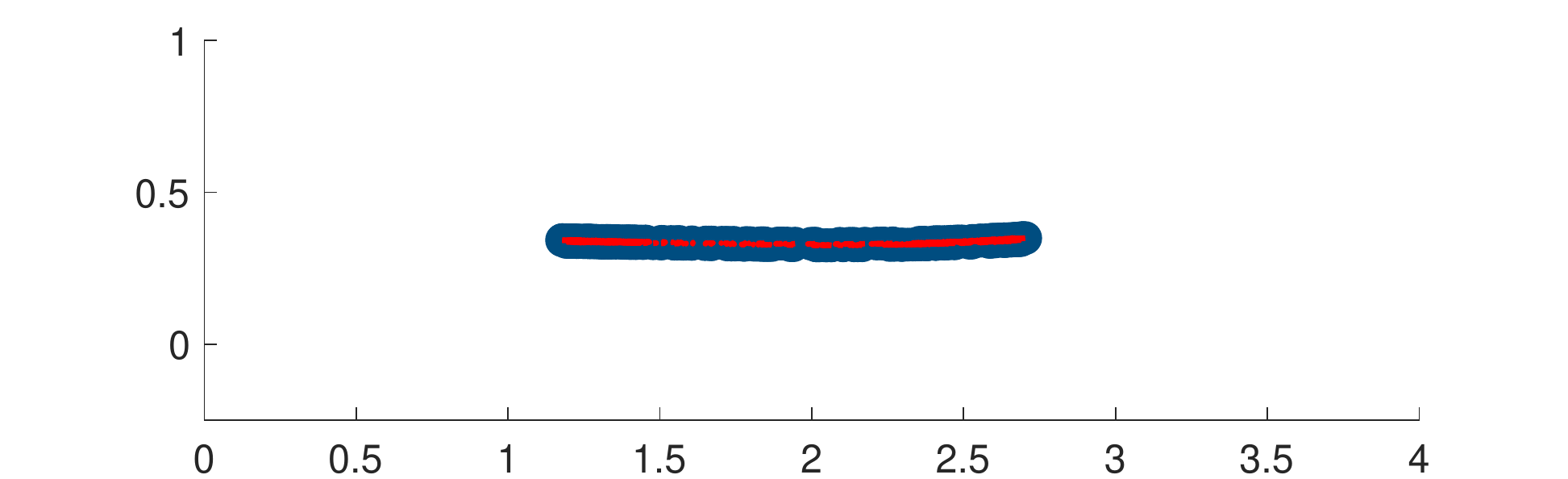}} \hspace*{-20pt}
%{figures/clear_blue_sky_middle_later_sigma_1over_root2_MSODE}} 
\subfloat[Graph mean shift at long time]{\includegraphics[width=0.38\textwidth]
%{figures/clear_blue_sky_end_sigma_1over_root2}} \\
{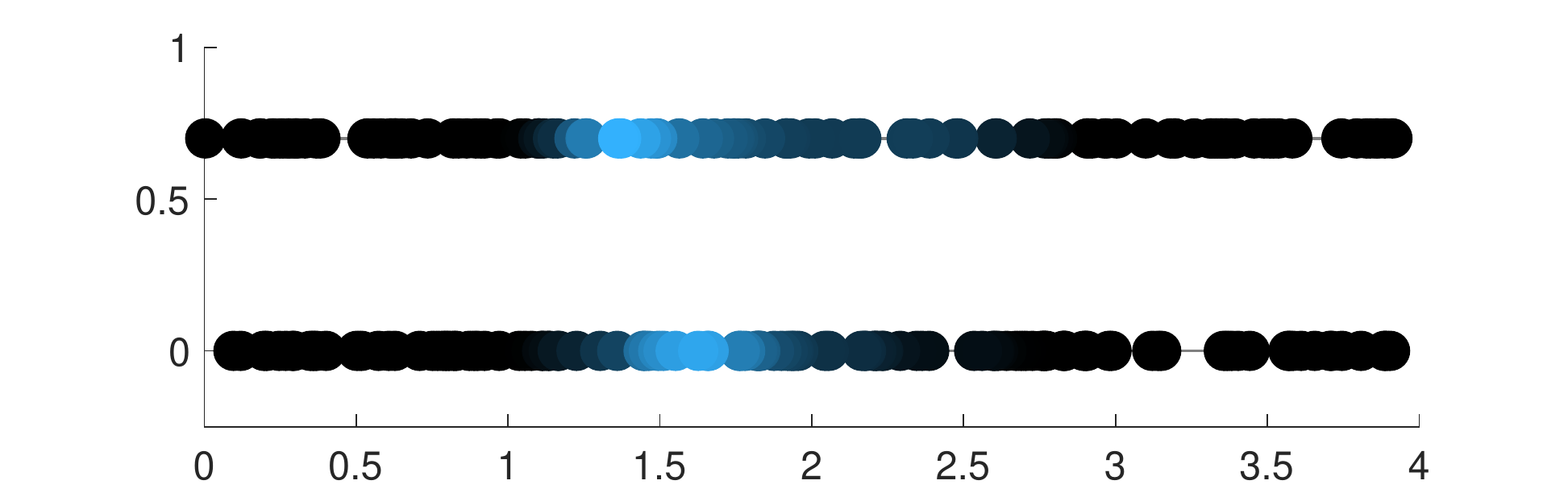}} \\
\caption{We compare the dynamics for the mean shift \eqref{eqn:flowMS} and graph mean shift \eqref{Bgradascent}-\eqref{eq:Qms-def}. The top row shows the dynamics for $\delta= \frac{1}{4 }$ bandwidth of the KDE. Both approaches give similar results. The stripes evolve independently and there are spurious local maxima due to randomness. The bottom row shows the dynamics for a larger $\delta= \frac{1}{\sqrt{2} }$. The KDE has a unique maximum. Mean shift quickly mixes the stripes into one, which then collapses to a point. On the other hand, since graph mean shift dynamics is constrained to the sample points the stripes do not mix and a single mode is identified in each stripe.}
\label{fig:clear} 
\end{figure}

\begin{figure}[h]
\centering
\hspace*{-18pt}
\subfloat[$t=0$]{\includegraphics[width=0.55\textwidth]
%{figures/noisy_blue_sky_start_sigma_1over_root2d}}
{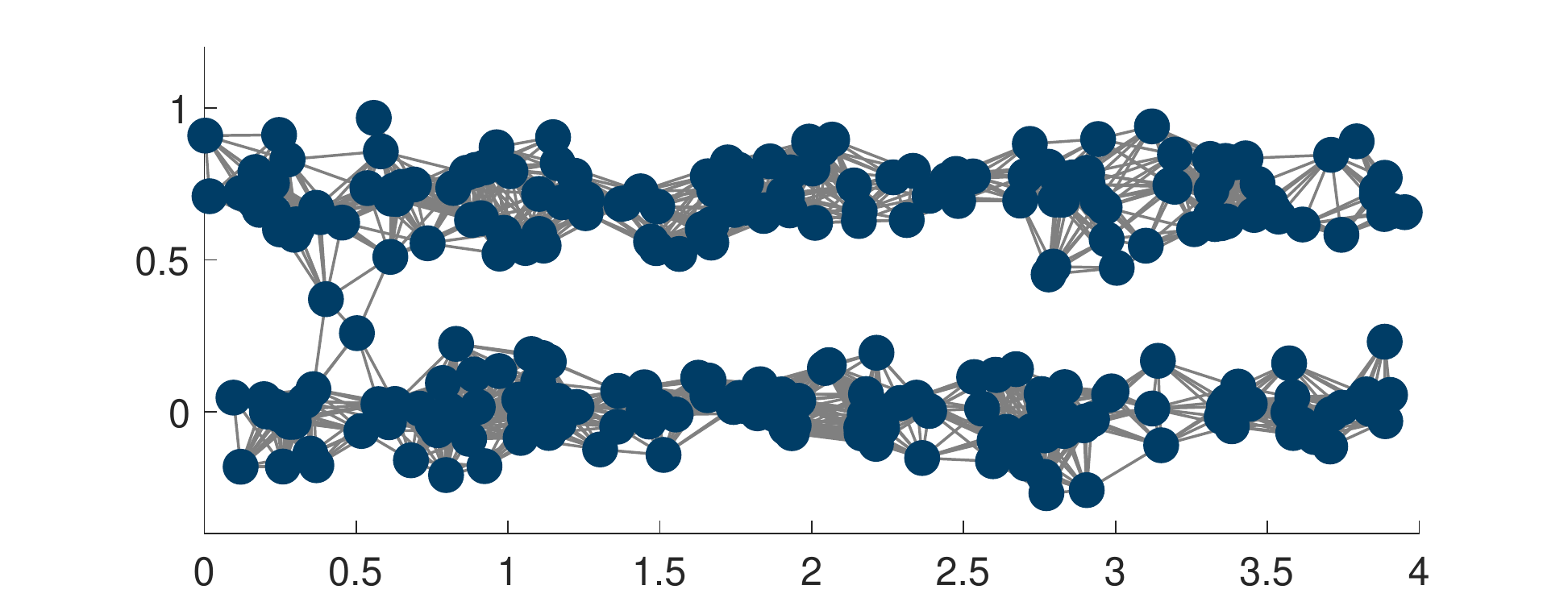}} \hspace*{-20pt}
\subfloat[$t= \infty$]{\includegraphics[width=0.55\textwidth]
{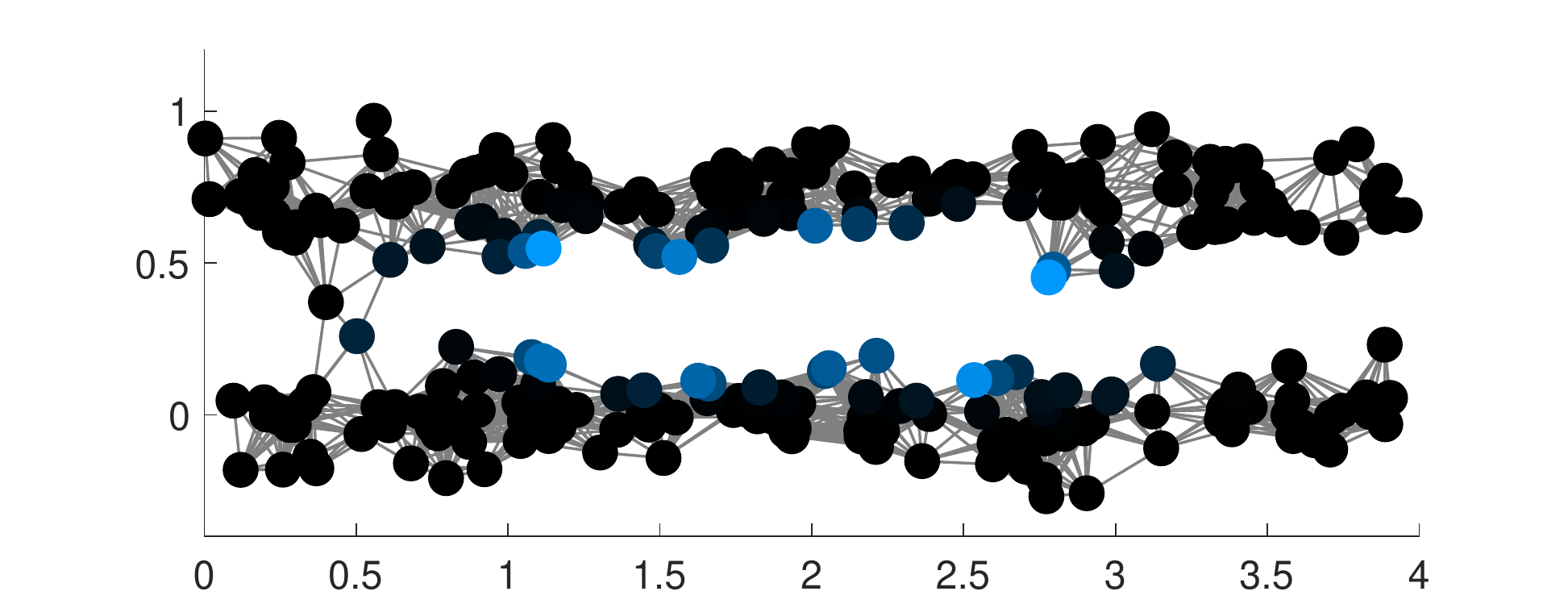}}  \\
%{figures/noisy_blue_sky_end_sigma_1over_root2}}  \\
\caption{Graph mean shift for  $\delta= \frac{1}{\sqrt{2} }$. If noise is added to the data above, most of the dynamics behave as before. The exception is shown. The graph mean shift does not reach the modes as on Figure \ref{fig:clear}(f). Namely due to geometric roughness of the data the dynamics gets trapped at blue points.}
\label{fig:noisy} 
\end{figure}
%%% end MS experiments

\subsubsection{Quickshift and KNF}
\label{sec:Quickshift}
% There are some alternatives to mean shift that are defined at the level of a graph on the data set.  Although they have some benefits in terms of computational time and exhibit good clustering performance in the same settings where the mean shift algorithm on graphs introduced in \cite{} succeeds, their mathematical structure can not be merged easily with that of the spectral based methods discussed in sections \ref{}. We thus discussed them here for completeness even though they are not used in the remainder.

There are some alternative definitions of mean shift on graphs that are popular in the literature. One such algorithm is Quickshift \cite{vedaldi2008quick}, which is similar to an earlier algorithm by Koontz, Narendra, and Fukunaga \cite{koontz1976graph}. Both algorithms can be described as hill-climbing iterative algorithms for the maximization of a potential function $\hat{B}$.

Let $\hat{B}: \X \rightarrow \R$ be the potential for which we want to define ``gradient ascent dynamics" along the graph $(\X, w)$. Let $\hat{D}(x,y) \geq 0$ be a notion of ``distance" between points $x$ and $y$ which is typically defined through the weights $w$. Both the Quickshift and KNF algorithms have a Markov Chain interpretation that we describe in a general form that allows for the (unlikely) existence of non-unique maximizers of $\hat{B}: \X \rightarrow \R$ around a given node $x \in \X$. To describe the associated rate matrices let us define for every $x \in \X$ the sets 
\[M_{QS,x}:= \left\{ y \in \X \::\: y \textrm{ maximizes: } \frac{1}{\hat D(x,y)}  \1_{\hat B(y)>\hat B(x) } \right\}\] 
and 
\[M_{KNF,x}:= \left\{ y \in \X \::\: y \textrm{ maximizes: } \frac{(\hat B(y)-\hat B(x))_+}{\hat D(x,y)}  \1_{\hat D(x,y)<r } \right\}.\]
The Quickshift and KNF algorithms are then the paths in the Markov chains with rate matrices:
\begin{align} \label{QQS}
 Q_{QS}(x,y) &=  \begin{cases} \frac{1}{\sharp M_{QS,x}}  &\text{ if } y \in M_{QS,x}, \\  
 - 1 & \text{ if } y=x, \\
 0 &\text{ otherwise, } \end{cases}  \\
  Q_{KNF}(x,y) &=  
  \begin{cases} \frac{1}{\sharp M_{KNF,x}}  &\text{ if } y \in M_{KNF,x}, \\  
 - 1 & \text{ if } y=x, \\
 0 &\text{ otherwise, } \end{cases}  
  \label{QKNF}
\end{align}
respectively. In section \ref{sec:DiffMpasKNF} we establish a connection between the family of rate matrices $Q_{\alpha}^{rw}$ and $Q_{KNF}$.

\section{Fokker-Planck equations on graphs}
\label{sec:FokkerPlanckGraphs}

\subsection{Fokker-Planck equations on graphs via interpolation}
\label{sec:FokkerPlanckGraphsInt}

The first type of interpolation between density and geometry driven clustering algorithms that we discuss in this paper is based on a direct interpolation of the rate matrices $Q^{ms}$ and $Q^{rw}_1$. Namely, for $\beta\in [0,1]$ we consider: 
\begin{align} \label{eq:upwindFPgraph}
     Q_{\beta} := \beta Q^{ms} + (1- \beta) Q^{rw}_1.
\end{align}
It is straightforward to see that the resulting $Q_{\beta}$ continues to be a rate matrix and as such it induces dynamics in the space $\mathcal{P}(\X)$. We can then use the framework from section \ref{sec:MSGraphs} and abuse notation slightly to write $\hat{\Psi}^\beta$ instead of $\hat{\Psi}_{Q_\beta}$ as well as $u_{i,T,\beta}$ instead of $u_{i,T, Q_\beta}$. 

The choice of rate matrix $Q_{\beta}$ is motivated by the Fokker-Planck equation:
\[ \partial_t f_t = \beta \divv( \nabla \phi f_t ) + (1-\beta) \Delta f_t  \]
on a submanifold $\M$ of $\R^d$, which in the context of section \ref{sec:CotninuumLimits} can be proved to be a formal continuum limit of the evolution induced by $Q_\beta$ as the number of data points grows. On the other hand, we notice that when we take $\beta=1$ in $Q_\beta$ we recover the mean shift dynamics from section \ref{sec:MSGraphsWass}. If on the contrary we set $\beta=0$, we obtain the dynamics induced by the rate matrix $Q^{rw}_1$, which, at least in the context of section \ref{sec:CotninuumLimits}, can be shown to be connected in the large sample limit to the heat equation on a manifold $\M$ where the data density plays no role. 

% This is the first type of interpolation between density-driven and geometry-based dynamics for clustering using Fokker-Planck equations on graphs. 

% \Ka{To me, $\Psi_T^\beta(x_i)$ would be better notation, more like standard semigroup notation. I think we also need this notation for the $Q^{rw}_\alpha$ operator. Can we write $\Phi_T^\alpha(x_i)$ for that? If it were up to me, I would call these operators something like $S^{ms,\beta}_T(x_i)$ and $S^{wd,\alpha}_T(x_i)$, so it looks like a semigroup operator.}

% \begin{theorem}[Lyapunov functional] Let $\beta \in (0,1]$. The dynamics 
% \begin{equation}
% \dot{u}= (Q_\beta)^T u 
% \label{eqn:FokkerPlanckDiscrete}
% \end{equation}
% has a unique steady state. Moreover, the functional
% \[\EE_\beta() = \sum_{x\in \X} \log() \]
% is a Lyapunov functional for \eqref{eqn:FokkerPlanckDiscrete}.
% \end{theorem}

\subsection{Fokker-Planck equation on graphs via reweighing and connections to graph mean shift}
\label{sec:DiffMpasKNF}

Another interpolation between density-driven and geometry-based clustering dynamics is induced by the family of rate matrices $\{ Q_{\alpha}^{rw} \}_{\alpha \in (-\infty, 1]}$. Indeed, in section \ref{sec:ContLimitsFP} we prove that in the proximity graph setting, the discrete dynamics associated to the rate matrices $Q_\alpha^{rw}$ are closely related, in the large data limit, to the same family of Fokker-Planck equations at the continuum level mentioned in section \ref{sec:FokkerPlanckGraphsInt}. What is more, without taking a large sample limit, we see that the family $\{ Q_{\alpha}^{rw} \}_{\alpha \in (-\infty, 1]}$ interpolates between $Q^{rw}_1$ and a rate matrix inducing graph mean shift dynamics, only that this time the version of mean shift that is meaningful is a particular case of the KNF formulation from section \ref{sec:Quickshift}. We prove this in the next proposition.

% Density information is also introduced via reweighing introduced in \cite{Coifman1} and described in Section. In particular we consider the weights 
% \[ w_\alpha(x,y) := \frac{w(x,y)}{d(x)^\alpha  d(y)^\alpha } , \]
% and the associate random walk Laplacian
% \begin{align*}
% Q^{rw}_\alpha(x,y) &:= C_\alpha \begin{cases} \frac{w_\alpha(x,y)}{\sum_{z \in \mathcal{X}} w_\alpha(x,z)} &\text{ if } x \neq y , \\
%  - 1 &\text{ if } x =y, \end{cases} 
% %  Q^{un}(x,y) &:= \begin{cases}  w_\alpha(x,y) &\text{ if } x \neq y , \\
% %  -\sum_{z \in \mathcal{X}} w_\alpha(x,z)  &\text{ if } x =y , \end{cases} \label{Qun}
% \end{align*}

\begin{proposition}
Let $(\X , w)$ be an arbitrary weighted graph satisfying the conditions at the beginning of section \ref{sec:SpecMethods}. Set $C_\alpha=1$ for every $\alpha \in (-\infty,1]$. Then,
\begin{equation}
    \lim_{\alpha \to -\infty } Q^{rw}_\alpha = 
    Q^{rw}_{-\infty}
\end{equation}
where 
\begin{equation}
   Q^{rw}_{-\infty}(x,y) :=  -\1_{y=x} + 
  \begin{cases} \frac{w(x,y)}{\sum_{z \in M_{KNF,x} } w(x,z)} &\text{ if } y \in M_{KNF,x} \\  
 0 &\text{ otherwise, } \end{cases}
 \end{equation}
 where in the definition of $M_{KNF,x}$ we are using $\hat{B}(z) = d(z)$, $\hat{D}(x,y)= 1$ if $w(x,y)>0$ and $\hat{D}(x,y) =\infty$ if $w(x,y)=0$, and $r>1$. 
\end{proposition}
We notice that this is essentially the KNF rate matrix defined in \eqref{QKNF} with only a difference in the way ties are broken when the maximum of $d$ around a point is not unique. This distinction is mostly irrelevant since generically we may expect no ties. On the other hand, if for some reason there are ties but the non-zero weights in the graph are equal, then the two tie-breaking rules coincide.
\begin{proof}
As the cases are analogous, let us consider only the case $y \neq x$. Note that 
\[ 
Q^{rw}_\alpha(x,y) =  \frac{w_\alpha(x,y)}{\sum_{z \not = x} w_\alpha(x,z)}
= \frac{w(x,y) d(x)^{-\alpha} d(y)^{-\alpha}}
{\sum_{z \not = x} w(x,z) d(x)^{-\alpha} d(z)^{-\alpha}} = \frac{w(x,y)  d(y)^{-\alpha}}
{\sum_{z \not = x} w(x,z)  d(z)^{-\alpha}}.
\]
If $y \not\in M_{KNF,x} $, consider $z \in M_{KNF,x}$. Then
\[ 
Q^{rw}_\alpha(x,y) \leq \frac{w(x,y)}{w(x,z)} \left( \frac{d(y)}{d(z)} \right)^{-\alpha} \to 0 \quad \text{ as } \alpha \to -\infty.
\]
If $y \in M_{KNF,x}$  then 
\[ 
Q^{rw}_\alpha(x,y)= \frac{w(x,y) }
{\sum_{z \not = x} w(x,z)  \left( \frac{d(z)}{d(y)} \right)^{-\alpha}}  \: \to \, 
\frac{w(x,y)}{\sum_{z \in M_{KNF,x} } w(x,z)} \quad \text{ as } \alpha \to -\infty. \]
\end{proof}

%%%%%%%%%%%%%%%%%%%%%%%%%%%%%%%%%%%%%%%%%%
\section{Continuum limits of Fokker-Planck equations on graphs and implications}
\label{sec:CotninuumLimits}

In this section we further study the Fokker-Planck equations introduced in Section \ref{sec:FokkerPlanckGraphs} and discuss their connection with Fokker-Planck equations at the continuum level. For such connection to be possible we impose additional assumptions on the graph $\G=(\X, w)$. In particular, we assume that $\G$ is a \emph{proximity graph} on $\X = \{ x_1, \dots, x_n \}$, where the $x_i$ are assumed to be i.i.d. samples from a distribution on a smooth compact $m$-dimensional manifold without boundary $\M$ embedded in $\R^d$, and having density $\rho: \M \rightarrow \R$ with respect to the volume form on $\M$. By proximity graph we mean that the weights $w(x_i, x_j)$ are defined according to:
\begin{equation} 
w(x,y):=  \eta_{\veps}(|x-y|) , \quad \eta_{\veps}(r) = \frac{1}{\veps^{m}} \eta\left( \frac{r}{\veps} \right), 
\label{eqn:ProximityGraph}
\end{equation}
where $\veps>0$ is a bandwidth appropriately scaled with the number of samples $n$, $\eta$ is a function $\eta : [0, \infty) \rightarrow [0,\infty)$ with compact support, and $|x-y |$ denotes the Euclidean distance between $x$ and $y$. 

% For simplicity and concreteness we restrict our attention to the choice
% \[ \eta(a) = \begin{cases} 1 & \text{ if } a <1 \\ 0 & \text{otherwise,}   \end{cases} \]
% but the discussion below is valid for any choice of $\eta$ with compact support. 

\subsection{Continuum limit of mean shift dynamics on graphs}
\label{sec:ContLimitMS}

In order to formally derive the large sample limit of equation \eqref{Bgradascent}, we study the action of $Q^{ms}$ on a smooth function $u: \M \rightarrow \R$. That is, we compute:
\begin{align*} 
\sum_{x \in \mathcal{X}} u(x) Q^{ms}(x,y) &  = -C_{ms} \sum_{x \in \X} \left[ \left( B(x)- B(y)   \right)_+ u(y) - (B(x) - B(y))_- u(x) \right] w(x,y) %\\
%(-Q^{ms} u)(y) &= \sum_{x \in \mathcal{X}} u (x) (-Q^{ms}(x,y)) =   \sum_{x \in \mathcal{X}} \left[ (B(x)-B(y))_+  u(y)  - (B(x)-B(y))_-   u(x) \right] w(x,y) , \label{Qmsaction}\\ 
\end{align*}
as $n \rightarrow \infty$ and $\veps \rightarrow 0$ at a slow enough rate.
 Since our goal below is to deduce formal continuum limits, we will assume that $\M$ is flat.
 We note that when $\M$ is a smooth manifold, the deflection of the manifold from the tangent space is at most quadratic, and thus the error introduced is small when $\varepsilon$ is small.
 In this way we can avoid using the notation and constructions from differential geometry as well as some approximation arguments that obscure the reason why the limit holds. Providing a rigorous argument for the convergence of the dynamics remains an open problem.
% @Katy and @Nicolas, it is nice to give people opportunity to get credit for some potential future work. 

In what follows we use $\rho_n = \frac{1}{n} \sum_{x \in \mathcal{X}} \delta_x$ to denote  the empirical distribution on the data points; here we use the notation $\rho_n$ to highlight the connection between the data points and the density function $\rho$. We also consider the constants
\[ C_{ms} = \frac{1}{n \veps^2 \sigma_\eta'}, \quad \sigma_\eta' = \frac{1}{2m} \int_{\R^m} |z|^2 \eta(|z|) dz, \]
and assume that the potential $B$ is a $C^3(\M)$ function. With the above definitions we can explicitly write: 
\begin{equation} \label{Qms-scaled}
-\sum_{x \in \mathcal{X}} u(x) Q^{ms}(x,y)  =  \frac{1}{n \veps^{m+2} \sigma_\eta'} \sum_{x \in \X} \left[ \left( B(x)- B(y)   \right)_+ u(y) - (B(x) - B(y))_- u(x) \right] \eta\left(\frac{|x-y|}{\veps} \right), \end{equation}
and using the smoothness of $u$ and $B$ equate the above to
\begin{align*}
&   =  \frac{1}{ \veps^{m+2} \sigma_\eta'}\int_{\M} \left[ \left( B(x)- B(y)   \right)_+ u(y) - (B(x) - B(y))_- u(x) \right] \eta\left(\frac{|x-y|}{\veps} \right) \rho_n(x) dx\\
& = \frac{1}{ \veps^{m+2} \sigma_\eta'}\int_{\M} \left[ \left( B(x)- B(y)   \right)_+ (u(y)-u(x)) \right.   \\
& \quad \left. + \left(\left( B(x)- B(y)   \right)_+ -(B(x) - B(y))_- \right) u(x) \right] \eta\left(\frac{|x-y|}{\veps} \right) \rho_n(x) dx\\
& = \frac{1}{ \veps^{m+2} \sigma_\eta'}\int_{\M} \left[ \left( B(x)- B(y)   \right)_+ (u(y)-u(x)) +  \left( B(x)- B(y)   \right) u(x) \right] \eta\left(\frac{|x-y|}{\veps} \right) \rho_n(x) dx\\
&   \approx   \frac{1}{\veps^{m+2} \sigma_\eta'}\int_{\M} \left[ \left( B(x)-B(y)  \right)_+ \langle \nabla u (y), y-x \rangle  + \langle \nabla B(y),x-y \rangle u(x) \right]\eta\left(\frac{|x-y|}{\veps} \right) \rho_n(x) dx
\\& \quad    + \frac{1}{2 \veps^{m+2}\sigma_\eta'}\int_{\M} \left[ \left( B(x)-B(y)    \right)_+ \langle D^2 u (y) (x-y), y-x \rangle  \right. \\
& \quad \left. + \langle D^2 B(y) (x-y),x-y \rangle u(x) \right] \eta\left(\frac{|x-y|}{\veps} \right) \rho_n(x) dx
\\& =:A_1+ A_2 + A_3+A_4 .
\end{align*}
Next we analyze each of the terms $A_1, A_2, A_3$ and $A_4$. 
For $A_1$ we see that:
\begin{align}
\label{auxeqn:A1}
\begin{split}
 A_1 & \approx \frac{1}{\veps^{m+2} \sigma_\eta'}\int_{\M}  \left( B(x)-B(y)  \right)_+ \langle \nabla u (y), y-x \rangle \eta\left(\frac{|x-y|}{\veps} \right)\rho(x)  dx
 \\&\approx -\frac{1}{ \veps^{m+2} \sigma_\eta'}\int_{  \langle x-y , \nabla B(y) \rangle\geq 0} \langle \nabla u (y), x-y \rangle  \langle x-y , \nabla B(y) \rangle \eta\left(\frac{|x-y|}{\veps} \right) \rho(x)   dx  
 \\& \approx - \frac{\rho(y)}{ \veps^{m+2} \sigma_\eta'}\int_{  \langle x-y , \nabla B(y) \rangle\geq 0} \langle \nabla u (y), x-y \rangle  \langle x-y , \nabla B(y) \rangle \eta\left(\frac{|x-y|}{\veps} \right) dx  
 \\& = -\frac{\rho(y)}{\sigma_\eta'}\int_{\langle z, v \rangle \geq 0 } \langle v', z\rangle \langle z, v\rangle \eta(|z|) dz,
 \end{split}
\end{align}
where $v= \nabla B(y)$ and $v'=\nabla u(y)$; notice that in the first line we have replaced the empirical measure $\rho_n$ with the measure $\rho(x) dx$ (introducing some estimation error) and in the second line we have considered a Taylor expansion of $B$ around $y$. On the other hand, notice that 
\[ \int_{ \langle z, v \rangle \geq 0 } \langle v', z\rangle \langle  z, v\rangle \eta(|z|) dz = \langle S v ,v' \rangle \]
where $S$ is a rank one symmetric matrix which can be written as $S= a \zeta \otimes \zeta$ for some vector $\zeta$ and some scalar $a$. Now, $a \langle \zeta, v \rangle^2 $ is equal to:
\[ \langle S v , v \rangle  =  \int_{\langle z, v \rangle \geq 0 } \langle v, z\rangle^2\eta(|z|) dz  = |v|^2\frac{1}{2m}  \sum_{l=1}^m\int \langle e_l , z \rangle^2 \eta(|z|)dz = |v|^2 \frac{1}{2m} \int |z|^2\eta(|z|) dz = \sigma_\eta' |v|^2. \]
The above computation shows that $\zeta$ can be taken to be $v$, and $a= \frac{\sigma'_\eta}{|v|^2}.$ Thus, 
\begin{equation}
\label{auxeqn:A1Final}
   A_1 \approx -\rho(y)  \left \langle  \nabla u (y),\nabla B(y) \right \rangle. 
\end{equation}
% here we interpret $\frac{\nabla \rho(x)}{|\nabla \rho (x)|^2} =0$ if $\nabla \rho (x)=0$. 

Regarding $A_2$, we have
\[ A_2 \approx \frac{1}{\veps^{m+2} \sigma_\eta'} \int_{\M} \langle \nabla B (y) , x-y \rangle u(x) \eta\left( \frac{|x-y|}{\veps} \right) \rho(x) dx, \]
introducing an estimation error to replace the integration with respect to the empirical measure with integration with respect to the measure $\rho(x) dx$. We can further decompose the computation introducing an approximation error:
\[ A_2 \approx A_{21} + A_{22} + A_{23} \]
where 
\begin{align*}
A_{21} & := \frac{1}{\veps^{m+2} \sigma_\eta'} \int_{\M} \langle \nabla B(y) , x-y \rangle \langle \nabla u(y)  , x-y \rangle \eta\left( \frac{|x-y|}{\veps} \right) \rho(y) dx,    \\
 A_{22} & := \frac{1}{\veps^{m+2} \sigma_\eta'} \int_{\M} \langle \nabla B(y) , x-y \rangle u(y) \eta\left( \frac{|x-y|}{\veps} \right) \rho(y) dx, \\
 A_{23} & := \frac{1}{\veps^{m+2}\sigma_\eta'} \int_{\M} \langle \nabla B(y) , x-y \rangle \langle \nabla \rho(y) , x-y \rangle u(y) \eta\left( \frac{|x-y|}{\veps} \right) dx.
 \end{align*}
By symmetry, the term $A_{22}$ is seen to be equal to zero. On the other hand, the terms $A_{21}$ and $A_{23}$ are computed similarly to the second expression in \eqref{auxeqn:A1} only that in this case there is no sign constraint in the integral. From a simple change of variables we can see that for arbitrary vectors $v$ and $v'$ we have
\[ \int_{ \langle z, v \rangle \geq 0 } \langle v', z\rangle \langle z, v\rangle \eta(|z|) dz = \int_{ \langle z, v \rangle \leq 0 } \langle v', z\rangle \langle z, v\rangle \eta(|z|) dz.  \]
In particular
\[\int \langle v', z\rangle \langle z, v\rangle \eta(|z|) dz = 2 \int_{\langle z, v \rangle \geq 0 } \langle v', z\rangle \langle z, v\rangle \eta(|z|) dz = 2 \sigma'_\eta \langle v ,v' \rangle, \]
and thus
\begin{align*}
 A_{21} & = 2\rho(y) \langle \nabla B(y) , \nabla u(y) \rangle,\quad A_{23}  = 2u(y) \langle \nabla B(y) , \nabla \rho(y) \rangle. 
\end{align*}
In summary,
\begin{equation}
\label{auxeqn:A2Final}
    A_2 \approx 2\rho(y) \langle \nabla B(y) , \nabla u(y) \rangle +  2 u(y) \langle \nabla B(y) , \nabla \rho(y) \rangle. 
\end{equation}

\nc

It is straightforward to see that $A_3 = O(\veps)$ and so for our computation we can treat $A_3$ as zero:
\begin{equation}
\label{auxeqn:A3Final}
A_3 \approx 0.
\end{equation}

For the final term $A_4$ we start by introducing an estimation error to write
\[A_4 \approx \frac{1}{2\veps^{m+2}\sigma_\eta'} \int_{\M} \langle D^2 B(y) (y-x),x-y \rangle u(x)  \eta\left(\frac{|x-y|}{\veps} \right) \rho(x) dx.\]
We can further replace the term  $u(x)$ with $u(y)$ (and $\rho(x)$ with $\rho(y)$) in the formula above. This replacement introduces an $O(\veps)$ term that we can ignore. It follows that:
\begin{align*}
A_ 4 &\approx u(y) \rho(y)\frac{1}{2\veps^{m+2} \sigma_\eta'} \int_{\M} \langle D^2 B(y) (x-y),x-y \rangle  \eta\left(\frac{|x-y|}{\veps} \right) dx  
\\&=u(y) \rho(y)\frac{1}{2 \sigma_\eta'} \int \langle D^2 B(y)z,z \rangle  \eta(|z|)dz  
% \\&= u(y) \rho(y)\frac{1}{2\veps^{m+2} \sigma_\eta'} \int_{0}^\veps \int_{\partial B(0,r)}\langle  D^2 B(y) v,v \rangle  \eta\left(\frac{r}{\veps} \right) dv   dr 
% \\&= u(y) \rho(y) \Delta B(y) \frac{1}{2 \veps^{m+2} \sigma_\eta'} \int_{0}^\veps  r^2  d \alpha_d r^{d-1}   dr 
% \\&= u(y) \rho(y) \Delta B(y) \frac{1}{2 d \veps^{m+2} \sigma_\eta'} \int_{0}^\veps  r^2  d \alpha_d r^{d-1}   dr
\\&=  u(y) \rho(y) \Delta B(y).
\end{align*}

Combining the above estimate with \eqref{auxeqn:A1Final}, \eqref{auxeqn:A2Final}, and \eqref{auxeqn:A3Final} we see that
\begin{align*}
 \sum_{x \in \mathcal{X}} u(x) Q^{ms}(x,y) & \approx -\left(  \rho(y) \left \langle  \nabla u (y), \nabla B(y) \right \rangle+ 2 u(y) \langle \nabla B(y) , \nabla \rho(y) \rangle +  u(y) \rho(y) \Delta B(y) \right)
\\&= - \frac{1}{\rho} \divv\left( u \rho^2 \nabla B \right). 
\end{align*}
%While it is not entirely transparent in our argument
Note that the graph dynamics takes place  on the provided data points, that is on $\P(\X) \subset \P(\M)$. As $n \to \infty$, $\P(\X)$ approximates $\P(\M)$. This partly explains that had we carried out the argument above in the full manifold setting the resulting dynamics would be restricted to the manifold and in particular both the divergence and the gradient above would take place on $\M$. That is for data on a manifold
\[  \sum_{x \in \mathcal{X}} u(x) Q^{ms}(x,y)  \approx - \frac{1}{\rho} \divv_{\M} \left( u \rho^2  \nabla_\M B \right). \]

\begin{remark}
\label{rem:FokkerPlanckLimit}
Notice that with the choice $B(x):= \log(\rho(x))$ the above becomes:
\[- \frac{1}{\rho} \divv_\M \left( u \rho  \nabla_\M  \rho \right) , \]
whereas with the choice $B(x)=-\frac{1}{\rho(x)}$ we get:
\[ -\frac{1}{\rho} \divv_\M \left( u \rho \nabla_\M \log(\rho) \right) .\]
The above analysis suggests that the formal continuum limit of the evolution \eqref{Bgradascent} when $B= -\frac{1}{\rho}$ is the PDE:
\[ \partial_t u_t = - \frac{1}{\rho}\divv_\M( u_t \rho \nabla_\M \log(\rho) ). \]
Notice, however, that the solution $u_t$ of the above equation must be interpreted as a ``density" with respect to the measure $\rho(x) d Vol_\M$ ($\rho(x) dx$ in the flat case).
Thus, in terms of ``densities" with respect to $d Vol_\M$ we obtain 
\[ \partial _t f_t = - \divv_\M( f_t \nabla_\M \log(\rho) ) \]
where $f_t:= u_t \rho$. We recognize this latter equation as the PDE describing the mean shift dynamics \eqref{eqn:flowMSM}.
%  \eqref{eqn:FokkerPlanckRd} - No, this one is not on manifold
\end{remark}

% \begin{remark}
% Notice that when $\eta$ is given by $\eta(a)=\1_{|a|<1}$,
% the factor $\sigma'_{\eta}$ can be written as $\frac{\alpha_m}{2(m+2)}$, where $\alpha_m$ is the unit ball in $\R^m$. This follows from:
% \[ \frac{1}{2m} \int_{B_m(0,1)}|z|^2 dz = \frac{1}{2m} \int_0^1  r^2 m\cdot \alpha_m r^{m-1} d r= \frac{\alpha_m}{2(m+2)}. \] 
% \end{remark}

\subsection{Continuum limits of Fokker-Planck equations on graphs}
\label{sec:ContLimitsFP}

In this section we formally derive the large sample limit of the two types of Fokker-Planck equations on $\G$ that we consider in this paper, i.e. equation \eqref{Markovchain} when $Q=Q_{\alpha}^{rw}$ (for $\alpha \in (-\infty, 1] )$ and when $Q=Q_\beta$ (for $\beta \in [0,1]$).

We start our computations by pointing out that after appropriate scaling and under some regularity conditions on the density $\rho$, the diffusion operator $L^{rw}_\alpha$ converges towards the differential operator:
\begin{equation}
\L_\alpha v := -\frac{1}{\rho^{2(1-\alpha)}}\divv_\M(\rho^{2(1-\alpha)} \nabla_\M v).
\end{equation}
To be precise, if we set
\[ C_\alpha= \frac{1}{\sigma_{\eta} \veps^2 }, \quad \sigma_\eta:=   \int_{\R^m} |z|^2 \eta(|z|) dz \Big{/} \int_{\R^m} \eta(|z|) dz,\]
then, for all smooth $v: \M \rightarrow \R$ we have
\[  -\sum_{y \in \X}  Q_\alpha^{rw}(x,y) v(y) =   C_\alpha \sum_{y}  L_\alpha^{rw}(x,y) v(y) \rightarrow  \L_\alpha v (x),  \]
as $n \rightarrow \infty$ and $\veps\rightarrow 0$ at a slow enough rate. This type of pointwise consistency result can be found in \cite{singer2006graph} and \cite{Coifman1}.
Furthermore eigenvalues and eigenvectors of the graph Laplacians converge as  $n \to \infty$ and $\eps \to 0$ (with $\eps \gg \left( \frac{\ln n}{n} \right)^{1/d}\:$ ) to eigenvalues and eigenfunctions of the corresponding Laplacian on $\M$, see \cite{trillos2018variational}. If $\M$ is a manifold with boundary then the continuum Laplacian is considered with no-flux boundary conditions.
%For spectral consistency results see \cite{trillos2018variational}; 
We note that
the results from \cite{trillos2018variational} are only stated for the case $\alpha=0$, i.e. for the standard \textit{random walk Laplacian}, but the proof in \cite{trillos2018variational} adapts to all $\alpha \in (-\infty,1]$ assuming that the density $\rho$ is smooth enough and is bounded away from zero and infinity.

Now, to understand the large sample limit of the dynamics \eqref{Markovchain} when $Q=Q_\alpha^{rw}$, we actually need to study the expression:
\begin{equation}
\sum_{x\in \X} u(x) Q_\alpha^{rw}(x,y), \quad  y \in \X, 
\label{eqn:DiffMaps2}
\end{equation}
which in matrix form can be written as $Q^{rw^T}_\alpha  u$ provided we view $u$ as a column vector. For that purpose we consider two smooth test functions $g,h$ on $\M$.
By definition of transpose
\begin{equation}
   \frac{1}{n}\sum_{i=1}^n h(x_i) (Q_{\alpha}^{rw^T} g)(x_i)    = \frac{1}{n}\sum_{i=1}^n g(x_i) (Q_{\alpha}^{rw} h)(x_i).
   \label{eqn:AuxCL0}
\end{equation}
At the continuum level, the definition of $\L_\alpha$ and integration by parts provide that
\[  \int_{\M} h(x) \frac{1}{\rho(x)}\divv_\M\left( \rho^{2(1-\alpha)} \nabla_\M \left( \frac{g}{\rho^{1-2\alpha}} \right)  \right)  \rho(x) dVol_\M(x)  = - \int_{\M} g(x) \L_\alpha h(x) \rho(x) dVol_\M(x)      \]
By the convergence of $Q_{\alpha}^{rw}$ towards $-\L_\alpha$ as $n \to \infty$ we can conclude that right hand sides converge, and thus the left hand sides do too; 
notice that the $\rho(x) dx$ on both sides appear because in both sums in \eqref{eqn:AuxCL0} the points $x_i$ are distributed according to $\rho$. From this computation we can identify the limit of \eqref{eqn:DiffMaps2} as:
\[ \frac{1}{\rho(y)}\divv_\M \left( \rho^{2(1-\alpha)} \nabla_\M \left( \frac{u}{\rho^{1-2\alpha}} \right)  \right).  \]
In turn, we obtain the formal continuum limit of the dynamics \eqref{Markovchain} when $Q=Q_{\alpha}^{rw}$: 
\[ \partial_t u_t =  \frac{1}{\rho}\divv_\M \left( \rho^{2(1-\alpha)} \nabla_\M \left( \frac{u_t}{\rho^{1-2\alpha}} \right)  \right),  \]
where  $u_t$ represents the density with respect to $\rho(x) dx$. If we consider 
\begin{equation}
 \label{eq:furho}
f_t(x) := u_t(x) \rho(x)
\end{equation}
that is  $f_t$ is a probability density w.r.t. $dx$, we see it satisfies
\begin{align} \label{CL_varyingalpha}
 \partial_t f_t =  \divv_\M \left(\rho^{2(1-\alpha)}  \nabla_\M \left(  \frac{\ f_t}{\rho^{2(1-\alpha)}}\right)\right)=   \Delta_\M f_t - 2(1-\alpha) \divv_\M ( f_t \nabla_\M \log(\rho))  ,
 \end{align}
where the last equality follows from an application of the product rule to the term $\nabla_\M \left(\frac{f}{\rho^{2(1-\alpha)}} \right)$. Notice that after considering a time change $ t \leftarrow \frac{t}{ 3-2\alpha} $, we can rewrite equation \eqref{CL_varyingalpha} as:
\begin{equation}
   \partial_t f_t = (1-\beta_\alpha) \Delta_\M f_t - \beta_\alpha \divv_\M( f_t \nabla_\M \log(\rho)),
    \label{eqn:FokkerPlanckContinuum}
\end{equation}
where $ \beta_\alpha = (2-2\alpha)/(3-2 \alpha) \in [0,1]$.

Using the above analysis and Remark \ref{rem:FokkerPlanckLimit}, we can also conclude that the (formal) large sample limit of equation \eqref{Markovchain} with $Q=Q_\beta$ and potential $B=-\frac{1}{\rho}$ is given by 
\begin{equation}
   \partial_t f_t = (1-\beta) \Delta_\M f_t - \beta \divv_\M( f_t \nabla_\M \log(\rho)),
    \label{eqn:FokkerPlanckContinuum}
\end{equation}
that is, the same continuum limit as for the Fokker-Planck equations constructed using the rate matrix $Q_\alpha$ for $\alpha$ such that $\beta=\beta_\alpha$.

\begin{remark}
Notice that when $\beta=1$ equation \eqref{eqn:FokkerPlanckContinuum} reduces to the heat equation on $\M$ where no role is played by $\rho$. In this case clustering is determined completely by the geometric structure of $\M$. On the other hand, when $\beta=0$, equation \eqref{eqn:FokkerPlanckContinuum} reduces to mean shift dynamics on $\M$ as discussed in section \ref{sec:MeanShift}. 
\end{remark}

\begin{remark}
Several works in the literature have established precise connections between operators such as graph Laplacians built from random data and analogous differential operators defined at the continuum level on smooth compact manifolds without boundary. For pointwise consistency results we refer the reader to \cite{singer2006graph,hein2005graphs,hein2007graph,belkin2005towards,ting2010analysis,GK}). 
For \textit{spectral convergence} results we refer the reader to \cite{vLBeBo08} where the regime $n \rightarrow \infty$ and $\veps$ constant has been studied. Works that have studied regimes where $\veps$ is allowed to decay to zero (where one recovers differential operators and not integral operators) include \cite{Shi2015,BIK,trillos2019error,Lu2019GraphAT,calder2019improved,DunsonWuWu,WormellReich}.
Recent work \cite{cheng2020tuned} considers the spectral convergence of $\L_\alpha$ with self-tuned bandwidths and includes the $\alpha<0$ range. 
The work \cite{CalderGTLewicka} provides regularity estimates of graph Laplacian eigenvectors. 

The case of manifolds \textit{with} boundary has been studied in papers like \cite{BerryVaughn,BoundaryWuWu,Lu2019GraphAT,trillos2018variational}. It is important to highlight that the specific computations presented in our section \ref{sec:ContLimitMS} would have to be modified 
to take into account the effect of the boundary, in particular on  the kernel density estimate. 
However, we remark that the tools and analysis from the papers mentioned above can be used to generalize these computations.
\end{remark}

\begin{remark}
A connection between Fokker-Planck equations at the continuum level and the graph dynamics induced by $Q_{\alpha}^{rw}$ when $\alpha=1/2$ was explicitly mentioned in \cite{NADLER2006113}. To establish an explicit link between mean shift and spectral clustering, however, we need to consider the range $(-\infty, 1]$ for $\alpha$. In the diffusion maps literature the interval $[0,1]$ is considered as natural range for $\alpha$, but the analysis presented in this section explains why $(-\infty,1]$ is in fact a more natural choice. 
\end{remark}

\begin{remark}
Besides the Fokker-Planck interpolations considered in section \ref{sec:FokkerPlanckGraphsInt}, another family of data embeddings that are used to interpolate geometry-based and density-driven clustering algorithms is based on the path-based metrics studied in \cite{LittleJMLR20,Little20}.  

\end{remark}

% In conclusion, we expect $\Delta_{\veps, \beta}$ to behave like:
% \begin{align}
% \begin{split}
%  \Delta_\beta u(x) &=  -\beta \sigma_\eta' \left \langle  \nabla u (x), \nabla \rho(x) \right \rangle - \frac{(1-\beta)\sigma_\eta}{\rho} \divergence(\rho^2 \nabla u).
%  \end{split}
%  \end{align}

\subsection{The Witten Laplacian and some implications for data clustering}

In the previous section we presented a (formal) connection between Fokker-Planck equations on proximity graphs and Fokker-Planck equations on manifolds. In this section we use this connection to illustrate why the Fokker-Planck interpolation is expected to produce better clusters in settings like the blue sky problem discussed in our numerical experiments in section \ref{sec:BlueSky} where both pure mean shift and pure spectral clustering perform poorly. For simplicity we only consider the Euclidean setting.

We start by noticing that  equation \eqref{eqn:FokkerPlanckContinuum} can be rewritten as
\begin{equation}
  \partial_t {\tilde f}_t = -\Delta_{\varrho} \tilde f_t,
\label{eqn:WittenDynamics}
\end{equation}
after considering the transformation:
\[f= \exp \left(-\frac{1-\beta}{2\beta} \varrho \right) \tilde f, \quad \varrho:= - \log(\rho). \]
In the above, the operator $\Delta_\varrho$ is the \textit{Witten Laplacian} (see \cite{Witten83} and \cite{MichelSmallEigenWittenLap19}) associated to the potential $\frac{1-\beta}{2} \varrho$ which is defined as:
\begin{equation}
\label{eqn:WittenDef}
  \Delta_\varrho v  := - \beta^2 \Delta v  + \frac{(1-\beta)^2}{4}|\nabla \varrho|^2 v - \frac{\beta(1-\beta)}{2}(\Delta \varrho) v.  
\end{equation}
From the above we conclude that the Fokker-Planck dynamics \eqref{eqn:FokkerPlanckContinuum} can be analyzed by studying the dynamics \eqref{eqn:WittenDynamics}. In turn, some special properties of the Witten Laplacian $\Delta_\varrho$ that we review next allow us to use tools from spectral theory to study equation \eqref{eqn:WittenDynamics} and in turn also the type of data embedding induced by our Fokker-Planck equations on graphs. 

To begin, notice that
\begin{equation}
 \Delta_\varrho =  \left(-\beta \divv + \frac{(1-\beta)}{2}\nabla \varrho \right) \left( \beta \nabla + \frac{(1-\beta)}{2}\nabla \varrho  \right).
\end{equation}
From the above we see that $\langle \Delta_\varrho f , g \rangle_{L^2(\M)}$ can be written as:
\[  \langle \Delta_\varrho g , h \rangle_{L^2(\M)} = \int_{\M}  \left \langle  \beta \nabla g + g\frac{(1-\beta)}{2} \nabla \varrho  ,  \beta \nabla h+ h\frac{(1-\beta)}{2} \nabla \varrho \right \rangle dx,     \]
from where we conclude that $\langle \Delta_\varrho f , g \rangle_{L^2(\M)}$ is a quadratic form with associated Dirichlet energy:
\begin{equation}
\label{eqn:Dirich}    
D(f) := \int_{\M} \left \lvert \nabla f + f\frac{(1-\beta)}{2} \nabla \varrho \right \rvert^2 dx.
\end{equation}
When $\M$ is compact, it is straightforward to show that there exists an orthonormal basis $\{ \varrho_k \}_{k \in \N}$ for $L^2(\M)$ consisting of eigenfunctions of $\Delta_\varrho$ with corresponding eigenvalues $0 = \lambda_1 < \lambda_2 \leq \lambda_2 \leq \dots$ that can be characterized using the Courant-Fisher minmax principle. Using the Spectral Theorem we can then represent a solution to \eqref{eqn:WittenDynamics} as
\[ \tilde{f}_t= \sum_{k=1}^\infty e^{-t\lambda_k} \langle\tilde{f}_0, \varrho_k\rangle_{L^2(\M)} \varrho_k, \]
and conclude that the dynamics \eqref{eqn:WittenDynamics} are strongly influenced by the eigenfunctions with smallest eigenvalues.

We now explain the implication of the above discussion on data clustering. Suppose that we consider a data distribution in $\R^2$ as the one considered in section \ref{sec:BlueSky} modelling the blue sky problem, so that in particular it has product structure, i.e. $\rho(x,y)= \rho_1(x) \rho_2(y)$. In this case we can use the additive structure of the potential $\varrho= - \log(\rho(x,y)) = - \log(\rho_1(x)) - \log(\rho_2(y))=: \varrho_1(x)+ \varrho_2(y) $ to conclude that  the set of eigenvalues of $\Delta_\varrho$ and a corresponding orthonormal basis of eigenfunctions can be obtained from:
\[ \lambda_{1,i} + \lambda_{2,j} , \quad  \varrho_{1,i}(x) \varrho_{2,j}(y)  \]
where $(\lambda_{1,i}, \varrho_{1,i})$ are the eigenpairs for the 1d Witten Laplacian $\Delta_{\varrho_1}$ and $(\lambda_{2,j}, \varrho_{2,j})$ are the eigenpairs for $\Delta_{\varrho_2}$. In particular, the first non-trivial eigenvalue of $\Delta_{\varrho}$ and its corresponding eigenfunction (which will be the effective discriminators of the two desired clusters if $\lambda_3$ is considerably larger than $\lambda_2$) are either $\lambda_{1,2}$ and $\varrho_{1,2}(x) \varrho_{2,1}(y)$ , or $\lambda_{2,2}$ and $\varrho_{1,1}(x) \varrho_{2,2}(y)$. This discussion captures the competition between a horizontal and a vertical partitioning of the data in the context of the blue sky problem from section \ref{sec:BlueSky}. While we aren't able to retrieve the desired horizontal partitioning by setting $\beta=0$ or $\beta=1$, we can identify the correct clusters by setting $\beta$ strictly between zero and one (closer to one than to zero). We notice that the results from \cite{MichelSmallEigenWittenLap19} can be used to obtain precise quantitative information on the small eigenvalues of the $1d$  Witten Laplacians $\Delta_{\varrho_1}$ and $\Delta_{\varrho_2}$ when $\beta$ is close to one (i.e. the diffusion term is small) which we can use to determine whether $\lambda_{2,2}< \lambda_{1,2}$ or vice versa.

% Finally, a time change $\tilde t = \frac{1}{1+2(1-\alpha)}$ now allows us to derive the equation

% From this and the convergence of $Q^{rw}$ towards the operator $\Delta_\alpha$, we can see that the continuum analogue of the (discrete) diffusion maps from section \ref{} is induced by the PDE:
% \[ \partial_t f=  \Delta_\alpha^* f, \]
% where the operator $\Delta_\alpha^*$  takes the form
% \[ \Delta_\alpha^* f =  \divv( \rho^{2(1-\alpha)} \nabla \left(\frac{\nabla f}{\rho^{2(1-\alpha)}} \right)) = - \Delta f + 2(1-\alpha) \divv( f \nabla \log(\rho) )    \]

% In the above $u$ is interpreted as a density with respect to the volume form of $\M$. 

% When $\alpha \leq 1$, and after applying a time rescale $t= \frac{1}{1+2(1-\alpha)} $, the dynamics
% \[ \dot{u}= \Delta_\alpha^* u \]
% can then be written as:
% \[   \]
% in other words, as the dynamics induced by the continuous version of our Fokker-Planck interpolation. If we consider the correspondence 
% \[\beta= \frac{1}{1+2(1-\alpha)} \]
% The above means that at least at the continuum level (and after appropriate time rescaling) the Fokker-Planck interpolation maps and the diffusion maps (with extended range for the parameter $\alpha$) 

%%%%%%%%%%%%%%%%%%%%%%%%%%%%%%%%%%%%%%%%%%%%%%%%%%%%%%

\section{Numerical Examples}
\label{sec:Numerics}

We now turn to the details of our numerical method and examples illustrating its properties.  We begin, in section \ref{nummethoddetails}, by describing the details of our numerical approach. We provide Algorithm \ref{nummethoddetails} for its  practical implementation.

In section \ref{simulationssection}, we consider several numerical examples, beginning with examples in one spatial dimension. In Figure \ref{fig:KDEofgraphdynamics}, we illustrate how the graph dynamics for the transition rate matrices $Q_\beta$ and $Q^{rw}_\alpha$ can be visualized as the evolution of a continuum density, and in Figure \ref{comparegdfd1d}, we illustrate the good agreement between the graph dynamics  and the dynamics of the corresponding continuum Fokker-Planck equation.  In Figure \ref{dynamicclustering}, we show how the clustering performance of our method depends on the balance between drift and diffusion ($\beta$), the time of clustering ($t$), and the number of clusters ($k$); we also illustrate the benefits and limitations of using the energy of the $k$-means clustering to identify the number of clusters. In Figure \ref{mesaKDEweakness}, we consider the role of   the kernel density estimate in clustering dynamics, showing how adding diffusion to mean shift dynamics can help the dynamics overcome spurious local minimizers in the kernel density estimate, leading to better clustering performance. In Figure \ref{differentdataup}, we illustrate the interplay between the underlying data distribution and the balance between drift and diffusion ($\beta$) .

Next, we consider several examples in two dimensions. In Figures \ref{blueskyclustering}-\ref{blueskyepsilon}, we consider a model of the \emph{blue sky problem}, in which data points are distributed over two elongated clusters that are separated by a narrow low density region. We illustrate how diffusion dominant dynamics prefer to cluster based on the geometry of the data, leading to poor performance. Similarly, pure mean shift dynamics can exhibit poor clustering due to local {maxima} in the kernel density estimate. By interpolating between the two extremes, we observe robust clustering performance, for a wide range of graph connectivity ($\varepsilon$). Finally, in Figures \ref{threeblobsdynamics}-\ref{threeblobsclustering}, we consider an example in which three blobs are connected by two bridges: one wide, low density bridge and another narrow, high density bridge. This example is constructed so that there is no correct clustering into two clusters. Instead, a geometry-based clustering method would prefer to cut the thin bridge, and a density-based clustering method would prefer to cut the wide bridge. We show how varying the balance between drift and diffusion in our method ($\beta$) allows our method to cut either bridge.
 
\subsection{Numerical Method} \label{nummethoddetails}
% We begin by describing the  details of our numerical approach, including the graph dynamics we consider, their corresponding continuum PDEs, and how we identify clusters in the graph dynamics. 

For our numerical experiments we consider a domain $\Omega \subseteq \Rd$ and a density $\rho: \Omega \to [0,+\infty)$, normalized so that $\int_\Omega \rho = 1$. All PDEs on $\Omega$ will be considered with no-flux boundary conditions, as the solutions of the graph-based equations converge to the solutions of PDE with no-flux boundary conditions
(observed in \cite{Coifman1} and rigorously proved in \cite{trillos2018variational} for Laplacians).

We draw $n$ samples $\{x_i\}_{i=1}^n$ from $\rho$ on $\Omega$. These samples are the nodes of our weighted graph, and for all simulations, the weights on the graph are given by a Gaussian weight function
\begin{align}\label{eq:gaussian}
w(x_i,x_j) = \varphi_\veps(|x_i - y_j|),  \quad     \varphi_\veps(a) = \frac{e^{-a^2/2\veps^2}}{(2\pi \veps^2)^{d/2} } ,  \quad a \in \mathbb{R}.
\end{align}
In our one dimensional simulations, we take the graph bandwidth parameter $\varepsilon$ to be
\begin{align} \label{epsdef}
\veps =  \sqrt{2} \max_i \min_{j : j\neq i} |x_i - x_j| ;
\end{align}
that is, $\veps$ equals the maximum distance to the closest node.  
We note that even in higher dimensions the $\veps$ above  
 scales as $\left( {\ln n}/{n} \right)^{1/d}$ with the number of nodes $n$. This has been identified as the threshold,  in terms of   $n$,  at which the graph Laplacian is spectrally consistent with the manifold Laplacian \cite{trillos2018variational}. In Figure \ref{blueskyepsilon}, we illustrate how the choice of $\eps$ impacts dynamics and, ultimately, clustering performance.

With this graphical structure, we now recall the weighted diffusion transition rate matrix $Q^{rw}_\alpha$,  for $\alpha \in (-\infty , 1]$, as in equation (\ref{Qrw}),  with the constant $C_\alpha = ((3-2 \alpha) \veps^2)^{-1}$,
\begin{align} \label{Qalphadeffinal}
Q^{rw}_\alpha(x,y) &:= \frac{1}{(3-2 \alpha) \veps^2} \; \begin{cases} \frac{w_\alpha(x,y)}{\sum_{z \not =x } w_\alpha(x,z)} &\text{ if } x \neq y , \\
 - 1 &\text{ if } x =y, \end{cases}  \\
 w_\alpha(x,y) &:= \frac{w(x,y)}{d(x)^\alpha  d(y)^\alpha }, \quad  \quad  d(x_i) = \sum_{x_j \not = x_i} w(x_i,x_j) .
\end{align}
Similarly, we recall the transition rate matrix $Q_\beta$, for $\beta \in [0,1]$, as in equation \eqref{eq:upwindFPgraph}, with the constant $C_{ms} = (\veps^2 n)^{-1}$,
\begin{align} \label{Qbetadeffinal}
Q_{\beta} &:= \beta Q^{ms} + (1- \beta) Q^{rw}_1 ,\\
    Q^{ms}(x,y) &:=  \frac{1}{\varepsilon^2 n} \begin{cases} \left(-\frac{1}{\hat \rho_\delta(y)} + \frac{1}{\hat \rho_\delta(x)} \right)_+ w(x,y) , & \text{ for } x \neq y, \\
-\sum_{z \neq x} \left(-\frac{1}{\hat \rho_\delta(z)}+\frac{1}{\hat \rho_\delta(x)} \right)_+ w(x,z) , & \text{ for } x =y , \end{cases} \\
  \hat{\rho}_\delta(x) &:=  \frac{1}{n} \sum_{y \in \mathcal{X}} \varphi_\delta(x-y)  .
\end{align}
Unless otherwise specified, we take the bandwidth $\delta$ in our kernel density estimate for our one dimensional examples to be
\begin{align} \label{typicaldelta}
\delta = \sqrt{2} \left( \frac{|\Omega|}{n} \right)^{0.5} .
\end{align}

With these transition rate matrices in hand, we may now consider solutions $u_t$ of \eqref{Markovchain} when $Q=Q^{rw}_\alpha$ or when $Q=Q_{\beta}$.  We solve the ordinary differential equations describing the graph dynamics by directly computing the matrix exponential $e^{tQ}$ in each case; see Definition \ref{def:MarkovChain}. Following the discussion in section \ref{sec:ContLimitsFP}, we know that for each of these dynamics,  as $n \to +\infty$ and $\veps, \delta \to 0$ (at an $n$ dependent rate that is not too fast),  the measures $\sum_{j=1}^n   u_t(x_j) \delta_{x_j}$ are expected to converge to solutions $f_t$ of the following Fokker-Planck equation:
  \begin{align}
   \label{e:ctmFP}
 \partial_t f_t &= (1-\beta) \Delta f_t - \beta \divv( f_t \nabla \log(\rho)),
 \end{align}
 where for the $Q^{rw}_\alpha$ dynamics, we take
\begin{align} \label{betatoalpha}
     \beta = \beta_\alpha = (2-2\alpha)/(3-2 \alpha)
\end{align}  
% two types of dynamics on graphs which converge to the Fokker-Planck equation in the continuum limit:
% \begin{align}  
% \dot{u}_n &= Q_\beta u_n , &&  \beta \in [0,1] \label{upgraph} \\
%  \dot{u}_n &= Q^{rw}_\alpha u_n , &&  \alpha \in (-\infty, 1]. \label{wdgraph}
% \end{align}
%   \begin{align} \label{e:ctmFP}
%  \partial_t f &= (1-\beta) \Delta f + \beta \divv( f \nabla \log(\rho))   , \\
%   \label{CFctm}
% \partial_t f &= (1-\beta_\alpha) \Delta f + \beta_\alpha  \divv( f \nabla \log(\rho)) , \quad
%  \beta_\alpha = (2-2\alpha)/(3-2 \alpha) ,
%  \end{align}
The steady state of the equation is the corresponding Maxwellian distribution
 \begin{align} \label{exactMaxwellian}
 c_{\rho, \beta} \ \rho^{\beta/(1-\beta)}(x) ,
 \end{align}
 where $c_{\rho,\beta} >0$ is a normalizing constant chosen so that the distribution integrates to one over $\Omega$. Note that, if  $d(x_i)$ represents the   degrees of the graph vertices, as in equation (\ref{eqn:degree}), then the function $  u_t(x_i) d(x_i) $ likewise converges to $f_t(x)$ as the number of nodes in our sample $n \to +\infty$. Consequently, when comparing our graph dynamics to the PDE dynamics, we will often plot $u_t(x_i) d(x_i)$ and $f_t(x)$.
 
%  Finally, we apply the graph dynamics (\ref{upgraph}-\ref{wdgraph}) to cluster the nodes. We do this by considering the evolution of the above dynamics from the initial value 
%  \begin{align}
% u_n(0) = \delta_{x_i}, \quad \text{ so that } \quad (u_n(0))_j = \begin{cases} 1 &\text{ if } j = i , \\ 0 &\text{ if } j \neq i .\end{cases}
%  \end{align} We denote by $\Psi_T^\beta(x_i)$ and $\Phi_T^\alpha(x_i)$ the value of the corresponding dynamics (\ref{upgraph}-\ref{wdgraph}) at time $t= T$, as in equation (\ref{meanshiftsemigroup}). We then apply $k$-means to the vectors $\{ \Psi_T^\beta(x_i) \}_{i=1}^n$ and $\{ \Phi_T^\alpha(x_i) \}_{i=1}^n$, obtaining a map from nodes $\{x_i\}_{i=1}^n$ to cluster centers $\{l_j\}_{j=1}^k$. 

 Finally, we use the embedding maps $\hat{\Psi}_\alpha$ and $\hat{\Psi}^\beta$ from sections \ref{sec:EmbeddingRWLap} and \ref{sec:FokkerPlanckGraphsInt} to cluster the nodes. In particular, we apply $k$-means to the vectors $\{ \hat{\Psi}_\alpha(x_i) \}_{i=1}^n$ and $\{ \hat{\Psi}^\beta(x_i) \}_{i=1}^n$, obtaining in this way a series of maps from nodes $\{x_i\}_{i=1}^n$ to cluster centers $\{l_m\}_{m=1}^k$.  Nodes mapped to the same cluster center are identified as belonging to the same cluster. While we will not discuss at any depth the methods to select the best number of clusters, we note that a number of methods to do so (in particular the elbow method and the gap statistics \cite{tibshirani-gap-01}) relies on the 
 value of the $k$-means energy,
 \begin{equation} \label{Ekmeans}
 E_k = \frac{1}{n} \sum_{i=1}^n \min_{m= 1, \dots k} | \Psi(x_i) - l_j |^2,
 \end{equation}
for each relevant $\Psi$. Note that $E_k$ always decreases with $k$.  While a large decrease in the energy as $k$ increases is indicative of the improved approximation of data by cluster centers the size of the jumps is truly telling only if we compare it with the relevant model for the data considered, see \cite{tibshirani-gap-01} and discussion in Section \ref{sec:Clustering Dynamics}. For ease of visualization, in our numerical examples, we will plot the \emph{normalized} $k$-means energy, which is rescaled so that energy of a single cluster equals one,
\begin{align} \label{kmeansenergy} 
 E_k^{\rm norm} = E_k / E_1 .
 \end{align}
 
 All of our simulations are conducted in Python, using the Numpy, SciPy, Sci kit-learn, and MatPlotLib libraries \cite{numpy,scipy,matplotlib,scikit-learn}.  In particular, we use the Sci kit-learn implementation of $k$-means to cluster the embedding maps.

 \begin{algorithm}[h] \raggedright
   \caption{Dynamic Clustering Algorithm for $Q_\beta$ or $Q^{rw}_\alpha$} \label{alglabel}
\SetAlgoLined
%$n=0$, $h^{n} = h(0)$ \\
\KwIn{$\{x_i\}_{i=1}^n$, $\varepsilon$, $\delta$, $t$, $k$}
%\KwOut{${h}^{n+1}$}
$Q = Q_\beta$ or $Q = Q^{rw}_\alpha$ \\
$\hat{\Psi}_Q (x_i) = (e^{tQ})_{(i, j=1,\dots n)} $ for $i =1, \dots, n$ \\
$l_m = {\rm Kmeans.fit}(\hat{\Psi}_Q (x_1), \dots, \hat{\Psi}_Q (x_n))$ with ${\rm n_{clusters}} =k$
\end{algorithm}

\subsection{Simulations} \label{simulationssection}

\begin{figure}[h]
 \begin{centering}
  \includegraphics[height=3.1cm,trim={.2cm 1.9cm .7cm .55cm},clip,left]{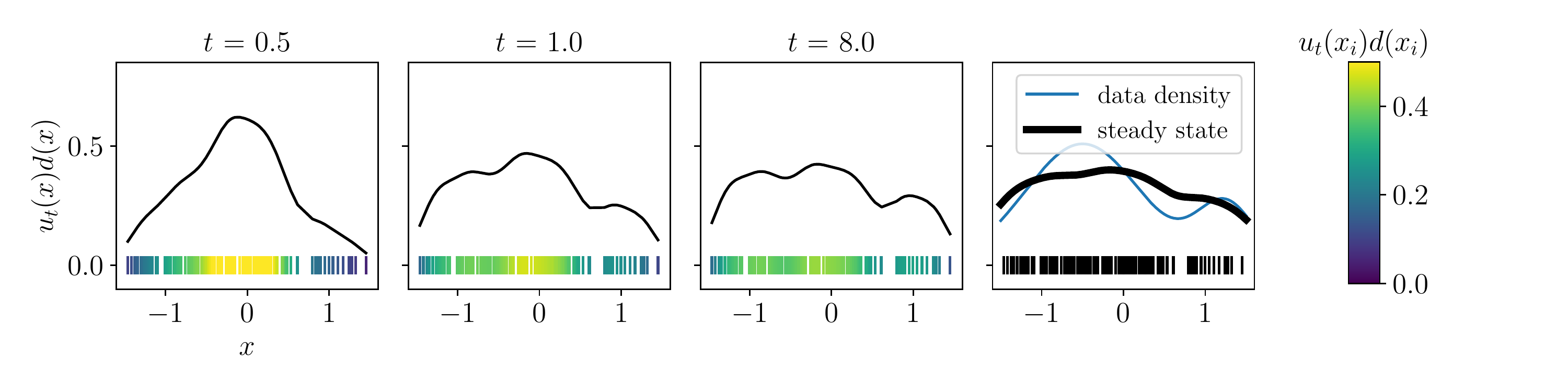}
  \includegraphics[height=3.46cm,trim={.2cm .7cm .7cm 1.15cm},clip,left]{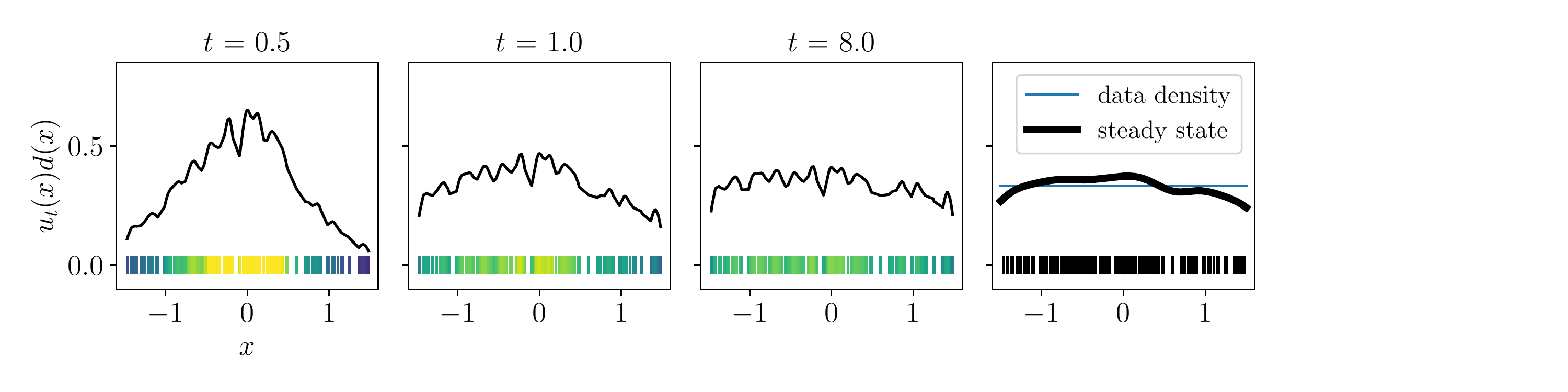}
  \end{centering}
\caption{Illustration of   the graph dynamics   $u_t$ for $Q_\beta$, $\beta = 0.25$, from   initial condition $\delta_{x_i}$, $x_i = -0.1$, for two choices of data density: $\rho_{\text{two bump}}$ (top) and $\rho_{\text{uniform}}$ (bottom).  The first three columns show the evolution of $u_t(x) d(x)$ at times $t=0.1, 0.5,$ and $8.0$, with the color of the markers representing the value of $u_t(x_i) d(x_i)$ at each node. The  last column depicts the data density (blue line) from which the nodes of the graph $\{x_i\}_{i=1}^n$ (black markers) are sampled, as well as the steady state of the dynamics (thick black line). }
 \label{fig:KDEofgraphdynamics}
\end{figure}

We now turn to simulations of the graph dynamics, PDE dynamics, and clustering. 

\subsubsection{Graph dynamics as density dynamics}
In Figure \ref{fig:KDEofgraphdynamics}, we illustrate how the  dynamics on a graph  can be visualized as the evolution of a density on the underlying domain $\Omega = [-1.5,1.5]$. The right column of Figure \ref{fig:KDEofgraphdynamics} illustrates two choices of data density (blue line),
\begin{align} \label{twobumpdef}
    \rho_{\text{two bump}}(x) = 4c \varphi_{0.5}(x+0.5) + c \varphi_{0.25}(x-1.25) \quad \text{ and } \quad \rho_{\text{uniform}}(x) = \frac13, \quad x \in \mathbb{R} .
\end{align}
The constant $c>0$ is chosen so that the integral of both densities over the domain equals one. We sample the nodes of the graph $\{ x_i\}_{i=1}^n$ (black markers)  from each density, with $n = 147$ nodes sampled for $\rho_{\text{two bump}}$ and $n = 140$ nodes sampled for $ \rho_{\text{uniform}}$.  The first three columns show the evolution of the graph dynamics $u_t(x) d(x)$ from equation \eqref{Markovchain} for $Q=Q_\beta$ with $\beta = 0.25$ and initial condition $\delta_{x_i}$, $x_i = -0.1$, where the top row corresponds to the graph arising from $\rho_{\text{two bump}}$ and the bottom row corresponds to the graph arising from $\rho_{\text{uniform}}$. The color of the markers represents the value of $ u_t(x_i) d(x_i)$ at each node. We observe in both rows that $u_t(x) d(x)$ approaches the steady state of the corresponding continuum PDE (\ref{exactMaxwellian}), depicted in a thick black line in the fourth column.

The fact that $d(x)u_t(x)$ appears more jagged in the bottom row compared to the top row is due to the smaller value of $\varepsilon$ in the graph weight matrix: see equations (\ref{eq:gaussian}-\ref{epsdef}). Since our sample of the data density in the top row has an isolated node  at $x_i = 1.44$, this leads to a significantly larger value of $\varepsilon$ in the simulations on the top row ($\varepsilon = 0.13$), compared to   the bottom row ($\varepsilon = 0.03$).   

\begin{figure}[h]
 $\beta = 0.0 \hspace{1.9cm} \beta = 0.25 \hspace{2cm} \beta = 0.5 \hspace{2cm} \beta  = 0.75$ \\
\hspace{-1.2cm} \includegraphics[height=3.52cm,trim={.6cm .5cm .6cm .2cm},clip,left]{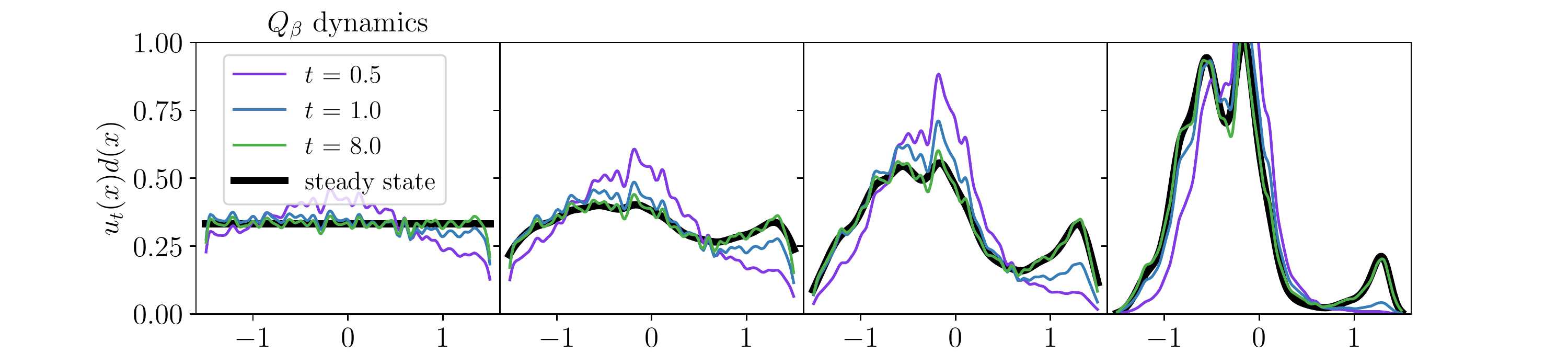} \\
\hspace{-1.2cm}  \includegraphics[height=3.52cm,trim={.6cm .5cm .6cm .2cm},clip,left]{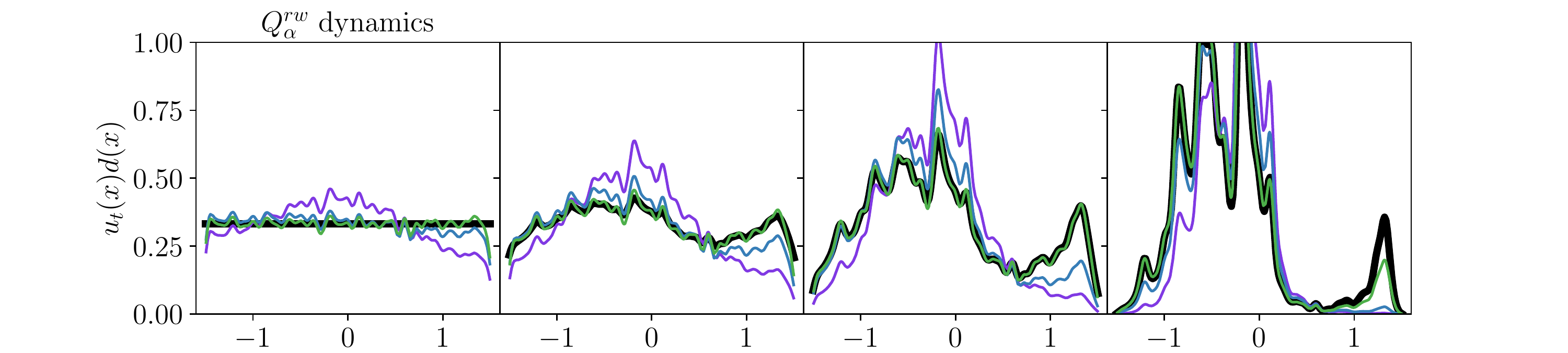} \\
\hspace{-1.2cm} \includegraphics[height=3.81cm,trim={.6cm 0cm .6cm .2cm},clip,left]{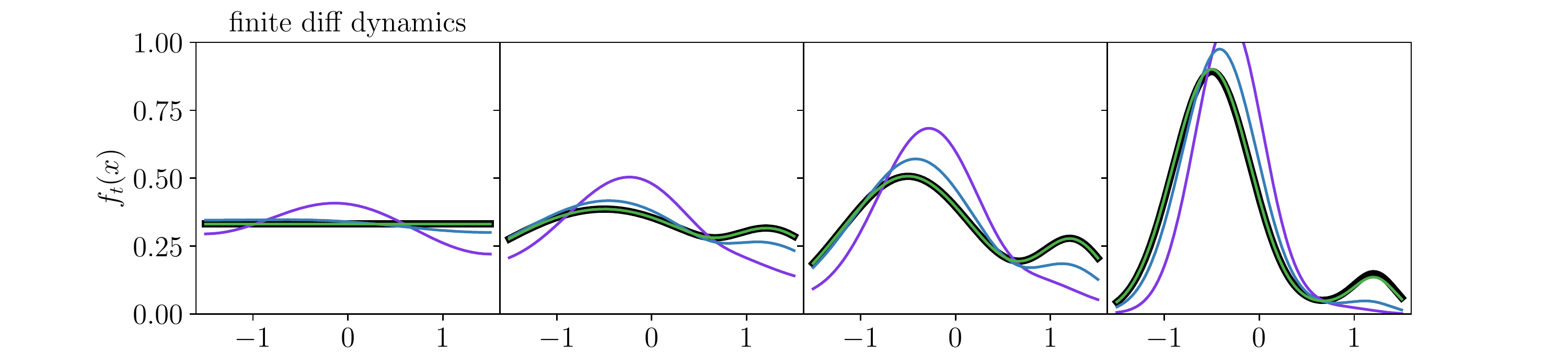} 
\caption{Comparison of  the graph dynamics for $Q_\beta$ (top) and $Q_\alpha$ (middle) with the PDE dynamics (bottom). The  data density is $\rho_{\text{two bump}}$, and the initial data is  $\delta_{x_i}$ for $x_i = -0.1$. The graphs are built from $n=625$ samples of the data density. The steady states  are obtained from equation (\ref{KDEMaxwellian}) for the graph dynamics and equation (\ref{exactMaxwellian}) for the finite difference dynamics.}
\label{comparegdfd1d}
\vspace{-.5cm}
\end{figure}

 \subsubsection{Comparison of Graph Dynamics and PDE Dynamics}
In Figure \ref{comparegdfd1d}, we compare the graph dynamics to the corresponding   Fokker-Planck equation (\ref{e:ctmFP}). We consider the data density given by $\rho_{\text{two bump}}$ and initial condition $\delta_{x_i}$, for $x_i = -0.1$. The graphs are built from $n=625$ samples of the data density, and solutions are plotted at times   $t = 0.5, 1.0, 8.0$. 

The top row illustrates the graph dynamics $u_t(x)d(x)$ arising from the transition rate matrix $Q_\beta$, for $\beta = 0, 0.25, 0.5, 0.75$. The middle row illustrates $u_t(x) d(x)$ arising from $Q^{rw}_\alpha$ for  $\alpha = 1.0, 0.83, 0.5, -0.5$. (The values of $\alpha$ are chosen to give the same balance between drift and diffusion as in the top row; see equation (\ref{betatoalpha}).) The last row shows a finite difference approximation of the   Fokker-Planck equation (\ref{e:ctmFP}).  We compute solutions of the   PDEs using a semidiscrete, upwinding finite difference scheme  on a one dimensional grid, with 200 spatial gridpoints and continuous time. This reduces the PDEs to a system of ODEs, which we then solve using the SciPy odeint method. 

The steady states we plot for the graph dynamics   are given by the following equation
\begin{align} \label{KDEMaxwellian}
 c_{n, \delta, \beta} (\hat{\rho}_\gamma(x)  )^{\beta/(1-\beta)}, \quad \hat{\rho}_\gamma(x)  =  \frac{1}{n} \sum_{y \in \mathcal{X}} \psi_\gamma(x-y)  , \quad \psi_\gamma(x) = \frac{1}{(2 \pi)^{1/2} \gamma} e^{-|x|^2/(2 \gamma^2)} ,
 \end{align}
 where $ c_{n, \delta, \beta}$ is a normalizing constant chosen so the steady state integrates to one over $\Omega$. For the $Q_\beta$ dynamics, we choose the standard deviation $\gamma = \delta$, and for the $Q^{rw}_\alpha$ dynamics, we choose $\gamma = \varepsilon$. Recall that $\hat{\rho}_\delta$ is the kernel density estimator used in the construction of the transition matrix $Q^\beta$; see equation (\ref{KDEkernel}).
  The steady states for the PDE dynamics are given by equation (\ref{exactMaxwellian}).   
 
 Interestingly, even though there is no explicit kernel density estimate of the data in the construction of the transition rate matrix  $Q^{rw}_\alpha$, the above simulations demonstrate better agreement of these dynamics as $t \to +\infty$ with the steady state arising from a kernel density estimate (\ref{KDEMaxwellian}) than with the steady state arising directly from the data density (\ref{exactMaxwellian}). This can be seen by observing the good agreement  at time $t=8.0$ with the solid black line shown in the middle row, rather than the solid black line shown in the bottom row. This suggests that the $Q^{rw}_\alpha$ operator effectively takes a KDE of the data density with bandwidth $\veps>0$, corresponding to the scaling of the weight matrix on the graph.

\begin{figure}[h]
\raggedright 

\hspace{2.1cm} $\beta = 0.25$ \hspace{2.2cm} $\beta = 0.9 $\hspace{2.4cm} $\beta = 1.0$ \hspace{4.2cm}

\includegraphics[height=5.5cm,trim={0.6cm .6cm .6cm .5cm},clip]{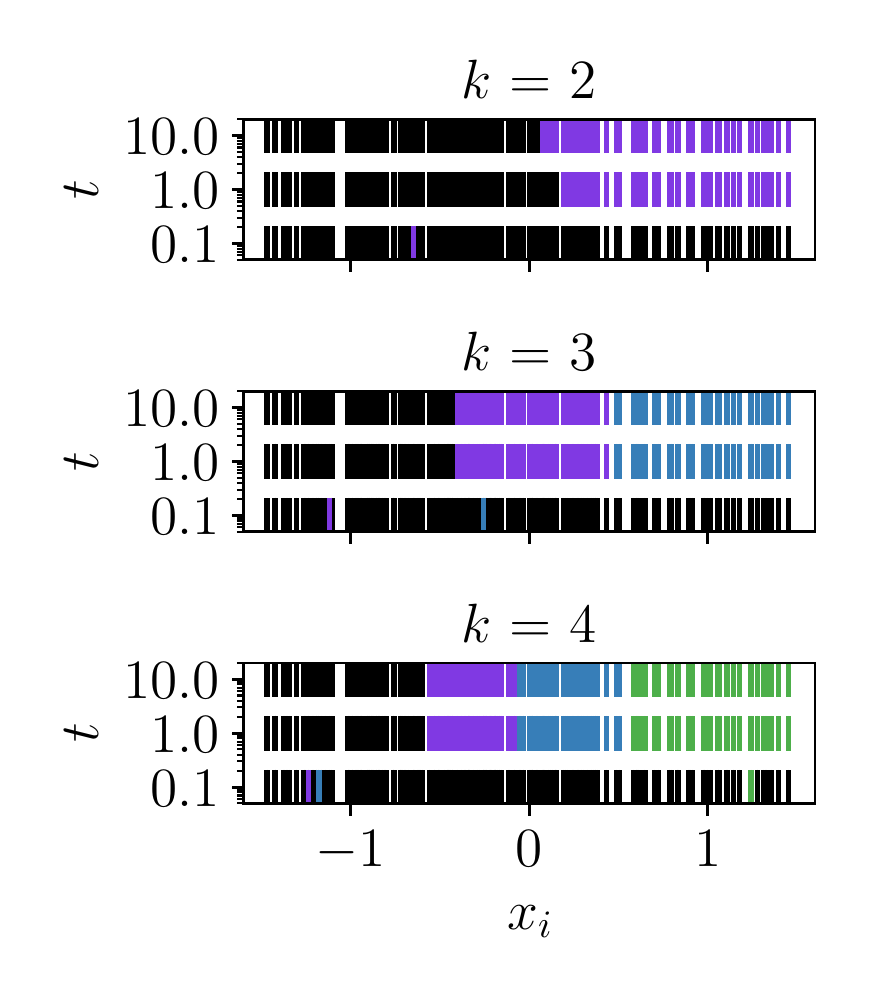} \includegraphics[height=5.5cm,trim={2.25cm .6cm .6cm .5cm},clip]{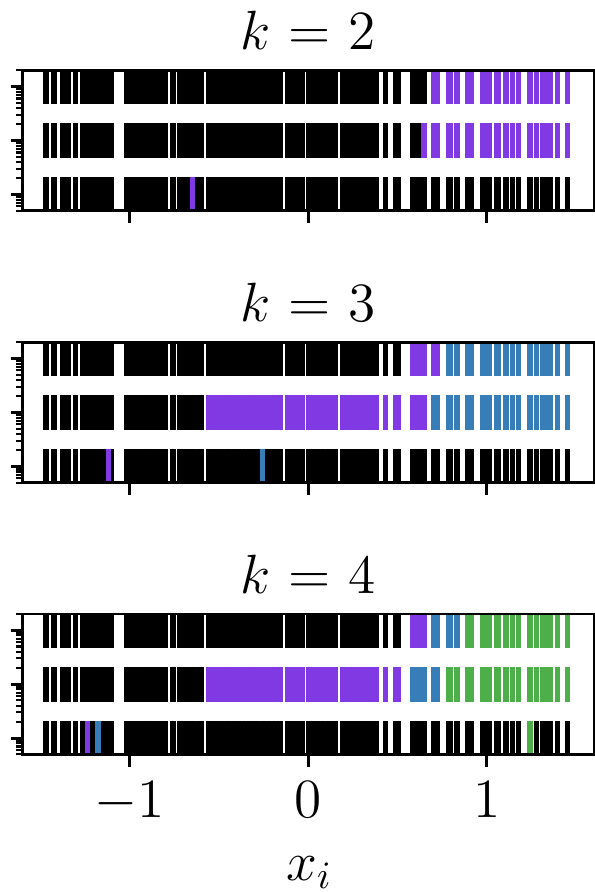} \includegraphics[height=5.5cm,trim={2.25cm .6cm .6cm .5cm},clip]{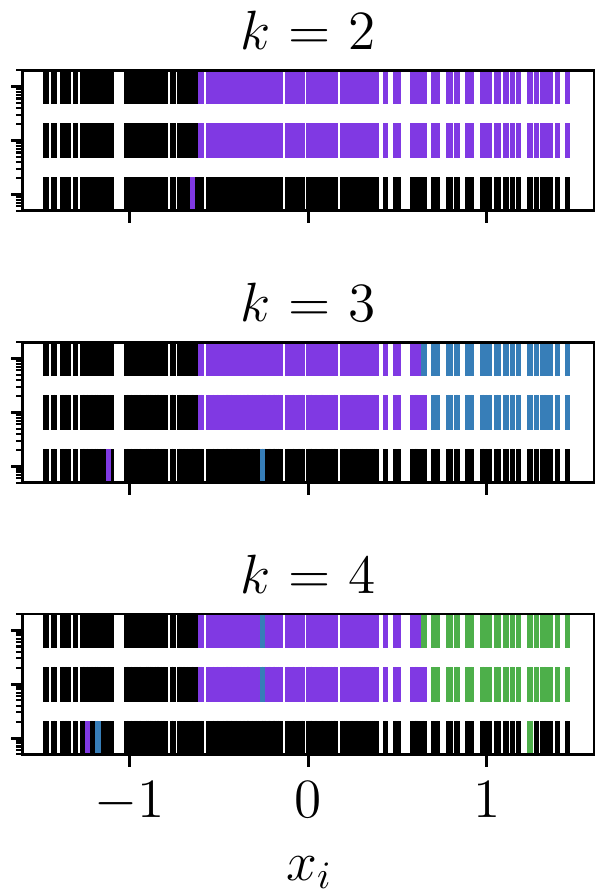}
 \includegraphics[height=5.1cm,trim={.25cm .5cm .4cm .5cm},clip]{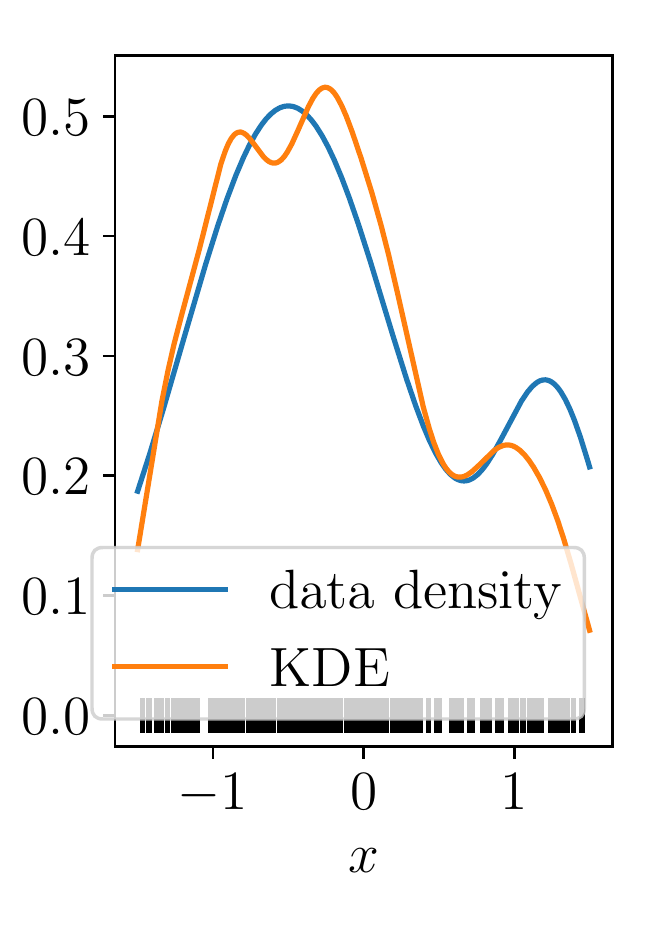}  

\  \\
  
   \includegraphics[height=4cm,trim={.6cm .6cm .6cm .5cm},clip]{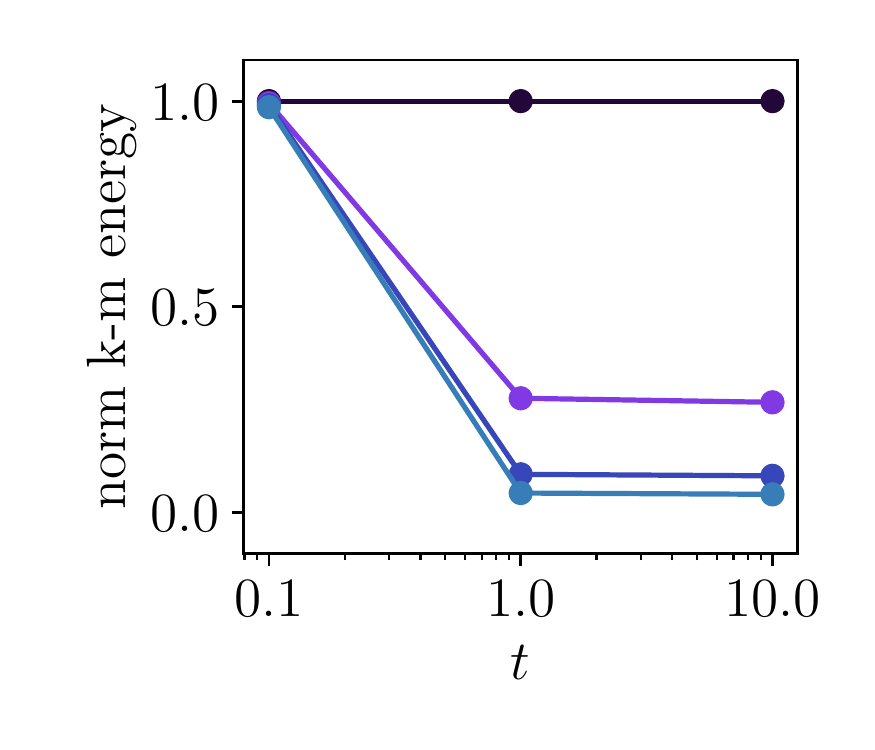}    \includegraphics[height=4cm,trim={2.3cm .6cm .6cm .5cm},clip]{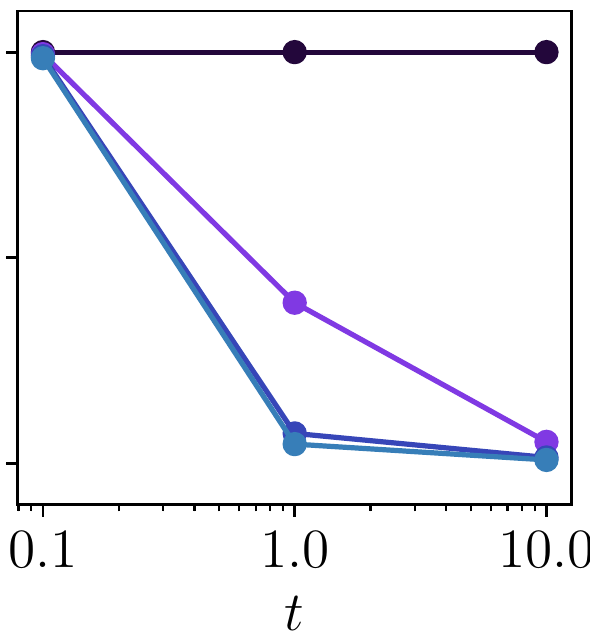}  \includegraphics[height=4cm,trim={2.3cm .6cm .5cm .5cm},clip]{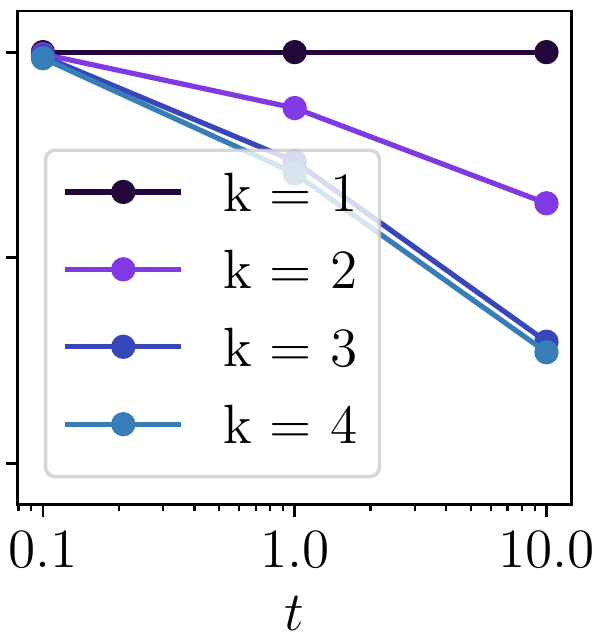} 

  \caption{Clustering performance of the graph dynamics for the transition rate matrix $Q_\beta$.  The top portion of the figure shows the results of the $k$-means clustering algorithm for $k=2,3,4$ (rows) and $\beta = 0.25, 0.9, 1.0$ (columns).  The color of a node indicates the cluster to which it belongs. On the top right, we show the data density from which $n = 204$ nodes are sampled and the KDE of the data density used to construct $Q_\beta$. The bottom of the figure   shows the value of the normalized $k$-means energy.} \label{dynamicclustering}
\end{figure}

\subsubsection{Clustering Dynamics} \label{sec:Clustering Dynamics}
In Figure \ref{dynamicclustering}, we illustrate how the graph dynamics $u_t$ of the transition rate matrix $Q_\beta$ can be used for clustering. The underlying data density is $\rho_{\text{two bump}}$, from which we choose $n =204$ samples. We consider $\beta = 0.25, 0.9$ and $1.0$, corresponding to the three columns of the figure. The top portion of the figure  shows the results of the $k$-means clustering algorithm for $k=2,3,4$. Each plot depicts the data samples at times $t =  10^{-1}, 1, 10$, coloring the samples according to which cluster they belong. The top right panel on the figure shows the data distribution and the kernel density estimate of the data distribution, which is used to construct the transition rate matrix $Q_\beta$. The bottom of the figure shows the value of the $k$-means energy $E_k$ (\ref{kmeansenergy}) for each clustering, normalized so that $k=1$ clustering (all nodes in a single cluster) has energy $E_1 = 1$.

 For all  $\beta$ and $k$, there is poor clustering behavior early in time, $t = 0.1$, suggesting that the Fokker-Planck dynamics have not had time to effectively mix within clusters. This can be seen by comparing the colors of the nodes to the data distribution displayed on the right: a correct clustering should identify one cluster for the large bump  and another cluster for the small bump. This can also be seen by considering the $k$-means energy, which is largest at $t = 0.1$, and shows little variation for different choices of $k$.

 On the other hand, we observe the best clustering performance for $\beta =0.9$ and time $t = 10$. Examining the colors of the nodes for $k=2$ reveals that the correct clusters are found. Furthermore, this clustering remains fairly stable as $k$ is increased. This can also be seen in the $k$-means energy, which shows a substantial decrease from $k=1$ to $k=2$, but remains stable for $k=3,4$, suggesting that two clusters is the correct number of clusters.
 
 While $\beta = 0.25$ and $\beta =1$ do not offer good clustering performance, they do shed light on key properties of our method, once time is sufficiently large to have allowed the dynamics to effectively mix, $t=10$. For example, when $\beta = 0.25$, diffusion dominates the dynamics, so that the density of the data distribution does not play a strong role in clustering. In fact, we see that the clusters are almost entirely driven by the geometry of the data distribution, which is fairly uniform on the domain: when $k=2$ the clusters are essentially even halves of the domain; when $k=3$, they are even thirds; and when $k=4$, they are   even quarters. The lack of awareness of density when $\beta = 0.25$ inhibits correct cluster identification.

We observe the opposite problem when $\beta =1$. In this case, the dynamics are driven entirely by density, with no diffusion. However, the density driving the dynamics is not the exact data density, but the kernel density estimate. Due to noise in the KDE, an artificial local minimum appears near $x=-0.75$, causing $k=2$ to cluster the nodes to the left and right of this local minimum and causing $k=3$ to cluster the nodes into three even groups, separated by the two local minima of the KDE. Unlike in the case $\beta =0.9$, when $\beta =1.0$ there is no diffusion to help the dynamics overcome spurious local minima in the KDE, leading to inferior clustering performance.

We close by considering the role of the  $k$-means energies in identifying the correct number of clusters. First, consider the  case of a  uniform data distribution. In this case, the $k$-means energy for $k=2$ would be $\frac14$ and for $k=3$ would be $\frac{1}{9}$. Consequently, while the correct number of clusters for the uniform data distribution is one, the $k$-means energy still drops significantly as $k$ is increased. For this reason, we caution that looking for the largest drop of the energy alone is not a good criterion for determining the correct number of clusters. Determining the correct number of clusters remains an active area of research, including, for example, the study of gap statistics \cite{tibshirani-gap-01}, in which the energy is compared to the energy one would have if the data were uniform. % 

\begin{figure}[h]
\begin{centering}

    \includegraphics[height=2.8cm,trim={.1cm 1.75cm .6cm .5cm},clip]{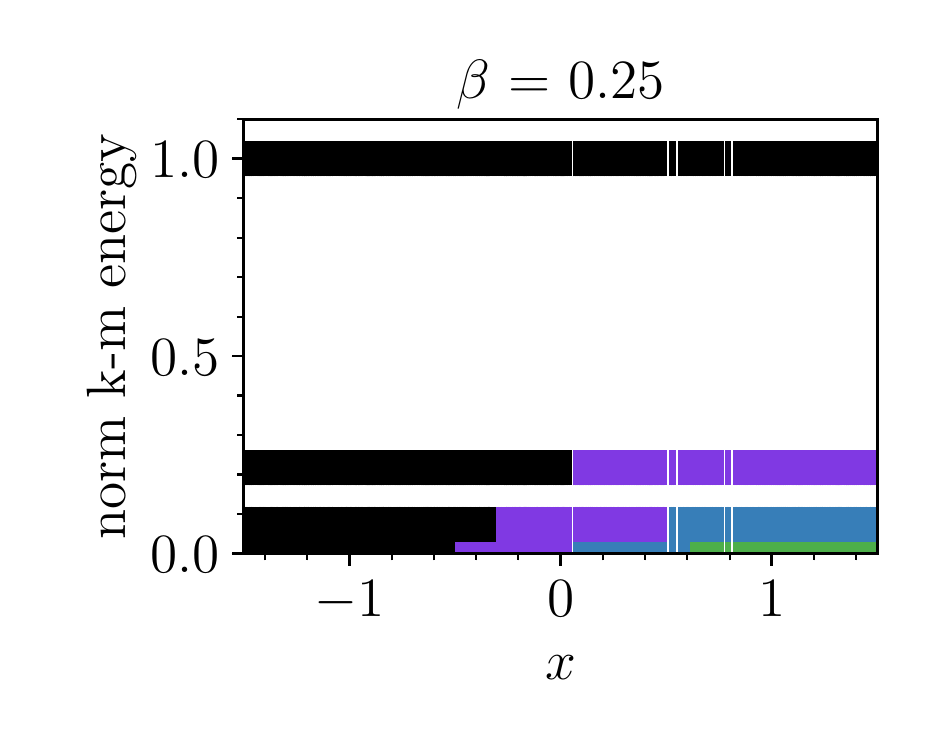}     \includegraphics[height=2.8cm,trim={2.4cm 1.85cm .6cm. .5cm},clip]{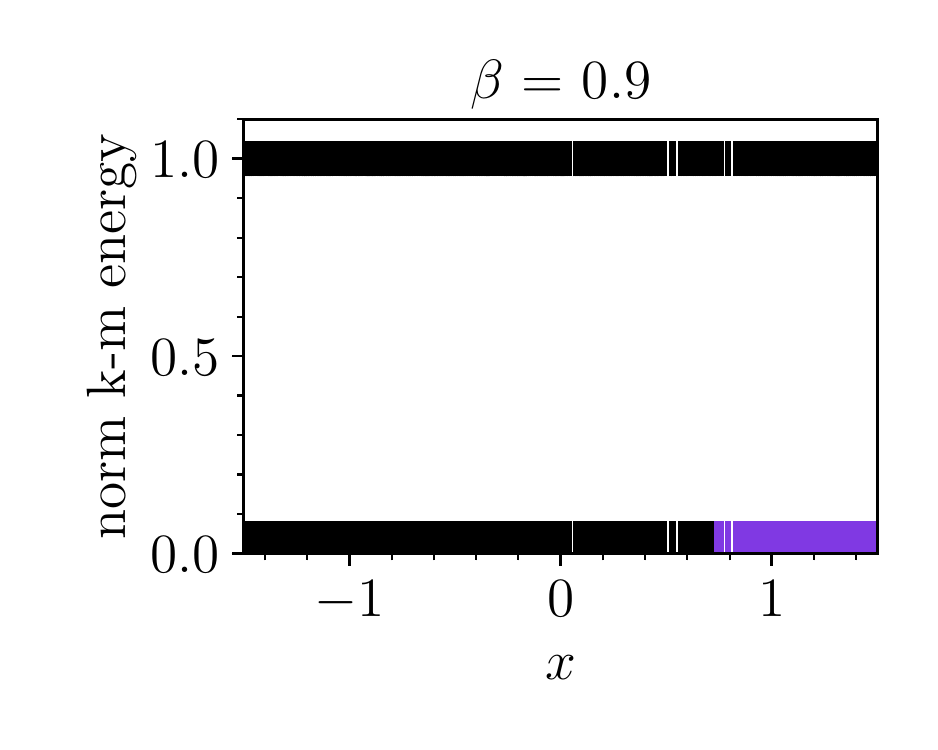}   \includegraphics[height=2.8cm,trim={2.4cm 1.85cm .6cm .5cm},clip]{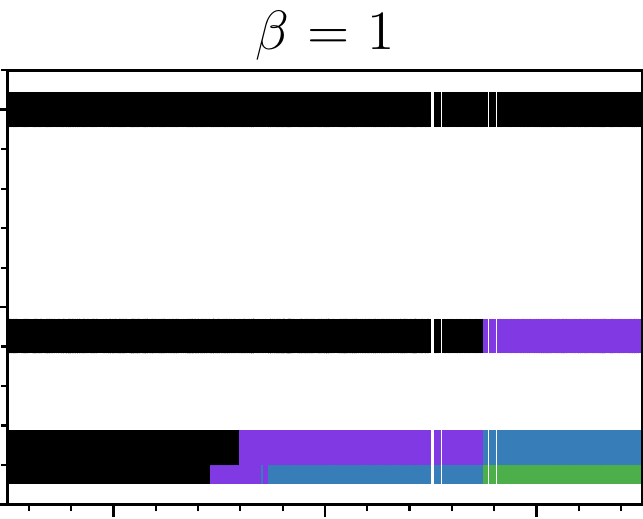} 
  \includegraphics[height=2.8cm,trim={1.3cm 1.85cm .6cm .5cm},clip]{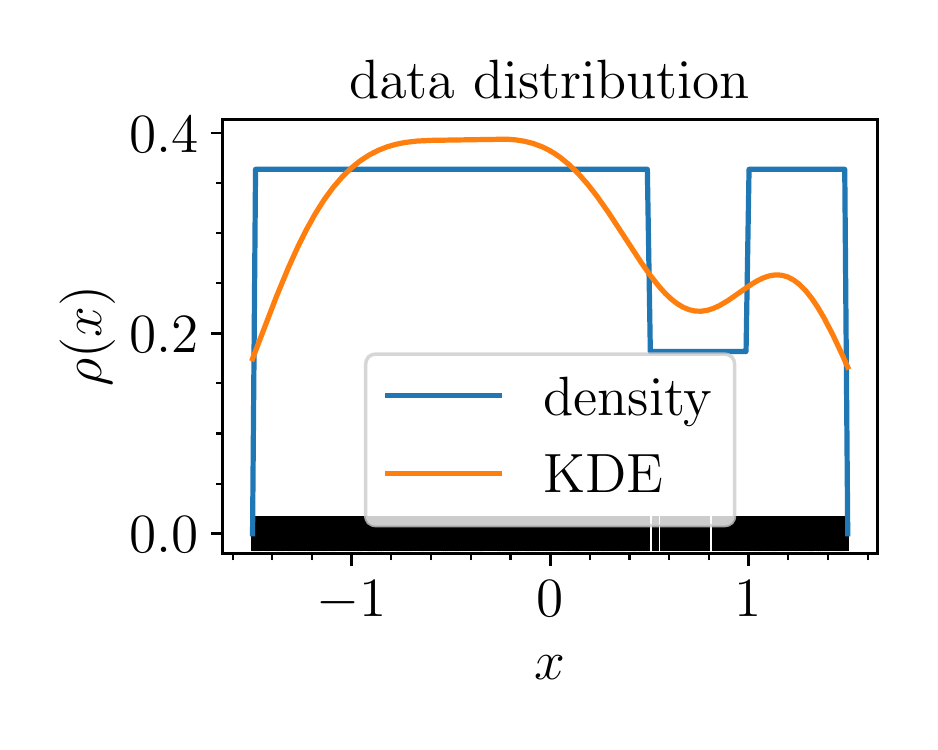}

    \includegraphics[height=2.8cm,trim={.1cm 1.75cm .6cm .5cm},clip]{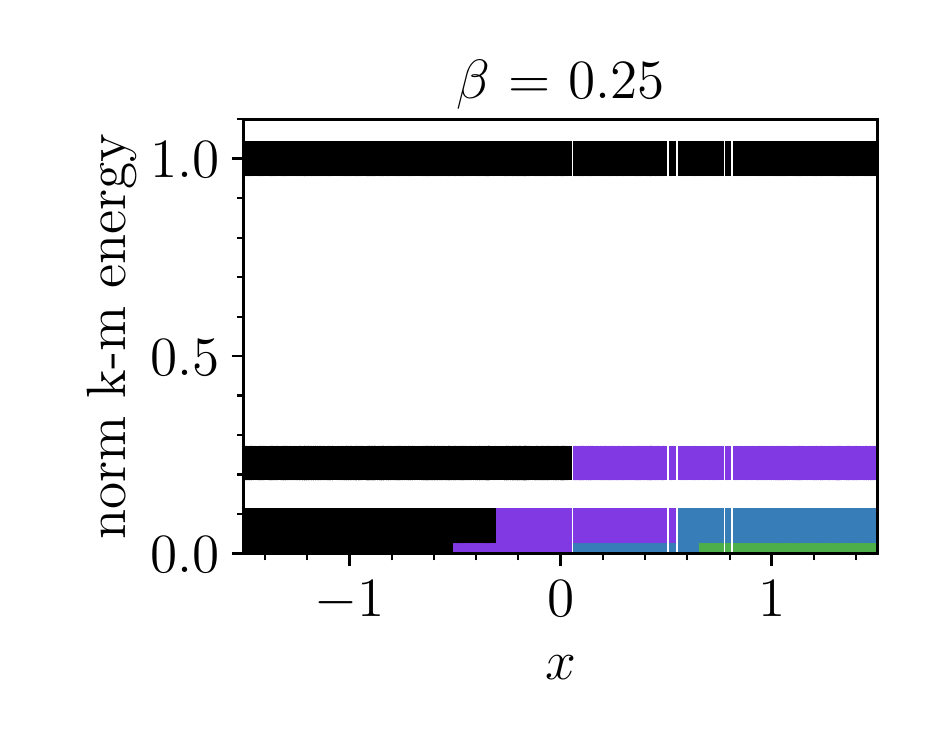}     \includegraphics[height=2.8cm,trim={2.4cm 1.85cm .6cm. .5cm},clip]{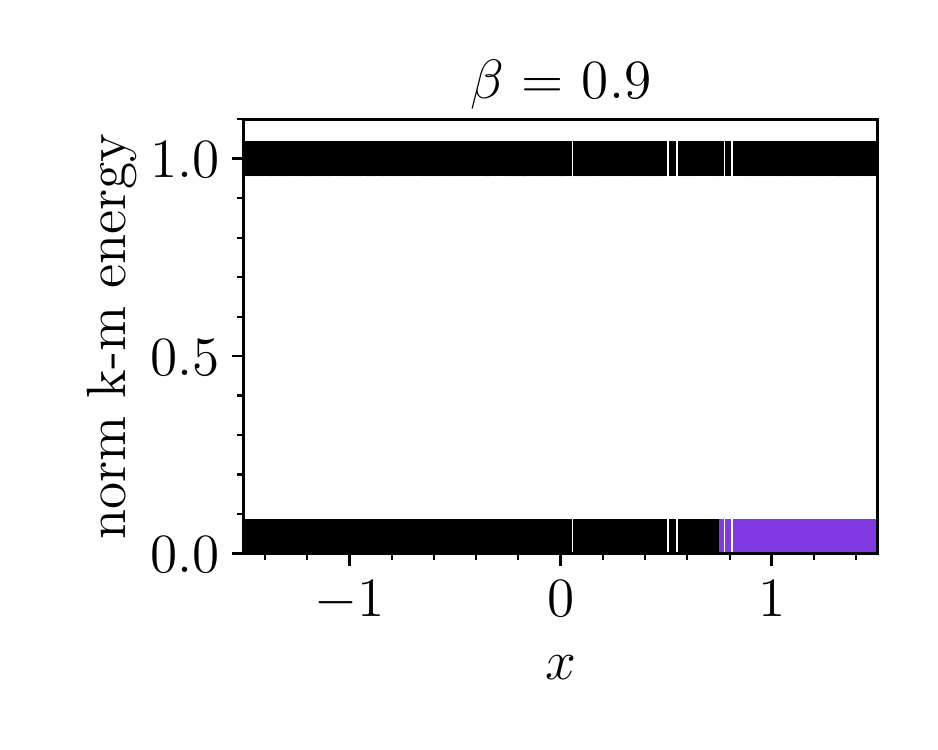}   \includegraphics[height=2.8cm,trim={2.4cm 1.85cm .6cm .5cm},clip]{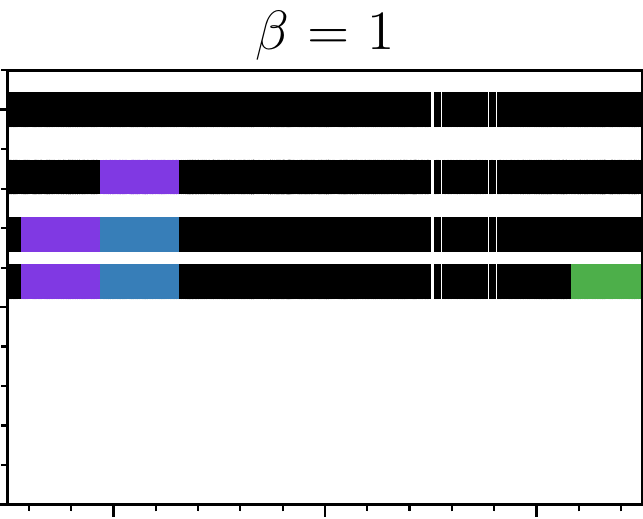} 
  \includegraphics[height=2.8cm,trim={1.3cm 1.85cm .6cm .5cm},clip]{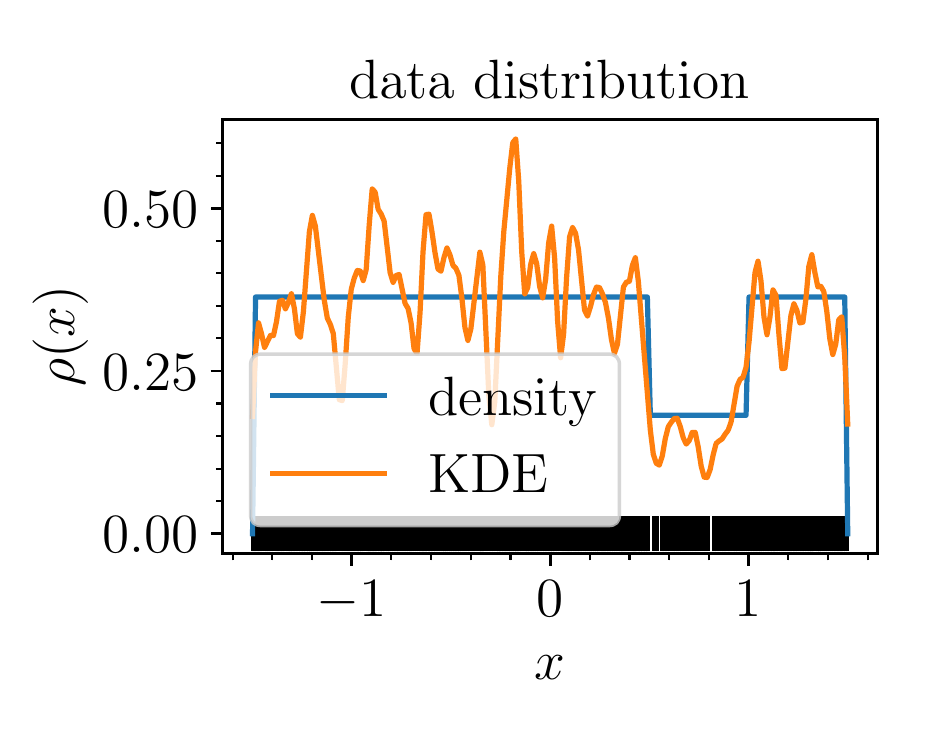}      
         
   \includegraphics[height=3.46cm,trim={.1cm .6cm .6cm .5cm},clip]{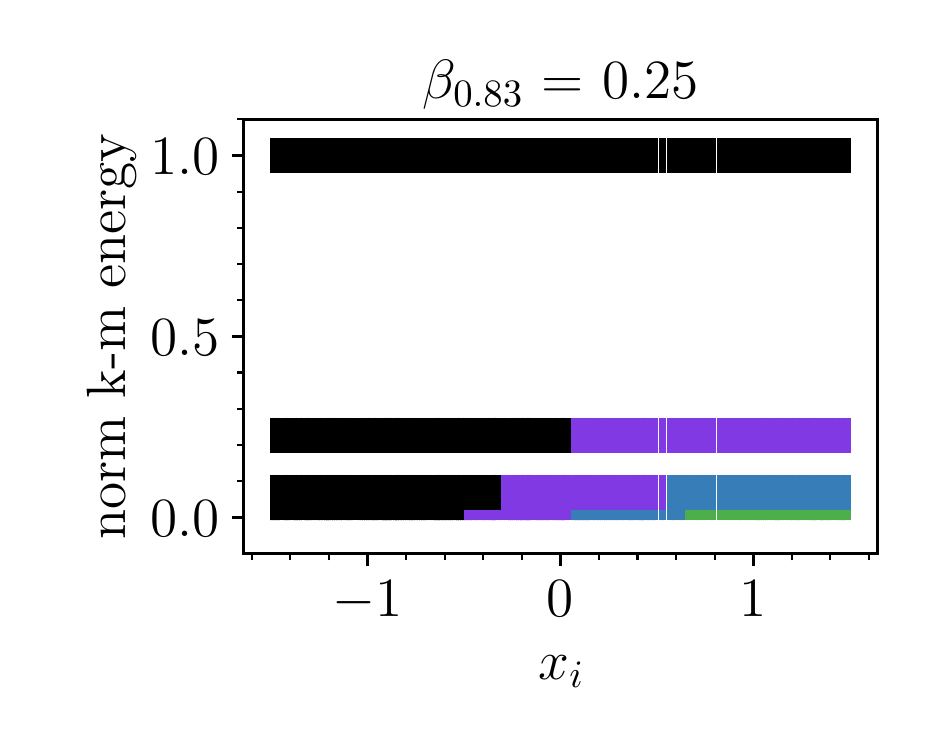}     \includegraphics[height=3.46cm,trim={2.4cm .6cm .6cm. .5cm},clip]{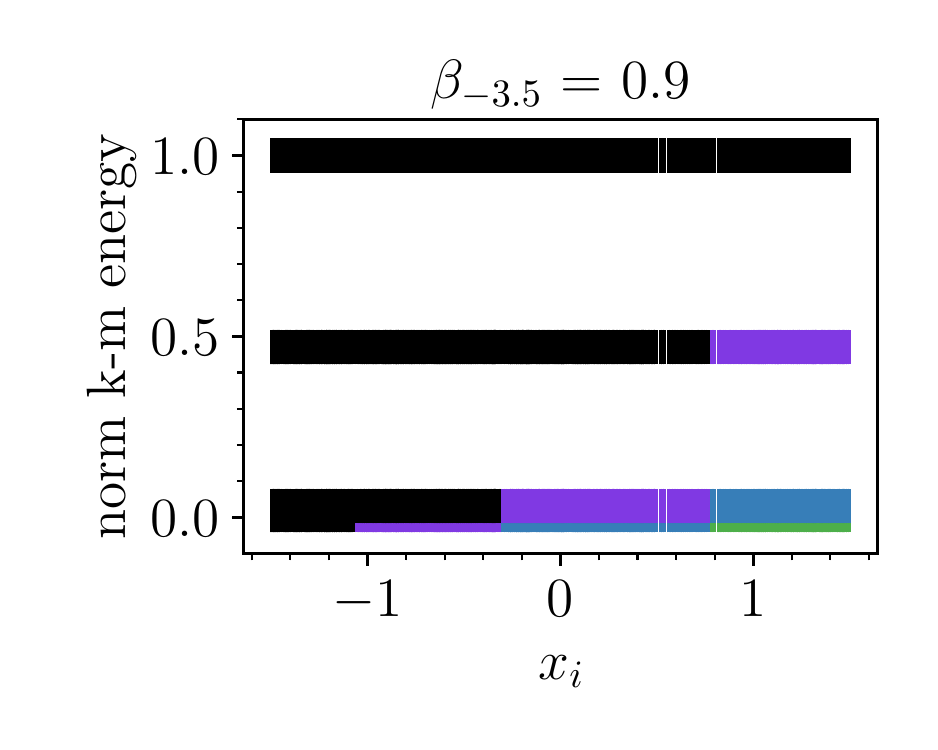}   \includegraphics[height=3.46cm,trim={2.4cm .6cm .6cm .5cm},clip]{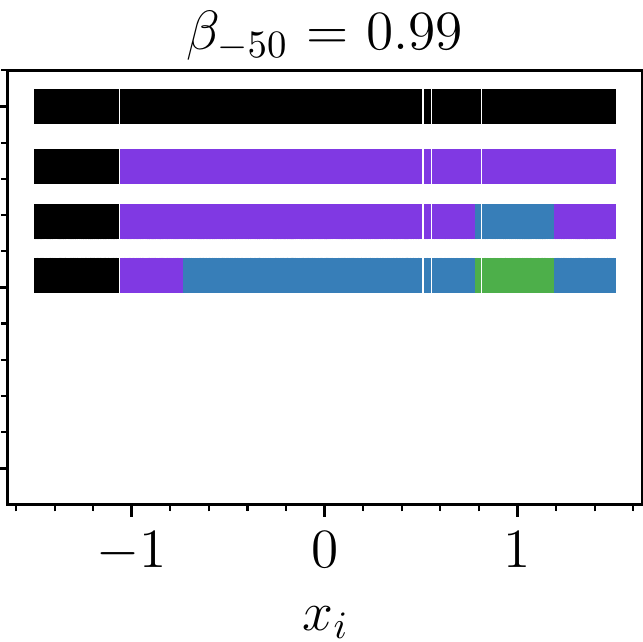} 
  \includegraphics[height=3.46cm,trim={1.3cm .6cm .6cm .5cm},clip]{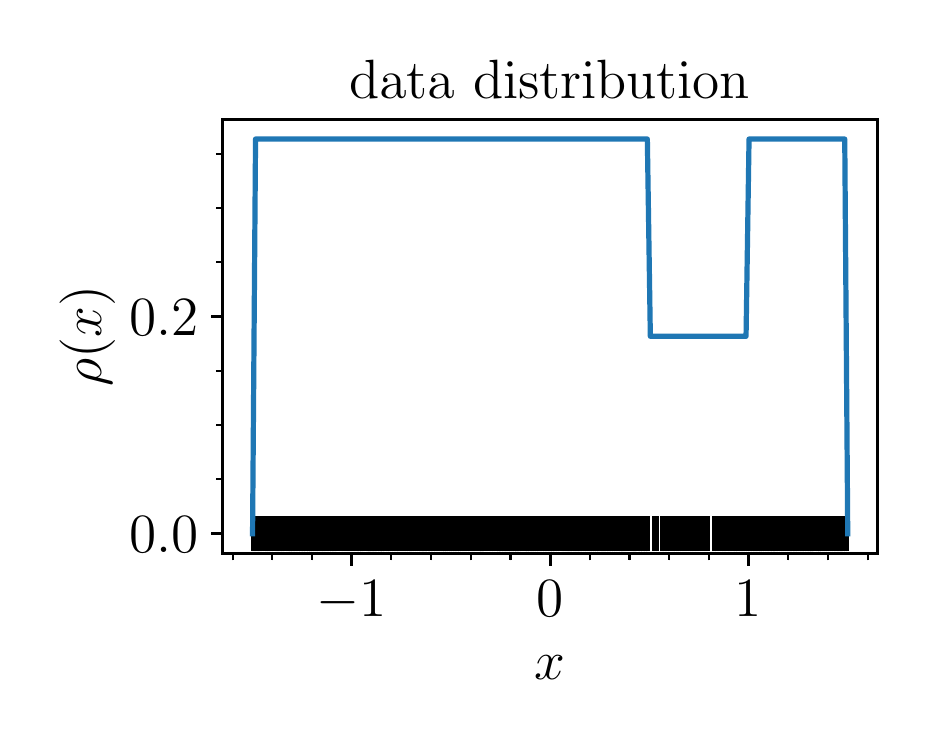}

\end{centering}
\caption{Effect  of  the bandwidth of the kernel density estimator   on clustering,   for $n = 676$ samples clustered at time $t=30$. 
First two rows show the clustering with $Q_\beta$, \eqref{Qbetadeffinal} for KDE bandwidths $\delta=0.2$ and $\delta=0.015$, while the third shows the dynamics of $Q^{rw}_\alpha$, \eqref{Qalphadeffinal}.
The first three columns show clustering performance for different balances of drift and diffusion, and the fourth column shows the data distribution and kernel density estimate. Note that, since no explicit kernel density estimate is used in the construction of $Q^{rw}_\alpha$, none is shown in the third row. The colors of the samples indicate the clusters to which they belong, and the height of the samples in each frame indicates the value of the normalized $k$-means energy \eqref{Ekmeans}.  The top row of markers in each frame corresponds to  a single cluster ($k=1$), the next one represents two clusters ($k=2$), then three and four clusters.
} \label{mesaKDEweakness}
\end{figure}

\subsubsection{Effect of the Kernel Density Estimate on Clustering} 
Figure \ref{mesaKDEweakness} illustrates the effect that the bandwidth $\delta$ of the kernel density estimate  of the data distribution has on clustering, see equation \ref{KDEkernel}.  The data distribution is given by a piecewise constant function, shown in the rightmost column. The number of samples chosen is $n = 680$, and the clustering is performed at time $t=30$. The graph connectivity parameter $\varepsilon$ is chosen as in equation (\ref{epsdef}), equaling $0.015$. The top two rows show clustering performance for $Q_\beta$ for $\beta = 0.25, 0.9, 1$, and the bottom row shows clustering performance for $Q^{rw}_\alpha$ for $\alpha = 0.83, -3.5,   -50$, where the values of $\alpha$ are chosen to give a comparable balance between drift and diffusion at the level of the continuum PDE; see equations (\ref{e:ctmFP}-\ref{betatoalpha}).

The first three columns shows the clustering results for   $\beta = 0.25, 0.90,  1.00$. The color of a marker indicates the cluster to which it belongs, and the height of the marker in the frame represents the value of the normalized $k$-means energy  (\ref{kmeansenergy}). Since the normalized $k$-means energy is  decreasing in $k$,  the top row of markers in each frame corresponds to  a single cluster ($k=1$), the next one represents two clusters ($k=2$), then three and four clusters.

In the top row, the bandwidth of the kernel density estimate  used to construct $Q_\beta$  is $\delta = 0.20$, and in the middle row, $\delta  = 0.015$. The effect of the bandwidth on the kernel density estimate can be seen in the rightmost column: the larger value of $\delta$ in the top row leads to a more accurate estimator of the data density than the smaller value of $\delta$ in the middle row.  As no explicit kernel density estimate is used to construct the transition rate matrix $Q^{rw}_\alpha$, no estimator is shown in the rightmost column of the bottom row. However, our previous numerical simulations in Figure \ref{comparegdfd1d} suggest that the dynamics of $Q^{rw}_\alpha$ most closely match the continuum Fokker-Planck equation with a steady state induced by a kernel density estimate with bandwidth $\delta = \varepsilon$ (\ref{KDEMaxwellian}). This is the motivation behind our choice of $\delta = 0.015 = \varepsilon$ in the second row, since it provides the closest comparison between the clustering dynamics of $Q^\beta$ and $Q^{rw}_\alpha$. Finally, we note that the choice of $\delta$ we suggest in equation (\ref{typicaldelta}) would lead to the choice $\delta = 0.07$, which is between the values considered 
in the top and middle rows of the figure, and leads to very similar performance for $\beta<1$.

In the top row, when the bandwidth in the KDE is   large, we observe good  clustering  performance  for $\beta = 0.9$ and $1.0$ and $k=2$.   On the other hand, $\beta =0.25$ performs poorly, since the large amount of diffusion cause the dynamics to ignore the changes in relative density and cluster   based on the fairly uniform geometry of the sampling. In the middle row, when the bandwidth in the KDE is small, we still observe good performance for $\beta =  0.9$, though $\beta =1.0$ clusters poorly: without diffusion, the dynamics cluster based on spurious local {maxima}. As before, $\beta =0.25$ identifies incorrect clusters, since it lacks information  about density. Finally, as we expected, the clustering performance in the bottom row is similar to the middle row, due to the fact that the bandwidth in the middle row was chosen to match the bandwidth of the implicit kernel density estimate which appears to drive the dynamics of $Q^{rw}_\alpha$. Note that, for the bottom row, the only way to increase the bandwidth of the implicit kernel density estimate would be to increase the graph connectivity parameter $\varepsilon$, which,  for compactly supported graph weights, would lead to a more densely connected graph and thus higher computational cost.

 \begin{figure}[h]
\begin{centering}
 
  \includegraphics[height=2.8cm,trim={.1cm 1.75cm .6cm .5cm},clip]{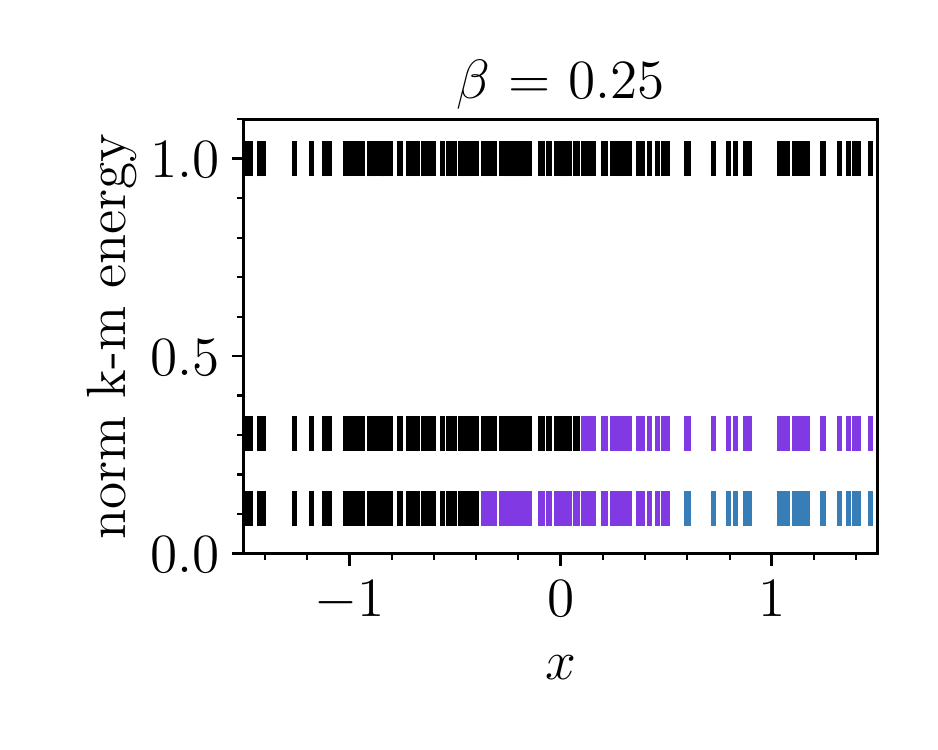}     \includegraphics[height=2.8cm,trim={2.4cm 1.85cm .6cm .5cm},clip]{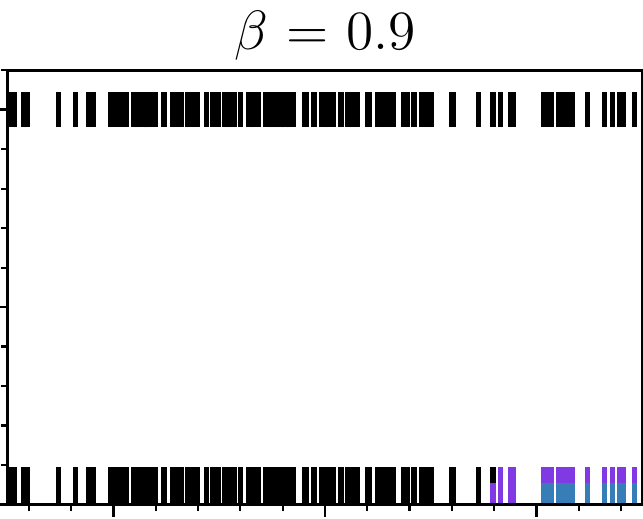}    \includegraphics[height=2.8cm,trim={2.4cm 1.85cm .6cm .5cm},clip]{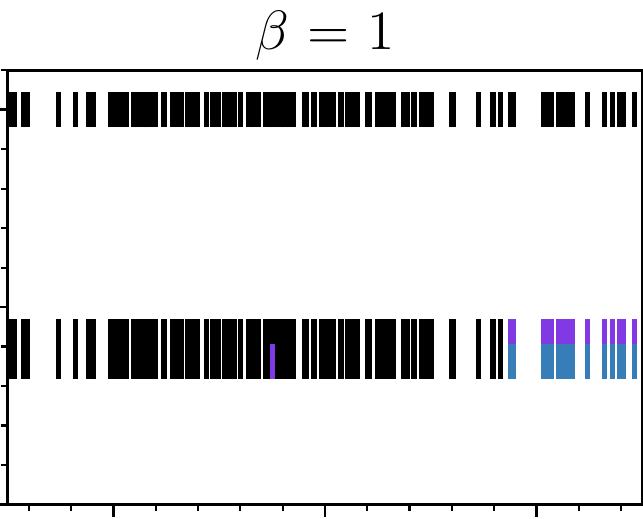} 
  \includegraphics[height=2.8cm,trim={1.3cm 1.85cm .6cm .5cm},clip]{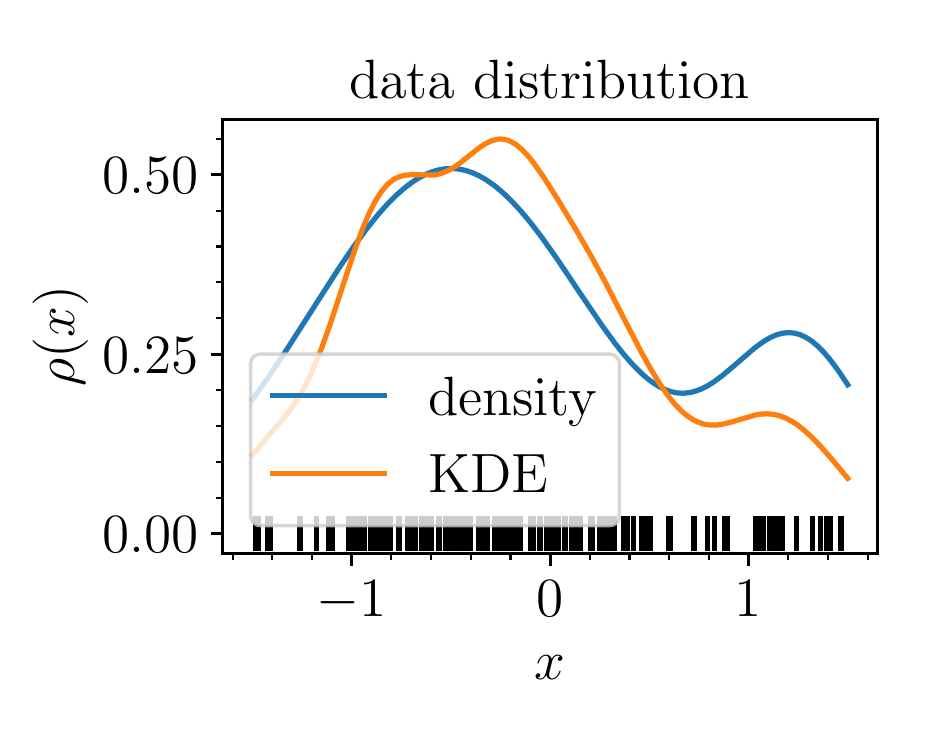}

  \includegraphics[height=2.8cm,trim={.1cm 1.75cm .6cm .5cm},clip]{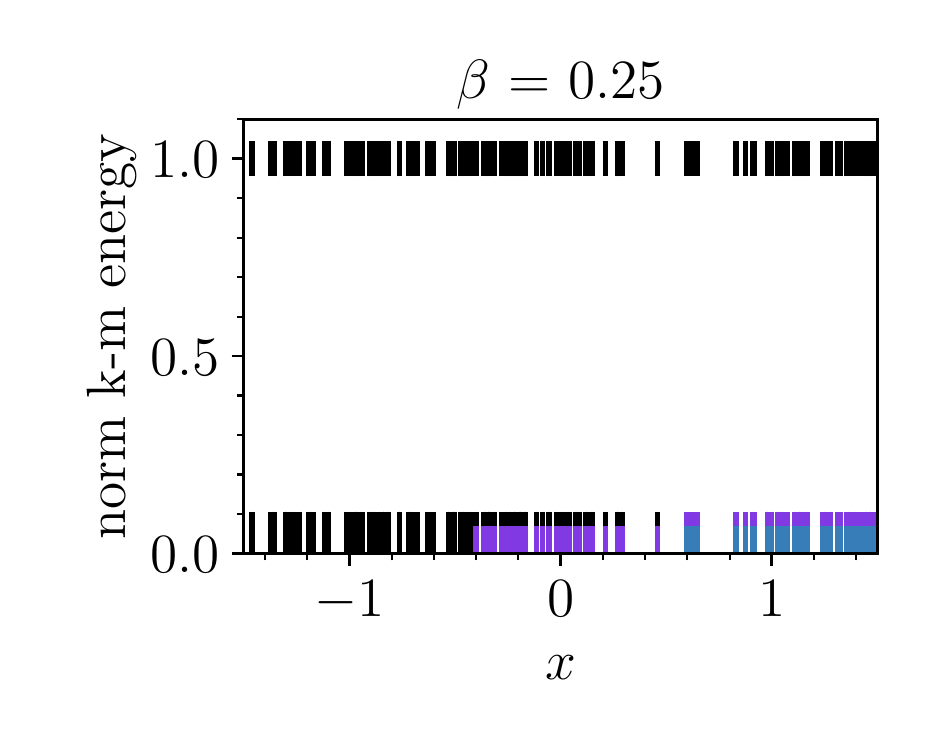}     \includegraphics[height=2.8cm,trim={2.4cm 1.85cm .6cm .5cm},clip]{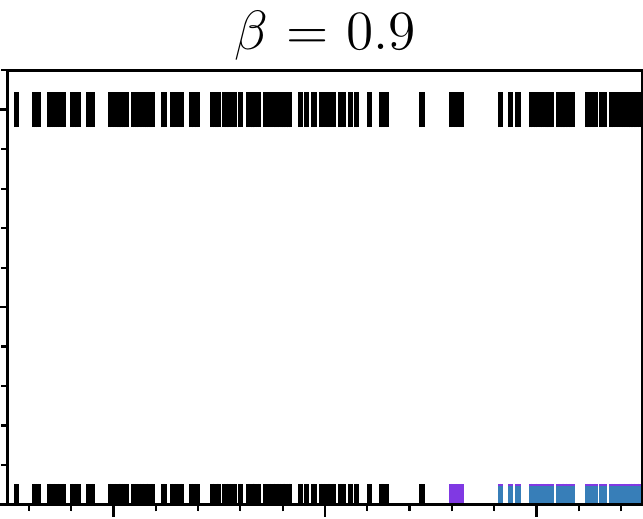}    \includegraphics[height=2.8cm,trim={2.4cm 1.85cm .6cm .5cm},clip]{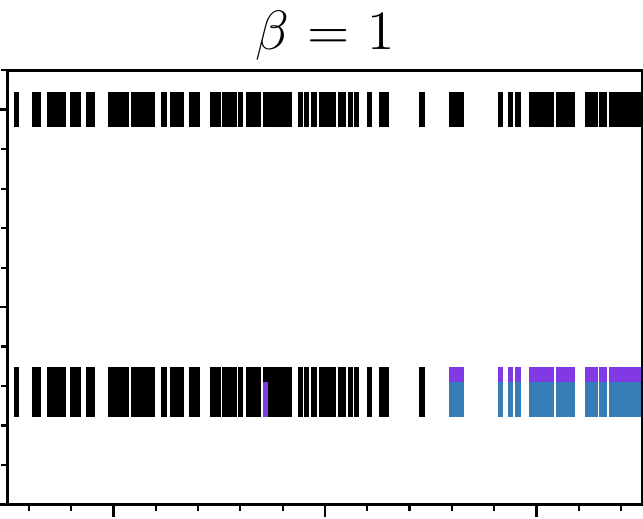} 
  \includegraphics[height=2.8cm,trim={1.3cm 1.85cm .6cm .5cm},clip]{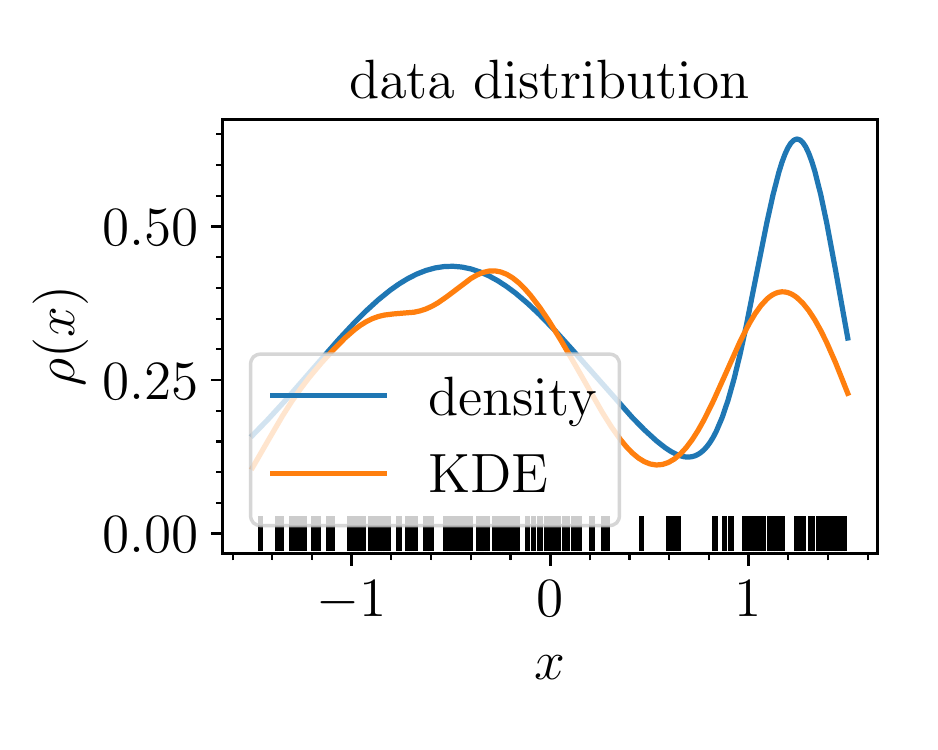}

   \includegraphics[height=3.46cm,trim={.1cm .6cm .6cm .5cm},clip]{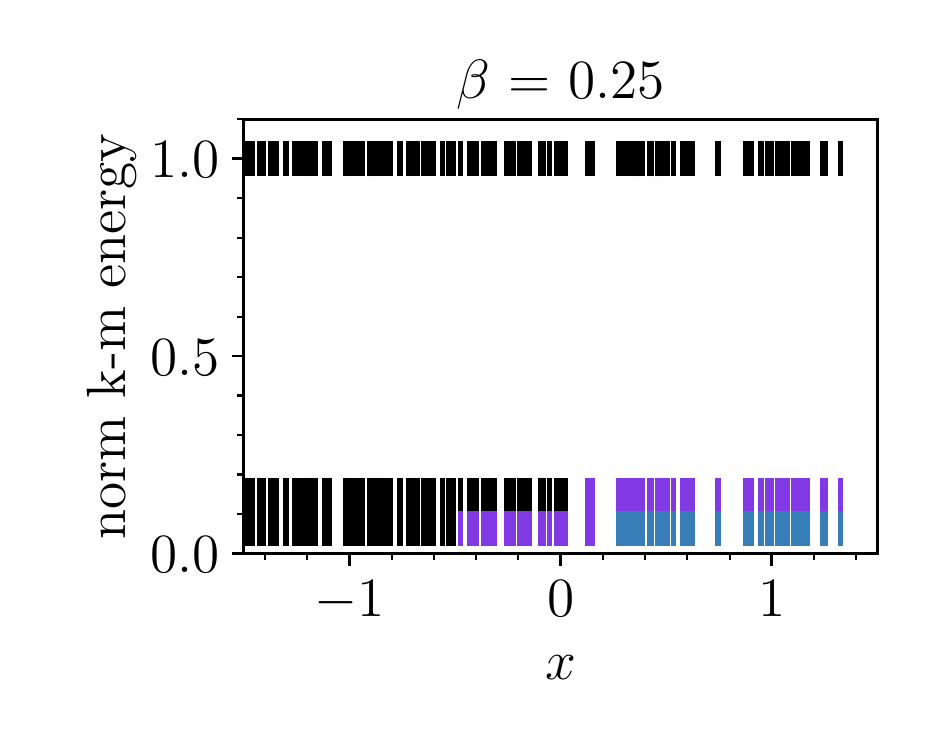}     \includegraphics[height=3.46cm,trim={2.4cm .6cm .6cm. .5cm},clip]{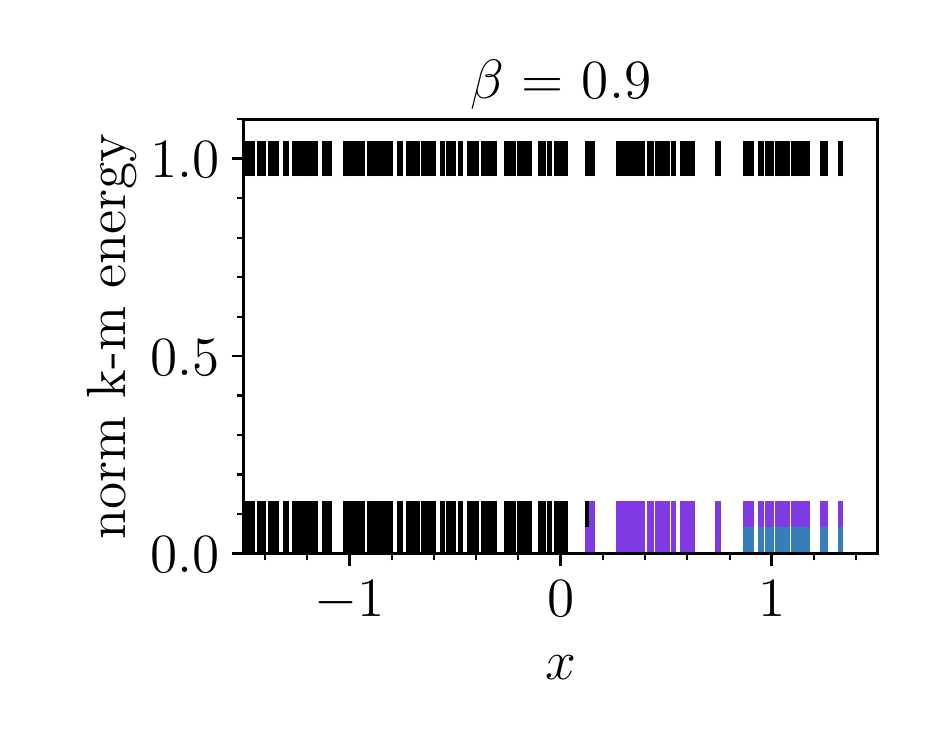}    \includegraphics[height=3.46cm,trim={2.4cm .6cm .6cm. .5cm},clip]{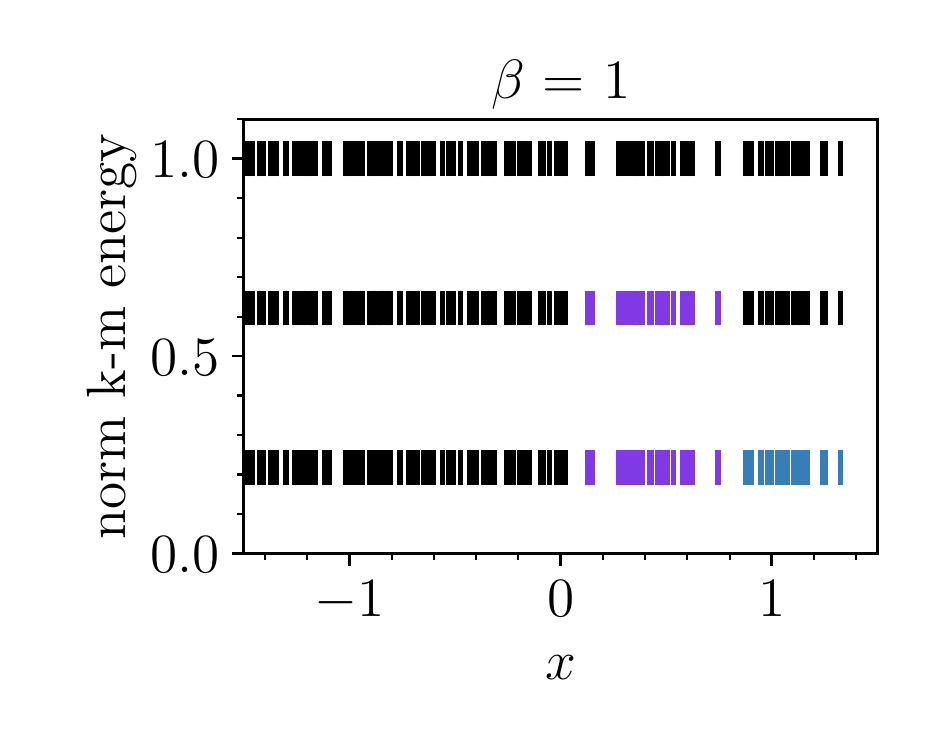} 
  \includegraphics[height=3.46cm,trim={1.3cm .6cm .6cm .5cm},clip]{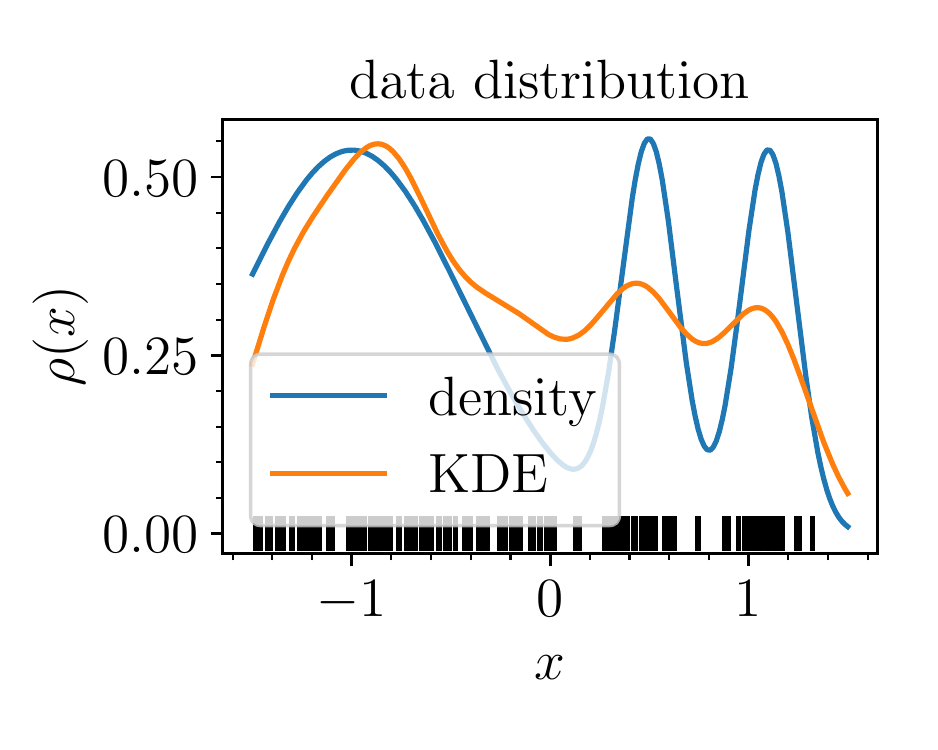}       
\end{centering}
\caption{Illustration of the effect that different choices of data distribution have on the clustering method based on $Q_\beta$ for $n = 160$ samples at time $t=30$. The first three columns illustrate different choices of $\beta$,  where the color of a marker  indicates the cluster it belongs to for $k=1,2,3$, and the height of the marker  in the frame represents the value of the normalized $k$-means energy.} \label{differentdataup}
\end{figure}

\subsubsection{Effect of data distribution on clustering} 

Figure \ref{differentdataup} illustrates the effect that different choices in data distribution have on  the clustering method based on $Q_\beta$, for $n = 160$ nodes and at time $t=30$. Each row considers a different data distribution: $\rho_{\text{twobump}}$ (see equation \ref{twobumpdef}) and
\begin{align*}
\rho_{\text{deep valley}}(x) &= 7c_0\varphi_{0.5}(x+0.5)+3c_0\varphi_{0.15}(x-1.25) ,\\
 \rho_{\text{three bump}}(x) &= c_1\varphi_{0.1}(x-0.5) +c_1 \varphi_{0.1}(x-1.1) + 4c_1\varphi_{0.4}(x+1) ,
\end{align*}
where $c_0,c_1>0$ are normalizing constants chosen so the densities integrate to one on $\Omega = [-1.5,1.5]$. The right column shows the data density, the sample of $n=160$ nodes, and the kernel density estimate used to construct the transition rate matrix $Q_\beta$.
The first three columns shows the clustering results for   $\beta = 0.25, 0.90,  1.00$. The color of a marker indicates the cluster to which it belongs for $k=1,2,3$, and the height of the marker in the frame represents the value of the normalized $k$-means energy  (\ref{kmeansenergy}). 

In the top row, for $\rho_{\text{two bump}}$, we observe good clustering performance for all $\beta \geq 0.9$ and poor performance for $\beta =0.25$: due to the good behavior of the KDE for this data distribution, problems do not arise as $\beta \to 1$, and as usual, $\beta = 0.25$ suffers due to the dominance of diffusion. In the middle row, for $\rho_{\text{deep valley}}$, we again observe good performance for all $\beta \geq 0.9$, and we even observe good performance for $\beta =0.25$ when $k=2$. This is due to the sparse sampling at the deep valley, which leads to a change in the geometry of the nodes: a gap that even diffusion dominant dynamics can detect. Finally, in the bottom row, for $\rho_{\text{three mountains}}$, we observe good performance for $\beta \geq 0.9$. Again, $\beta = 0.25$ is  able to capture some correct information when $k=2$, due to the sparsity of the data near the left valley, but it fails at the most relevant $k=3$.

\subsubsection{Blue Sky Problem}
\label{sec:BlueSky}
\begin{figure} 
\hspace{-15.5cm} $\beta = 0.20$ \\  \includegraphics[height=3.55cm,trim={.7cm 1.2cm 2.8cm .5cm},clip]{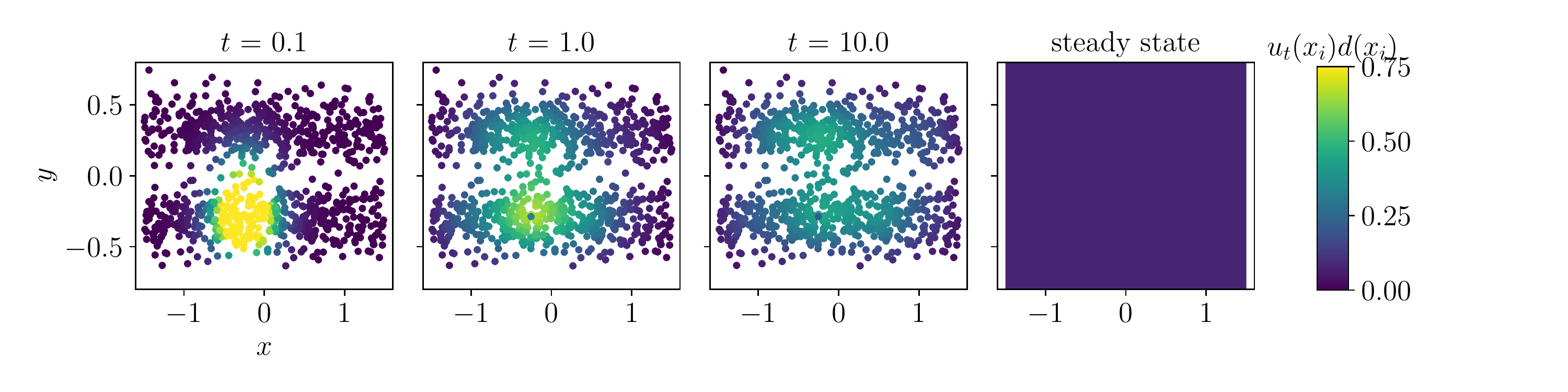} \\
\hspace{-15.5cm} $\beta = 0.95$ \\
\hspace{-1.4cm}  \includegraphics[height=3.5cm,trim={.7cm .6cm 5.2cm 1.15cm},clip]{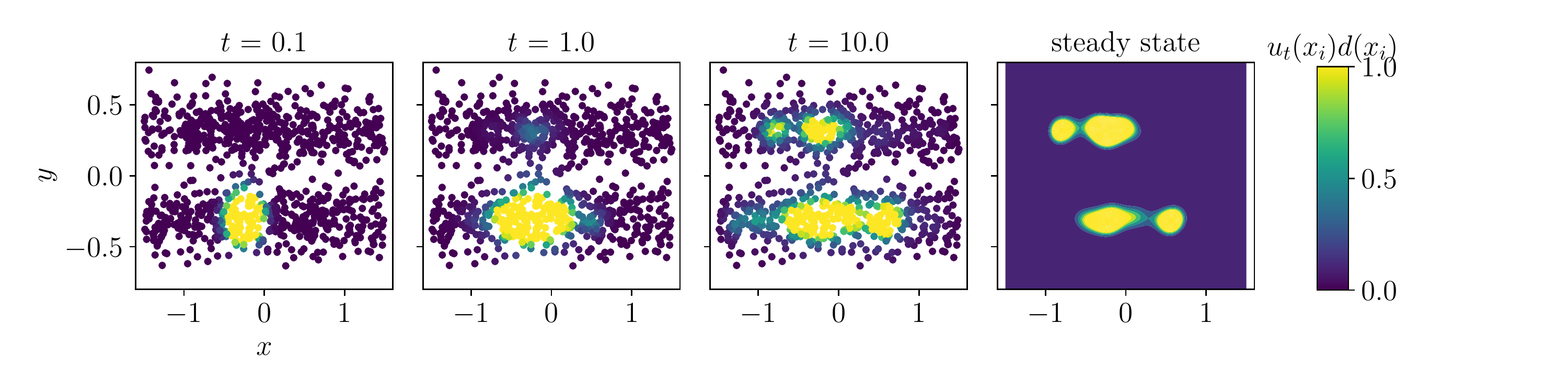} 
\caption{Graph dynamics of $Q_\beta$ on $\rho_{\text{blue sky}}$ for $n = 965$ samples for  $\beta = 0.20$ (top) and $\beta = 0.95$ (bottom). The initial condition is $\delta_{x_i, y_i}$ for $(x_i,y_i) = (-0.26,-0.29)$. In the first three columns, the dots  represent the locations of the samples, and the colors of the markers represent the value of $ u_t(x_i) d(x_i) $. In the fourth column, we plot  the steady state of the corresponding continuum PDE (\ref{KDEMaxwellian}).  } \label{BlueSkyDynamics} 
 
\end{figure}

In Figure \ref{BlueSkyDynamics}, we consider the graph dynamics of $Q_\beta$ on a two dimensional data distribution inspired by the blue sky problem from image analysis, for $n = 965$ samples on the domain $\Omega = [-1.5,1.5] \times [-1,1]$. We choose $ \veps = 0.04$ and $\delta = 0.10$, in order to optimize agreement between the discrete dynamics and the  steady state of the continuum PDE (\ref{KDEMaxwellian}).

In simple terms, the blue sky problem can be described as a setting in which data points are distributed over two elongated clusters that are separated by a narrow low density region. For concreteness, we model this  with a density  of the form:
\begin{align*}
\rho_{\text{blue sky}}(x,y) = \varphi_{1.0}(x)*(\varphi_{0.09}(y-0.32) + \varphi_{0.09}(y+0.32)) .
\end{align*}
In the top row, we choose $\beta = 0.20$, and in the bottom row, we choose $\beta = 0.95$. In both cases, we choose the initial condition for the dynamics to be $\delta_{x_i, y_i}$ for $(x_i,y_i) = (-0.26,-0.29)$. As in Figure \ref{fig:KDEofgraphdynamics}, the markers in the first four columns represent the locations of the samples, which form the nodes of our graph, and the colors of the markers represent the value of $  u_t(x_i)d(x_i)$ at each node. In the rightmost column, we plot  the steady state of the  continuum PDE (\ref{KDEMaxwellian}). We observe good agreement between the graph dynamics and the steady state by time $t= 10.0$. In the case $\beta = 0.95$,   the diffuse profile of the steady state illustrates that there is significant diffusion, in spite of the fact that $\beta$ is close to one.

\begin{figure}
\begin{centering}
\hspace{-.4cm}  $\beta = 0.20$ \hspace{1.2cm} $\beta = 0.95$ \hspace{1.3cm} $\beta =1.00$

\hspace{-1.6cm}   \includegraphics[height=8cm,trim={.6cm .5cm .6cm 5.25cm},clip]{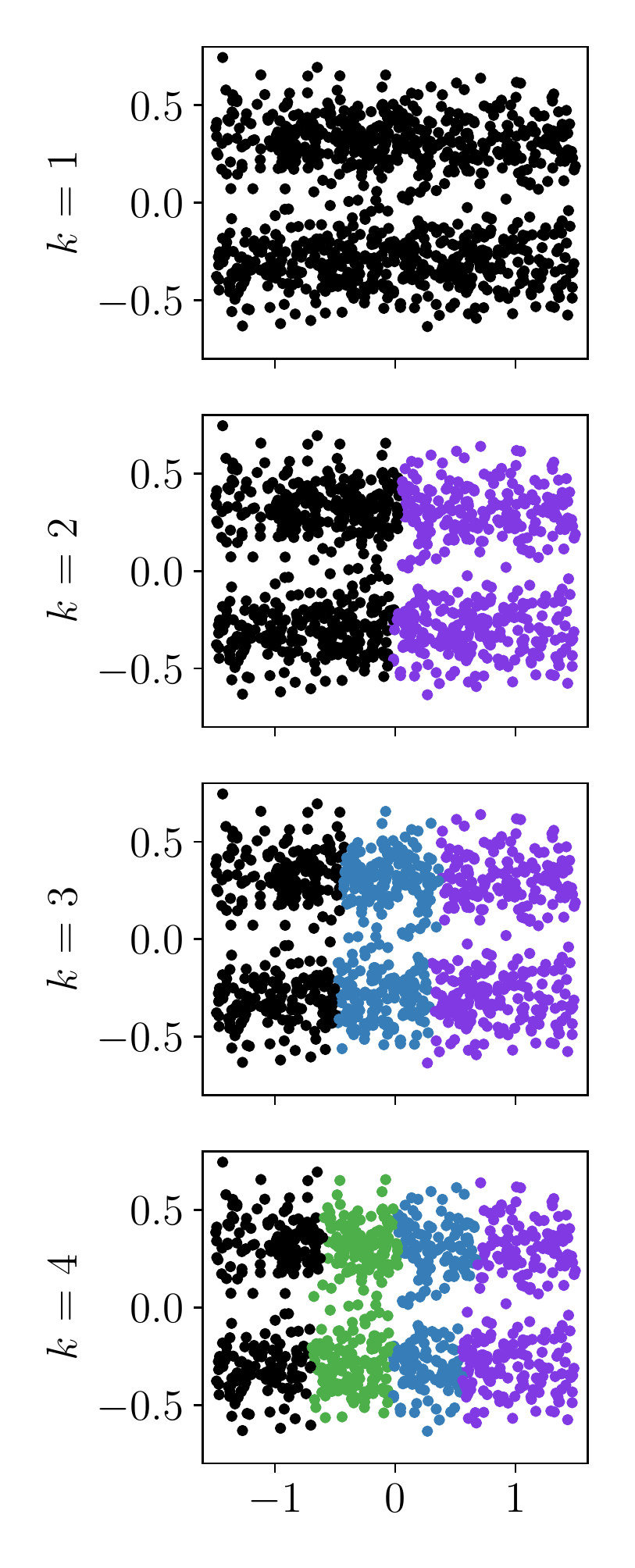} 
\includegraphics[height=8cm,trim={2.5cm .5cm .6cm 5.25cm},clip]{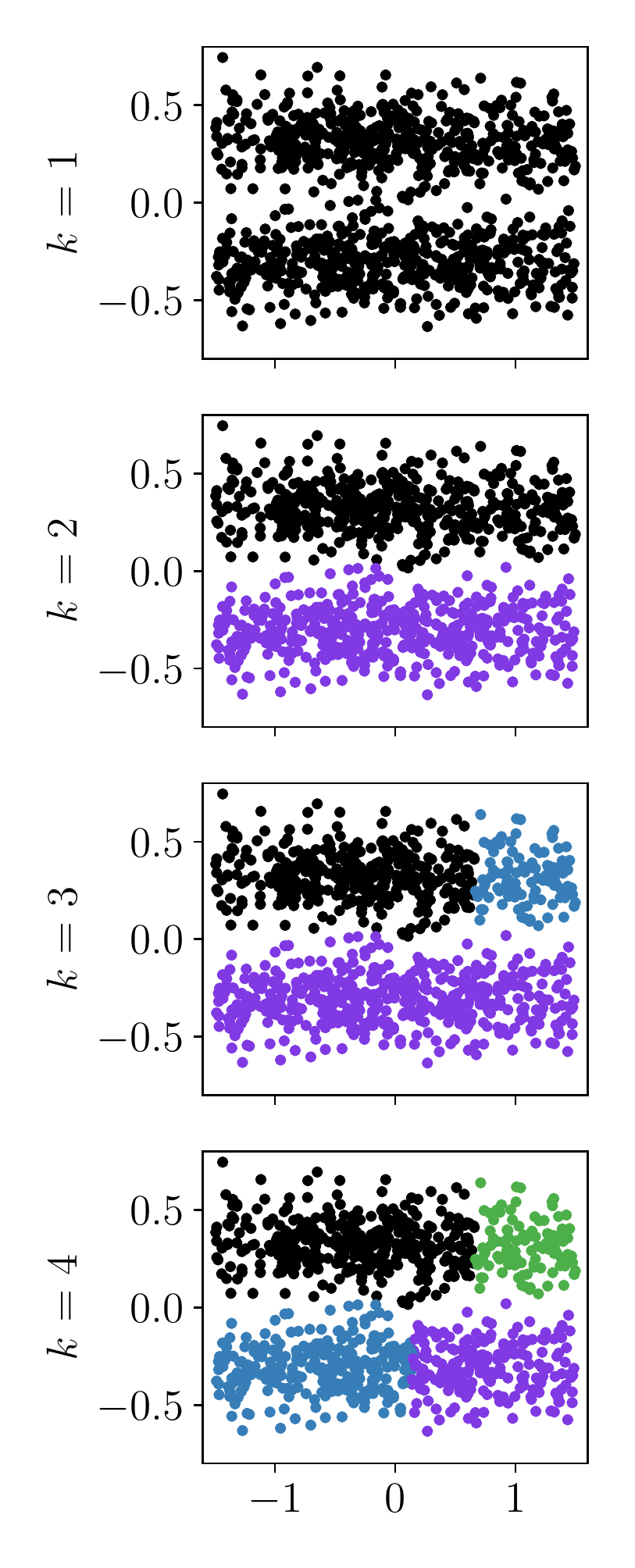} 
\includegraphics[height=8cm,trim={2.5cm .5cm .6cm 5.25cm},clip]{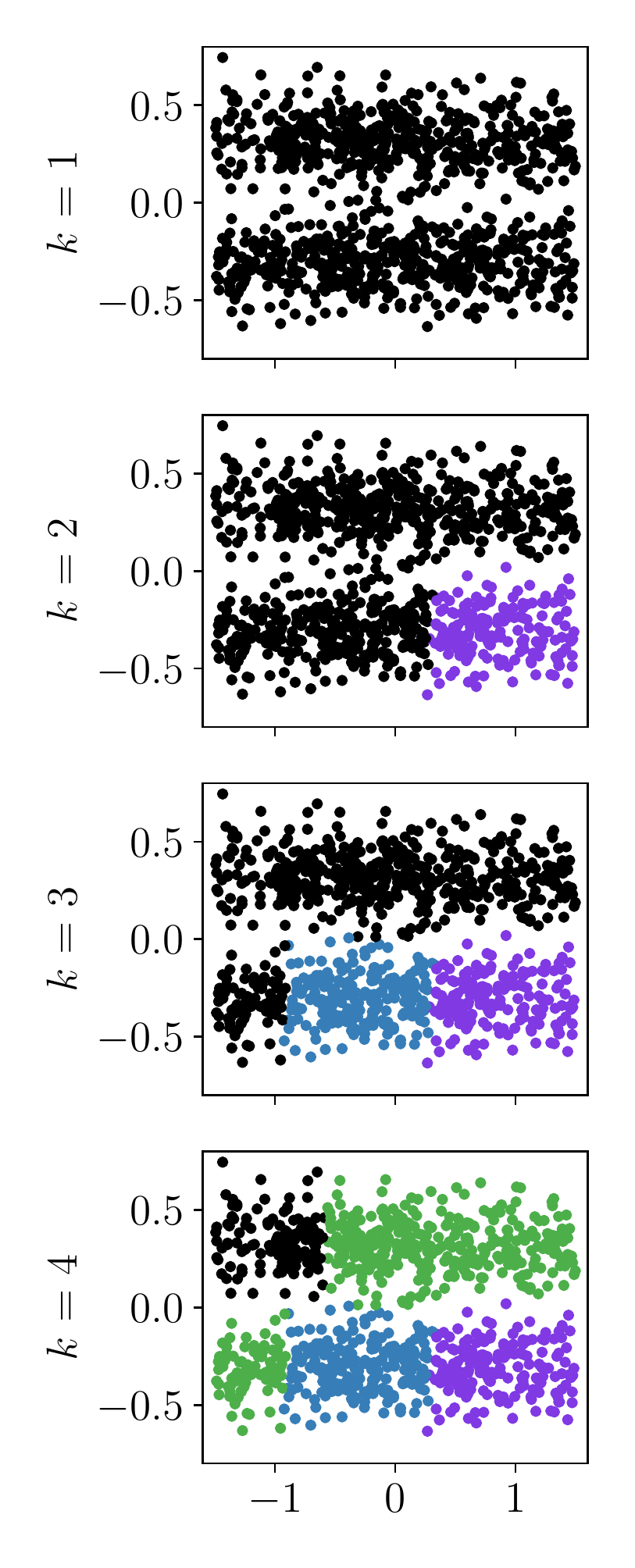} 

\hspace{-1.2cm}       \includegraphics[height=2.1cm,trim={.5cm 1.2cm .9cm 1.1cm},clip]{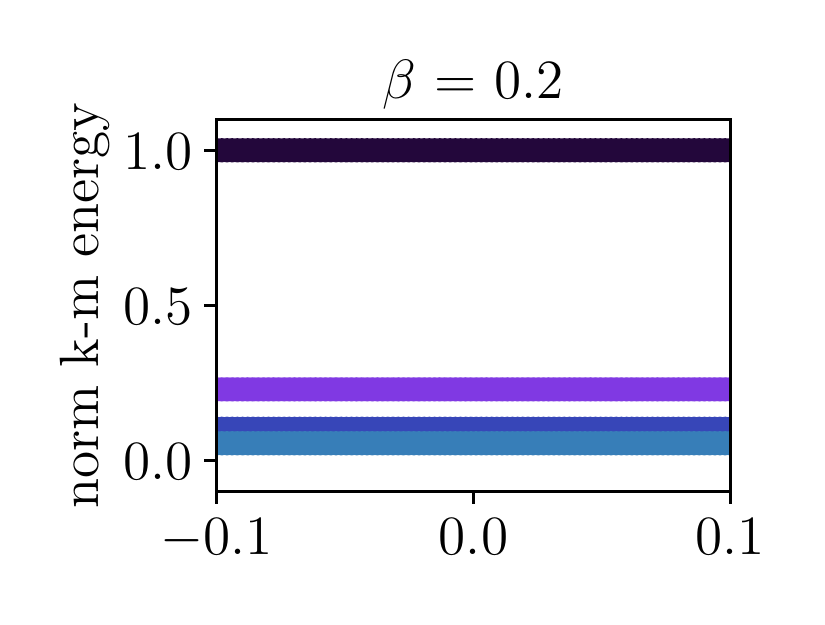}    
    \includegraphics[height=2.05cm,trim={2cm 1.2cm .9cm 1.2cm},clip]{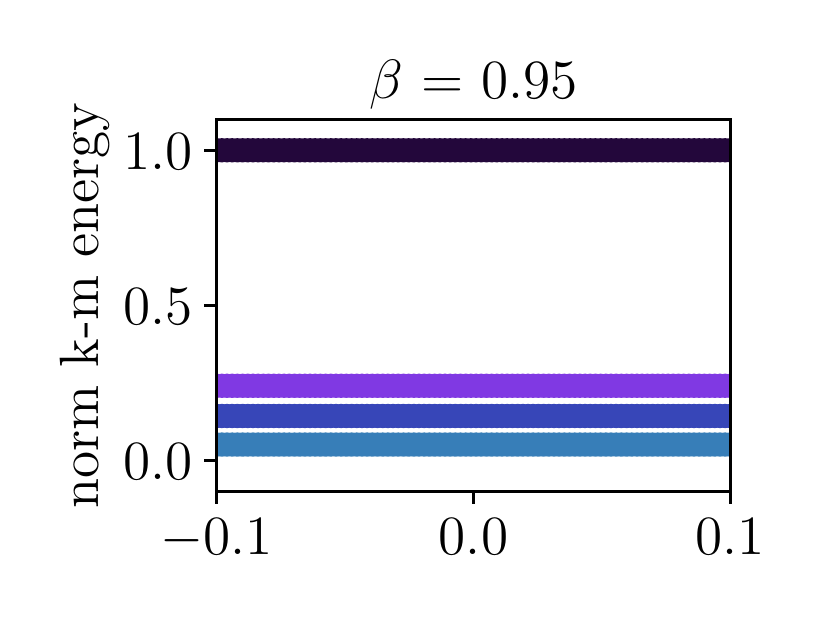}    
    \includegraphics[height=2.05cm,trim={2cm 1.2cm .9cm 1.2cm},clip]{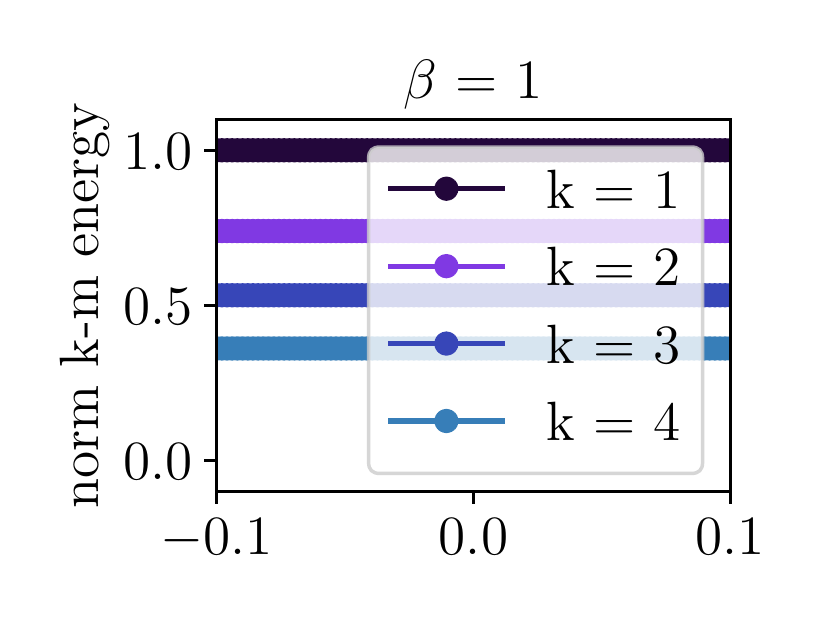}    
\end{centering}
\caption{Illustration of the clustering behavior of   $Q_\beta$ on $\rho_{\text{blue sky}}$ for $n = 965$ samples at time $t = 10$. The first three rows show the clustering behavior for $k=2,3,4$, with each node colored according to which cluster it belongs. The fourth row shows the normalized $k$-means energy for each number of clusters $k$. } \label{blueskyclustering}
\end{figure}

In Figure \ref{blueskyclustering}, we show the clustering behavior of   $Q_\beta$ on $\rho_{\text{blue sky}}$ for $n = 965$ samples at time $t= 10$. The columns correspond to $\beta =0.2, 0.95, 1.0$. The first three rows show the clustering behavior for $k=2,3,4$, with each node colored according to which cluster it belongs. In the fourth row, we show   the normalized $k$-means energy for each choice of $\beta$.

We observe the best clustering performance for $\beta =0.95$ and $k=2$. Furthermore, in this case, the plot of the normalized $k$-means energy  indicates that higher values of $k$ do not lead to significant decreases of the energy, providing further evidence that $k=2$ is the correct number of clusters. The clustering performance deteriorates for both $\beta =0.2$ and $\beta = 1.0$. In the case of $\beta = 0.2$, diffusion dominates and the clustering is based on the geometry of the sample points, preferring to cluster by slicing the sample points evenly in two pieces via the shortest cut through the dataset. In the case of $\beta = 1.0$, we expect that inaccuracies in the kernel density estimation lead to spurious local minima, and in the absence of diffusion to help overcome these local minima, incorrect clusters are found. Note that simply increasing the bandwidth  $\delta$ of kernel density estimate in this case would not necessarily improve performance for $\beta =1.0$, since for a large enough bandwidth, the two lines would merge into one line.

\begin{figure}
\hspace{.7cm}  $\beta = 0.20$ \hspace{1.2cm} $\beta = 0.95$ \hspace{1.3cm} $\beta =1.00$ \\
\includegraphics[height=2.5cm,trim={.6cm 1.2cm .6cm .5cm},clip]{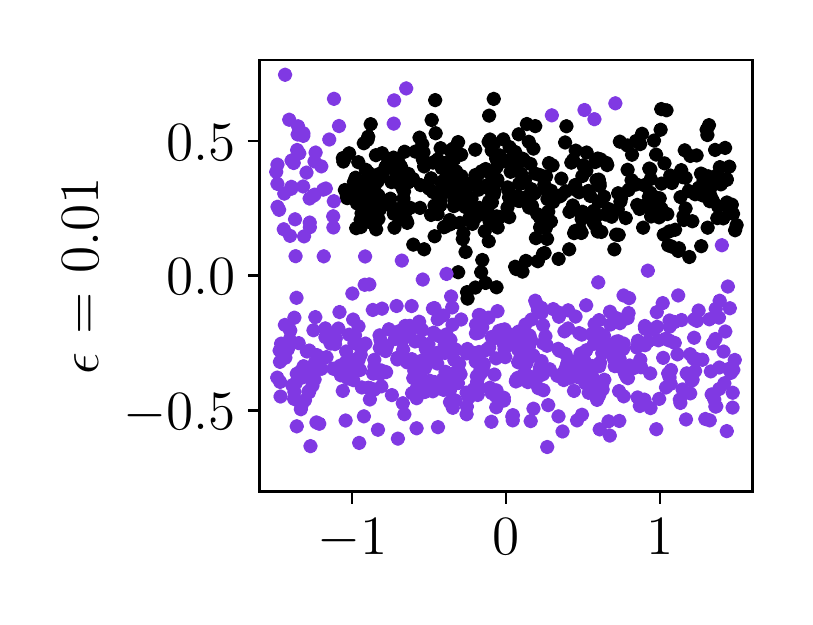}
\includegraphics[height=2.5cm,trim={2.4cm 1.2cm .6cm .5cm},clip]{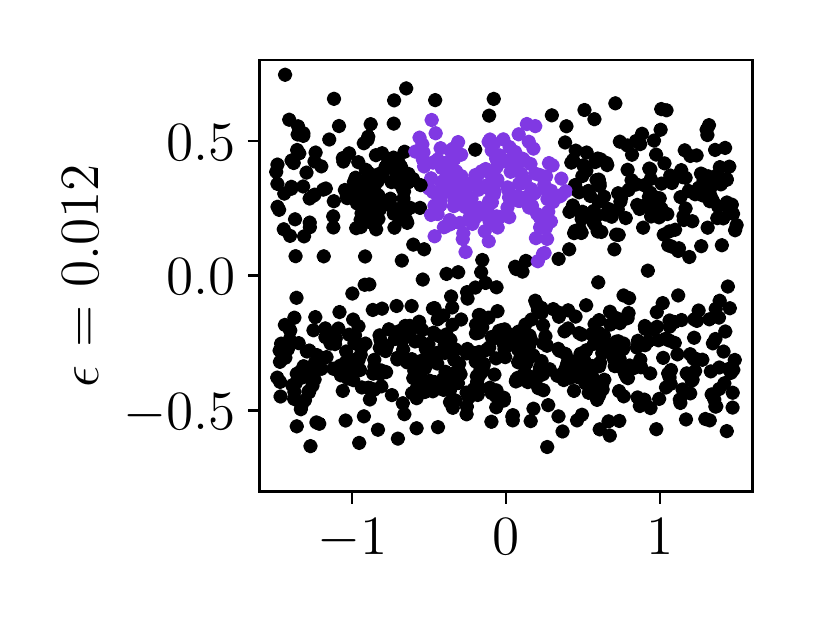} \includegraphics[height=2.5cm,trim={2.4cm 1.2cm .6cm .5cm},clip]{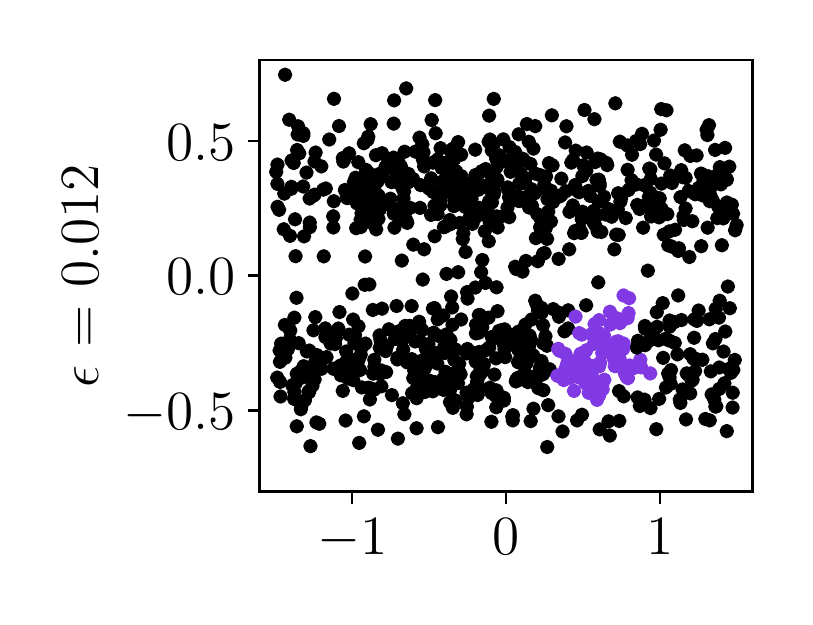}  \\
\includegraphics[height=2.5cm,trim={.6cm 1.2cm .6cm .5cm},clip]{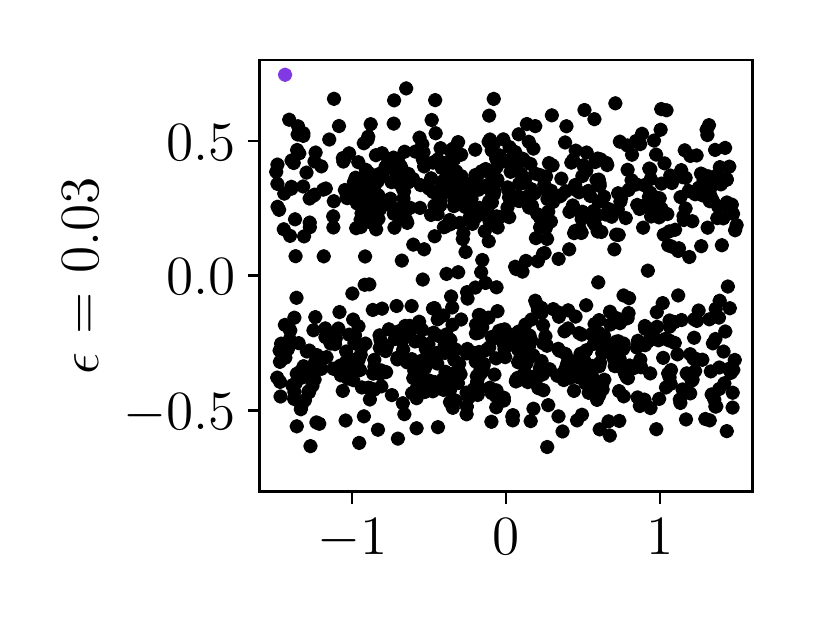}
\includegraphics[height=2.5cm,trim={2.4cm 1.2cm .6cm .5cm},clip]{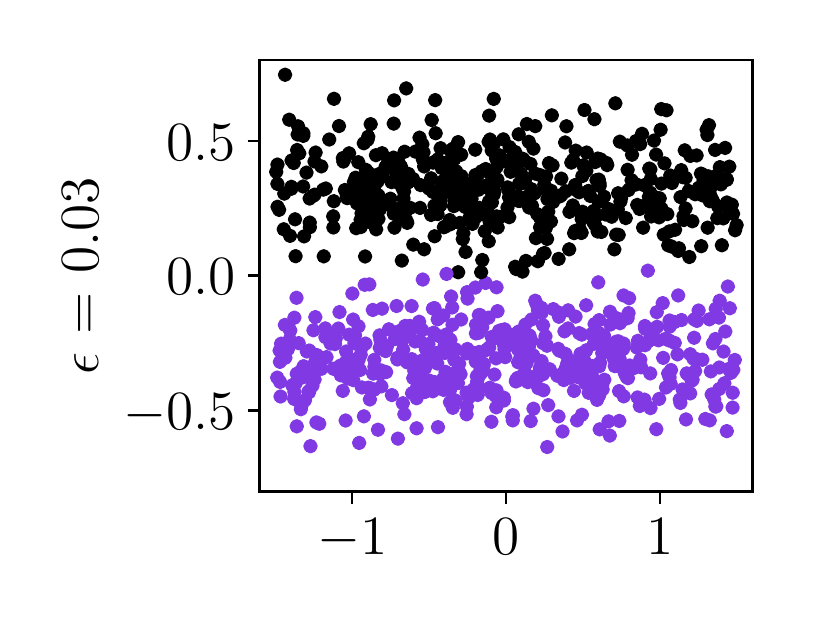} 
\includegraphics[height=2.5cm,trim={2.4cm 1.2cm .6cm .5cm},clip]{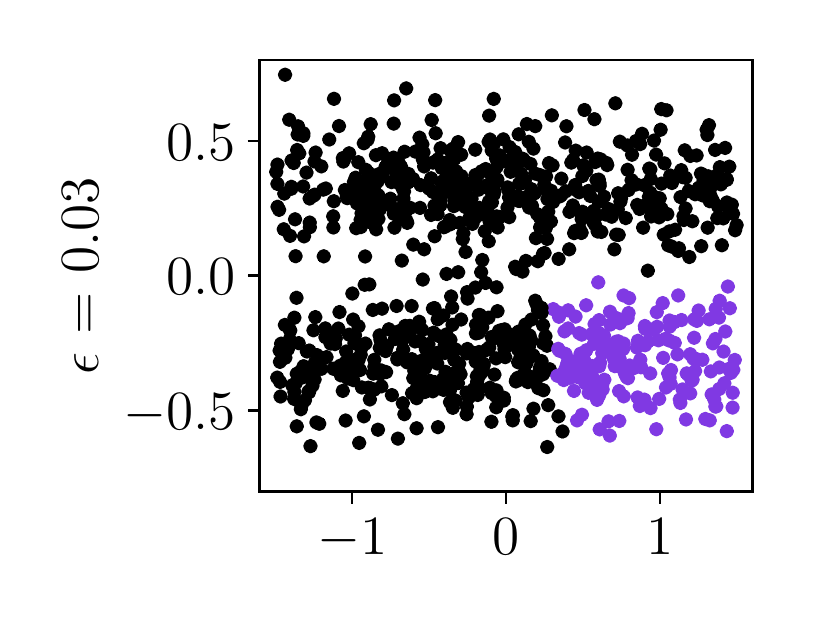}\\
\includegraphics[height=2.5cm,trim={.6cm 1.2cm .6cm .5cm},clip]{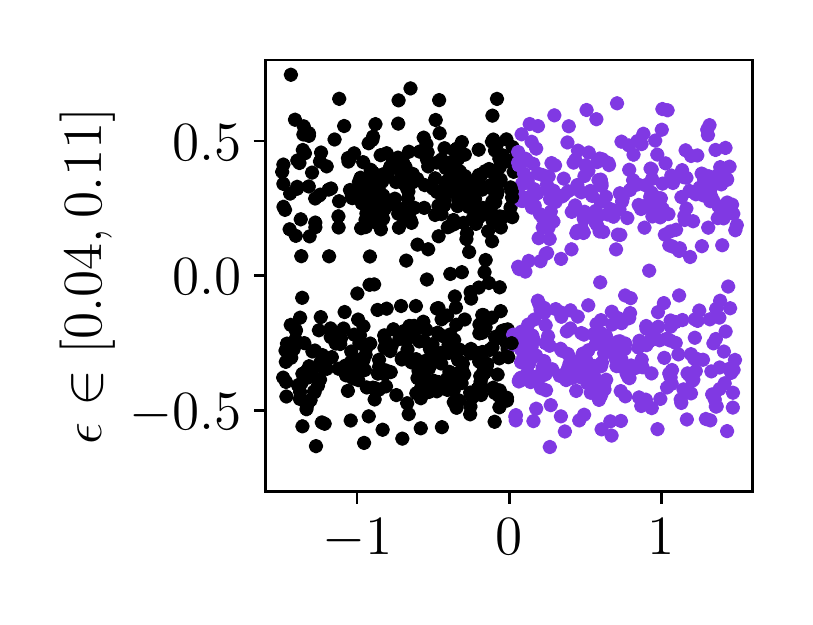} 
\includegraphics[height=2.5cm,trim={2.4cm 1.2cm .6cm .5cm},clip]{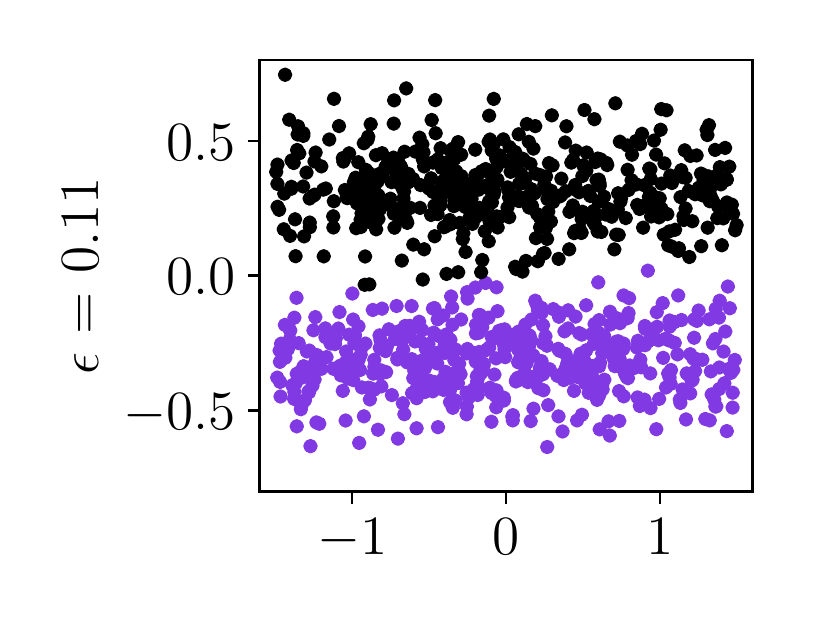} 
\includegraphics[height=2.5cm,trim={2.4cm 1.2cm .6cm .5cm},clip]{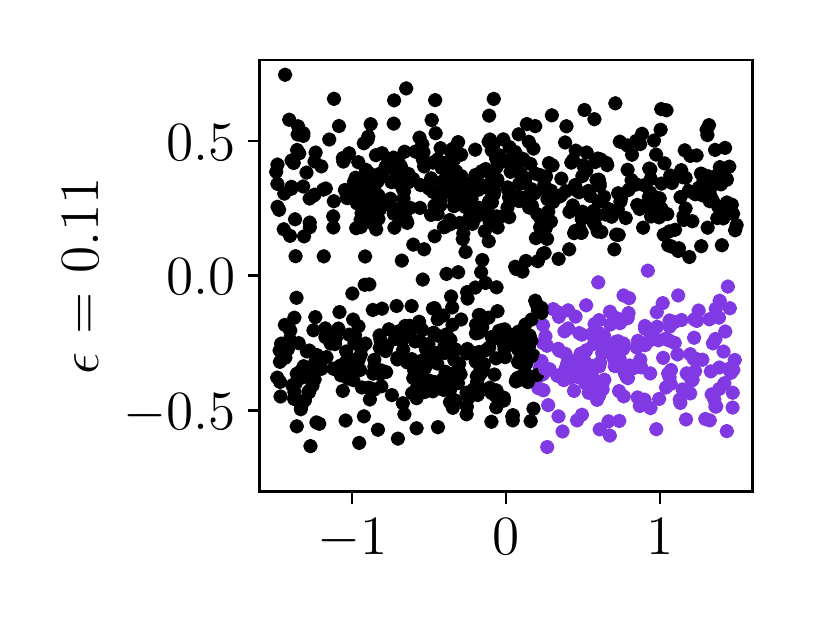} \\
\includegraphics[height=2.5cm,trim={.6cm 1.2cm .6cm .5cm},clip]{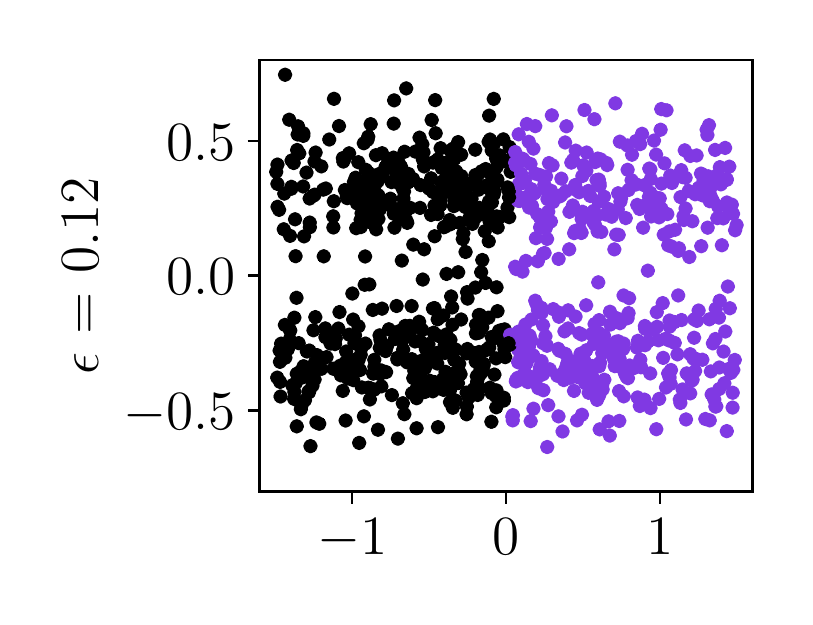} 
\includegraphics[height=2.5cm,trim={2.4cm 1.2cm .6cm .5cm},clip]{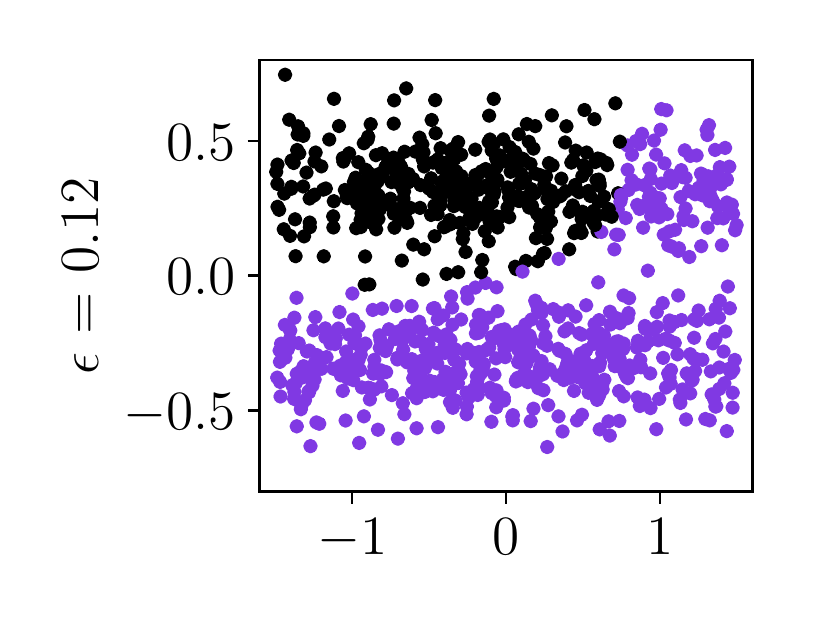} 
\includegraphics[height=2.5cm,trim={2.4cm 1.2cm .6cm .5cm},clip]{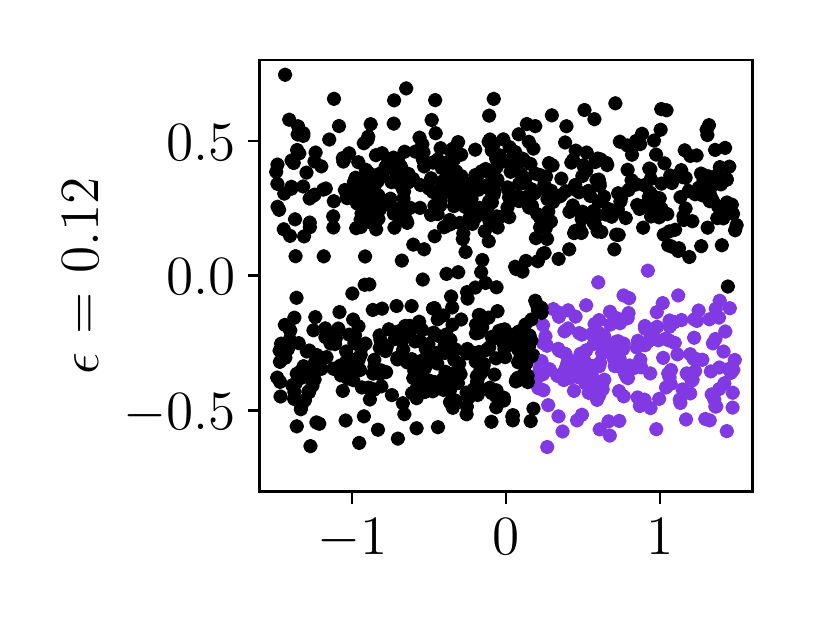} 
\caption{For the same data distribution as in Figure \ref{blueskyclustering}, we investigate how the clustering behavior of $Q_\beta$ depends on the graph connectivity length scale $\varepsilon$ for 
$\beta =0.20, 0.95, 1.00$.}
\label{blueskyepsilon}
\end{figure}

Finally, in Figure \ref{blueskyepsilon}, we investigate how the clustering behavior of $Q_\beta$ for $\rho_{\text{blue sky}}$ depends on the graph connectivity $\varepsilon$; see equation (\ref{eq:gaussian}). The columns correspond to $\beta =0.20, 0.95, 1.00$ and the rows correspond to $\varepsilon = 0.01, 0.03, [0.04,0.11]$ and $0.12$.  
{We note that for a wide range of $\varepsilon$, the diffusion dominated regime $\beta=0.2$ prefers to make shorter cuts even over parts of the domain where data are dense, which is undesirable for the data considered. On the other hand, the pure mean shift suffers, as in other examples, from the tendency to identify spurious local maxima of KDE as clusters. 
We observe the best clustering performance over the wide range of $\varepsilon$ for $\beta = 0.95$.
Considered together with other experiments, 
this suggests that adding even a small amount of diffusion goes a long way towards correct clustering.}

% While for a large amount of diffusion ($\beta =0.20$), some components of the clusters are identified for $\varepsilon$ very small ($\varepsilon = 0.01$), we observe the best clustering performance over the widest range of $\varepsilon$ for $\beta = 0.95$.

% We continue to observe poor clustering for pure mean shift ($\beta = 1$), for all choices of $\varepsilon>0$. Note that, if the   mean shift dynamics were restricted to the two rectangular regions, which one could think of as the data manifolds without noise, it would perform well.   However, in the presence of noise, mean shift is fragile and the dynamics  is sensitive to initial conditions, leading to poor results in our simulation.
%  As in our previous experiments, we see that adding just a bit of diffusion ($\beta=0.95$) resolves the issue, while making the most of the density information. This suggests that a little bit of diffusion goes a long way towards correct clustering. 

\subsubsection{Density vs. Geometry}

\begin{figure}
\hspace{-.5cm} \includegraphics[height=4.8cm,trim={.1cm 0.7cm 3.25cm .5cm},clip]{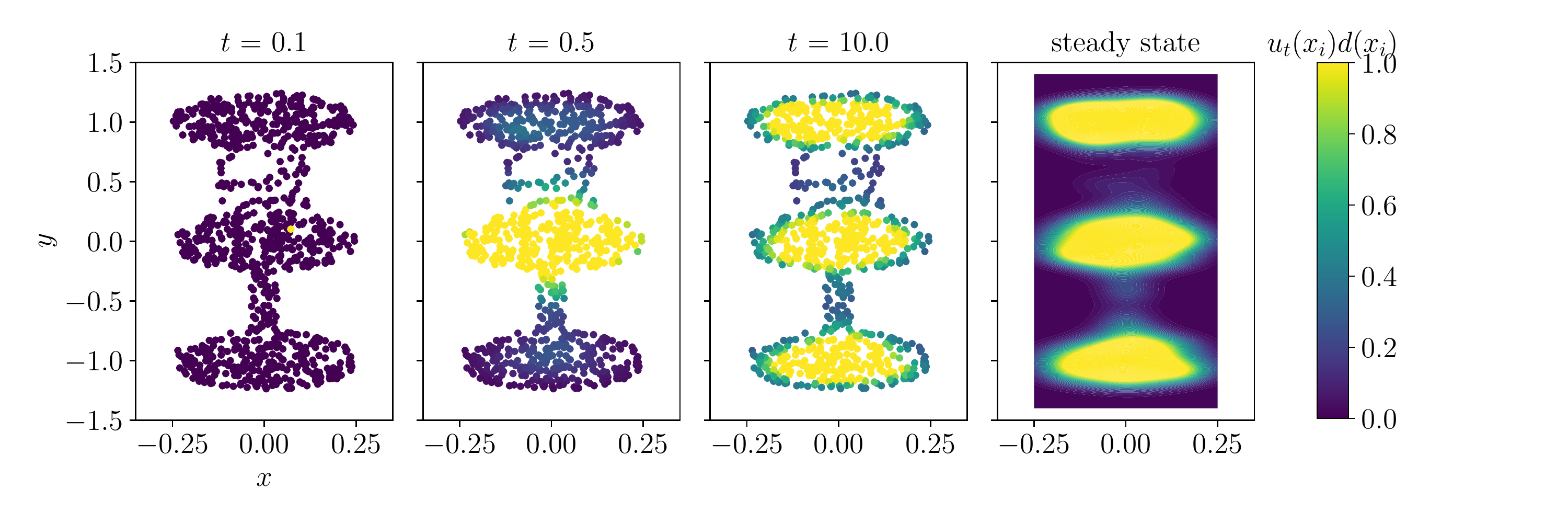} 

\caption{Illustration of graph dynamics of $Q_\beta$ on $\rho_{\text{three blobs}}$ for $n = 966$ samples, with $\beta = 0.7$ and  initial condition  $\delta_{x_i, y_i}$ for $(x_i,y_i) = (0.07,0.10)$. The markers in the first three columns represent the locations of the samples, and the colors of the markers represent the value of $u_t(x_i) d(x_i)$. In the right column, we plot  the steady state of the corresponding continuum PDE (\ref{KDEMaxwellian}).} \label{threeblobsdynamics}
\end{figure}

In Figure \ref{threeblobsdynamics}, we consider the graph dynamics of $Q_\beta$ on a two dimensional data distribution chosen to illustrate how the competing effects of density and geometry depend on the parameter $\beta$. We choose $n = 966$ samples, $\varepsilon = 0.07$, and $\delta = 0.05$,
in order to optimize agreement between the discrete dynamics and the continuum steady state. 

The data density, which we refer to as $\rho_{\text{three blobs}}$ is given by a piecewise constant function that is equal to height one on the three circles of radius $0.25$, as well as on the wide rectangle $[0.25,0.75]\times[-0.125,0.125]$ on the top. On the narrow rectangle $[-0.75,0.25]\times[-0.04,0.04]$ on the bottom, the piecewise constant function has height four. Finally, the data density is multiplied by a normalizing constant so that it integrates to one over the domain $\Omega = [-1.5,1.5]\times[-1,1]$. 

In Figure \ref{threeblobsdynamics}, we choose $\beta = 0.7$ and initial condition for the dynamics to be $\delta_{x_i, y_i}$ for $(x_i, y_i) = (0.07,0.10)$. The locations of the markers represent the samples from the data distribution, and the colors of the markers represent the value of $ u_t(x_i)d(x_i) $ at each node.  In the right column, we plot  the steady state of the corresponding continuum PDE (\ref{KDEMaxwellian}). We observe good agreement with the graph dynamics and the steady state by time $t= 10$. 

\begin{figure}
\begin{centering}
\hspace{0.6cm} $\beta = 0.70$ \hspace{1.6cm} $\beta = 0.75$ 

 \hspace{-.2cm} \includegraphics[height=12.5cm,trim={.6cm .5cm .6cm .5cm},clip]{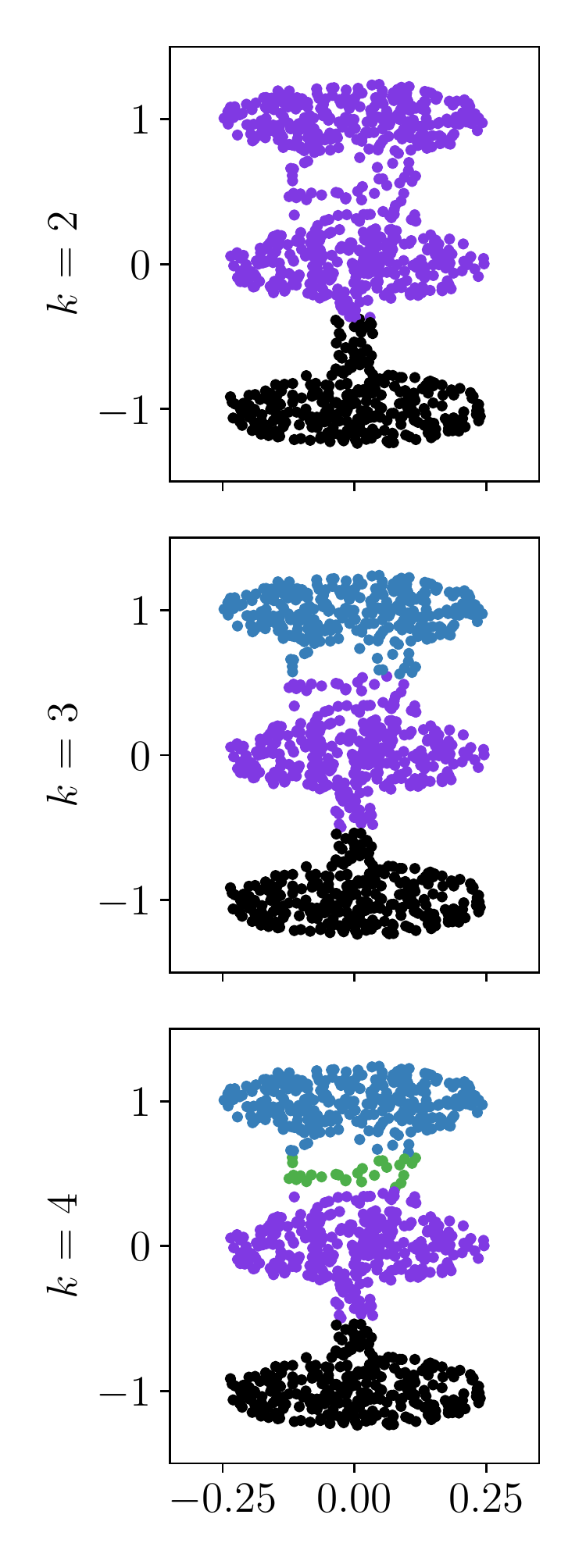} 
\includegraphics[height=12.5cm,trim={2cm .5cm .6cm .5cm},clip]{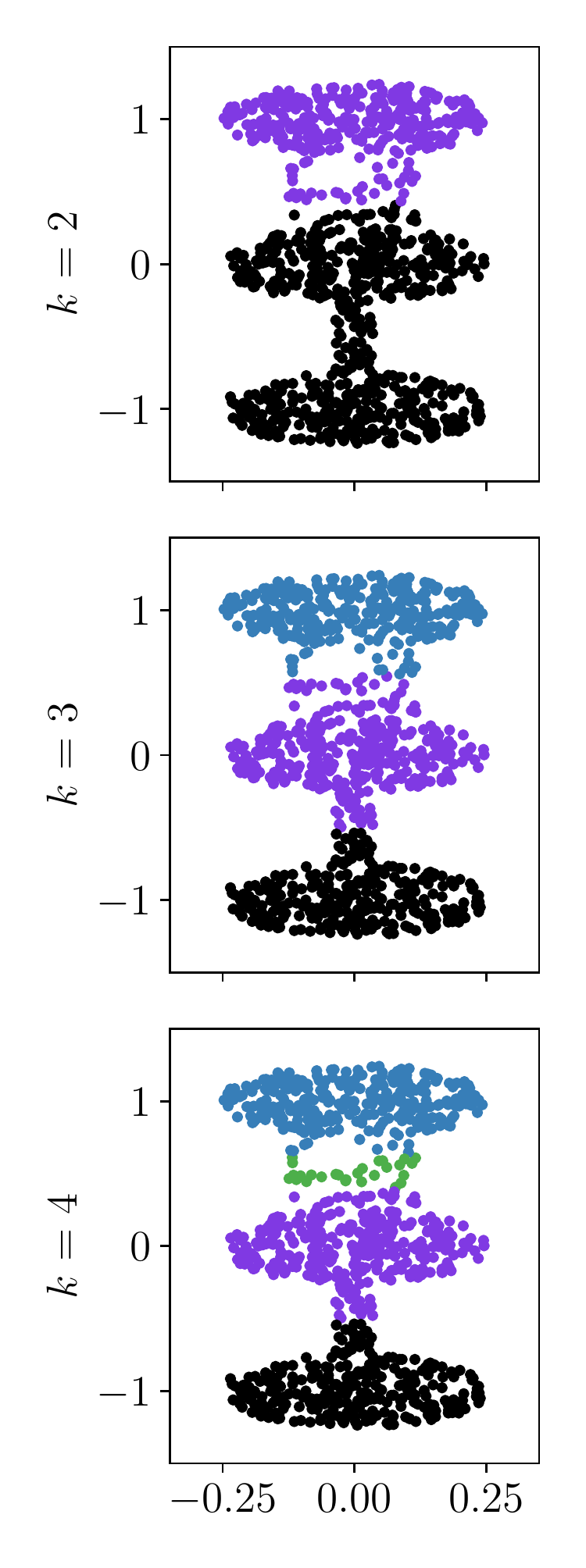} \\
  \includegraphics[height=2.45cm,trim={.6cm 1.1cm .9cm 1.1cm},clip]{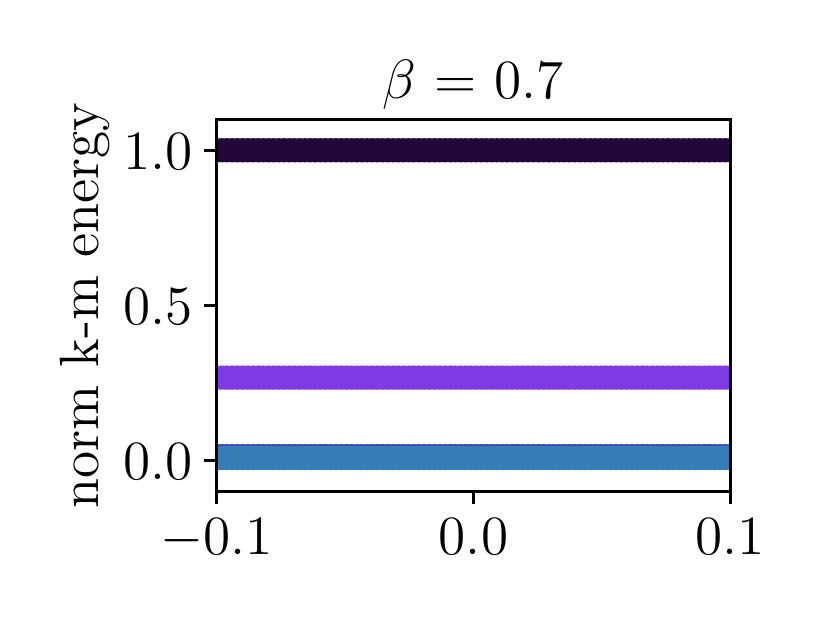}   
  \includegraphics[height=2.45cm,trim={2cm 1.1cm .9cm 1.1cm},clip]{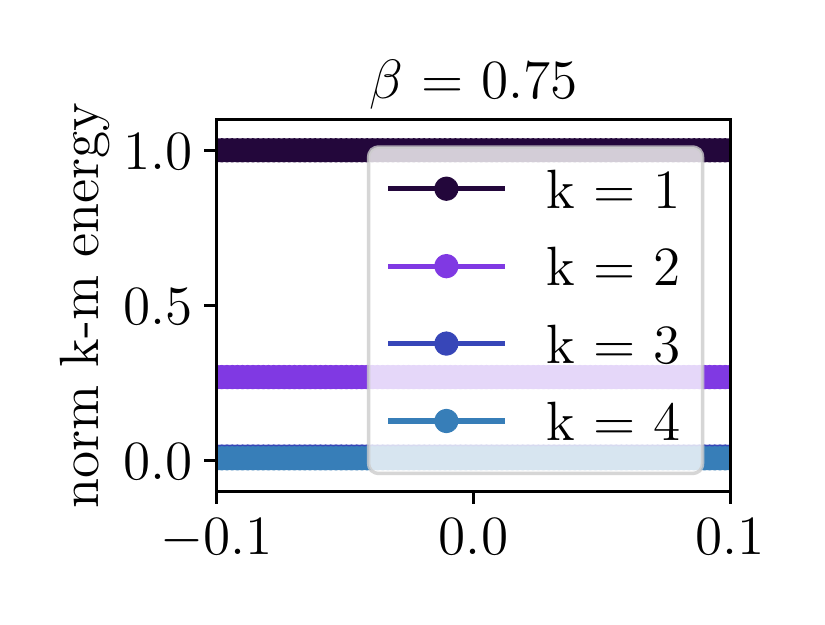}  
\end{centering}
\caption{Illustration of the clustering behavior of the  $Q_\beta$ on $\rho_{\text{three blobs}}$ for $n = 966$ samples at time $t = 10$. The first three rows show the clustering behavior for $k=2,3,4$, with each node colored according to which cluster it belongs. In the fourth row, we show the normalized $k$-means energy for each choice of $k$.} \label{threeblobsclustering}
\end{figure}

In Figure \ref{threeblobsclustering}, we show the clustering behavior of the  $Q_\beta$ on $\rho_{\text{three blobs}}$ for $n = 966$ samples at time $t= 10$. The two columns correspond to $\beta =0.7$ and $0.75$. The first three rows show the clustering behavior for $k=2,3,4$, with each node colored according to which cluster it belongs. In the fourth row, we show the normalized $k$-means energy for each choice of $k$.

This simulation provides an example of a data distribution where there is no single ``correct'' choice of clustering for $k=2$: a ``good'' clustering algorithm might seek to cut either the thin, high density rectangle on the bottom, or the wide, low density rectangle on the top. For small values of $\beta \leq 0.7$, diffusion dominates, and the clusters are chosen based on the geometry of the data, preferring to cut the thin, high density rectangle. For large values of $\beta \geq 0.75$, density dominates, and the clustering prefers to cut the wide, low density rectangle. For intermediate values of $\beta$, there is a phase transition for  which the clustering becomes unstable.

% \section{Conclusions}
% \label{sec:Conclusions}

% We establish two different connections between diffusion maps and mean shift algorithms. First, we describe the continuum analogues of the diffusion maps from \eqref{eqn:DiffusionMap} assuming that $\mathcal{G}$ is a proximity graph on random samples from a distribution supported on a compact smooth manifold without boundary. These continuum analogues turned out to be the same as the ones induced by the Fokker-Planck equations discussed in section \ref{sec:FokkerPlanckContinuum}. Through this connection we can conclude that the family of diffusion maps \eqref{eqn:DiffusionMap} produces the same interpolation between mean shift and spectral clustering, in the large sample limit, that the graph Fokker-Planck equations on graphs from section \ref{sec:FokkerPlanckGraphs} produce. Then, we turn our attention to discussing connections between mean shift and diffusion maps when we consider a fixed arbitrary graph $\G$, i.e. without considering any large sample asymptotics. We show that when the parameter $\alpha \in (-\infty, 1]$ for the diffusion maps is sent to $-\infty$, the dynamics converge to a version of the KNF mean shift dynamics on $\G$. With these two results we provide new conceptual and theoretical insights on diffusion maps.

\textbf{Acknowledgements:} The authors would like to thank Eric Carlen, Paul Atzberger, Anna Little, James Murphy, and Daniel McKenzie for helpful discussions. KC was supported by NSF DMS grant 1811012 and a Hellman Faculty Fellowship. NGT was supported by NSF-DMS grant 2005797.
DS was supported by NSF grant DMS 1814991.
NGT would also like to thank the IFDS at UW-Madison and NSF through TRIPODS grant 2023239 for their support.

\bibliographystyle{abbrv}
\bibliography{bib_spect_mean_shift}

\end{document}